\documentclass{article}
\let\articleaddcontentsline\addcontentsline
\usepackage{iclr2027_conference,times}
\let\addcontentsline\articleaddcontentsline
\usepackage{amsmath,amssymb,amsthm}
\usepackage{graphicx,booktabs,tabularx,array}
\usepackage{longtable}
\usepackage{enumitem,subcaption,xcolor}
\usepackage{microtype}
\usepackage{float}
\usepackage{xurl}
\usepackage{wrapfig}
\usepackage{hyperref}
\usepackage{verbatim}
\hypersetup{colorlinks=true,citecolor=blue,linkcolor=black,urlcolor=blue,linktoc=all}
\usepackage{etoc}
\usepackage{tikz}
\usetikzlibrary{arrows.meta}
\graphicspath{{./}{figs/main/}{figs/}{followup_figures/}%
  {../../followup_figures/}{../../figs/main/}{../../figs/}}
\usepackage{wrapfig}
\usepackage{needspace}

\newcommand{\E}{\mathbb{E}}

\DeclareMathOperator{\Var}{Var}
\DeclareMathOperator*{\argmin}{arg\,min}
\newcommand{\cH}{\mathcal{H}}
\newcommand{\cZ}{\mathcal{Z}}

\newcommand{\PRCE}{\ensuremath{\mathrm{PRCE}}}
\newcommand{\SKCE}{\ensuremath{\mathrm{SKCE}}}
\newcommand{\SCF}{\ensuremath{\mathrm{SCF}}}

\newcommand{\mstar}{m^{\star}}

\newcommand{\tauhat}{\hat{\tau}_n}

\newcommand{\iprod}[1]{\langle #1 \rangle_{\cH_k}}

\theoremstyle{plain}
\newtheorem{theorem}{Theorem}[section]
\newtheorem{lemma}[theorem]{Lemma}
\newtheorem{corollary}[theorem]{Corollary}
\newtheorem{proposition}[theorem]{Proposition}
\theoremstyle{definition}

\theoremstyle{remark}
\newtheorem{remark}[theorem]{Remark}

\newif\ifinternalreview
\internalreviewfalse

\newcommand{\revm}[1]{}

\usepackage{iclr2027_conference}

\iclrfinalcopy
\title{Are You Sure Your LLM Is Sure? \\ Repairing and Auditing Confidence}
\author{Arijit Dey$^{1, 2}$\textsuperscript{\textsection} \quad Bhartendu Kumar$^{2}$\textsuperscript{\textsection} \quad Ramya Hebbalaguppe$^{2,3}$\textsuperscript{\textsection} \quad Khuushi Maheshwari$^{2}$\\
$^{1}$ MBZUAI \\
$^{2}$ TCS Research Labs, India \\
$^{3}$ IIT Delhi
}
\begin{document}
\etocdepthtag.toc{main}
\maketitle

\lhead{Preprint -- October 2026}

\begin{abstract}
Confidence scores determine whether an LLM answer is accepted, reviewed, or
regenerated, yet scalar recalibration can only correct errors that are visible
through the score itself. We introduce \textbf{RADAR}
(\textbf{R}epair, \textbf{A}dopt, and \textbf{D}iagnose \textbf{A}fter
\textbf{R}ecalibration), a reliability framework that separates confidence
repair from deciding whether the repair should be deployed and auditing the
errors that remain. Under squared loss, we decompose prediction risk into
outcome noise, information hidden by the score, and mapping error; restricting
repair to order-preserving maps additionally introduces a pooling gap.
For \emph{repair}, we propose Probe-Anchored Isotonic Calibration (PAIC),
which learns a monotone map from correctness-probe predictions on unlabeled
target prompts. We show that PAIC's excess risk depends on the probe's
score-conditional bias and fitting error, rather than on its overall prediction
accuracy. For \emph{adoption}, we derive a paired procedure with finite-sample
control of false improvements. For \emph{auditing}, we develop a
prompt-clustered kernel test that detects feature-aligned residual errors while
accounting for dependent answers from the same prompt. Across multiple LLM
checkpoints, benchmarks, and distribution shifts, we find that richer
predictors can improve confidence estimation, but projecting their information
back to a scalar score can lose these gains, and lower prediction risk does not
necessarily yield better review decisions. Across 47,400 replayed raw-score allocations, paired betting adopts candidates with lower held-out test risk in 42.95\% of allocations, compared with 15.09\% for one-sided empirical Bernstein. Together, \textbf{RADAR} provides a principled workflow to repair what a scalar confidence score can express, adopt a repair only when supported by evidence, and diagnose the systematic errors that recalibration leaves behind. Code and data will be released upon acceptance.
\end{abstract}

\section{Introduction}

\paragraph{A confidence score can be calibrated and still fail predictably.}
LLM confidence often determines whether an answer is accepted, reviewed, or
regenerated. A standard remedy is \emph{scalar recalibration}, which learns a
transformation of confidence using methods such as temperature scaling, Platt
scaling, or isotonic regression
\citep{zadrozny2002,guo2017calibration}. Yet scalar recalibration has two
fundamental limitations: it typically requires labeled data from the deployment
population, which may be unavailable under distribution shift, and it can
change score values but cannot recover information absent from the score.

For example, suppose every answer receives confidence $0.8$, $80\%$ are
correct, and the remaining $20\%$ systematically confuse minutes with hours.
Expected calibration error is zero
\citep{naeini2015obtaining,guo2017calibration}, yet the failure remains
perfectly predictable. Even correcting an original confidence of $0.9$ to
$0.8$ reduces squared loss without revealing this subgroup. Thus,
\emph{correcting a score and recovering distinctions hidden by that score are
different problems}.





\paragraph{RADAR: Repair, Adopt, and Diagnose After Recalibration.}
We formalize this distinction through \textbf{RADAR}
(\textbf{R}epair, \textbf{A}dopt, and \textbf{D}iagnose
\textbf{A}fter \textbf{R}ecalibration), illustrated in
Figure~\ref{fig:pipeline}. Under squared loss, the risk of a scalar-repaired
score decomposes into outcome noise, information hidden by the score, and
mapping error; only the last term can be corrected by scalar recalibration.
Restricting the repair to preserve score ordering can introduce an additional
pooling gap. This decomposition motivates three distinct questions:
\emph{how should a score be repaired without target labels, when should the
repair be adopted, and what systematic errors remain afterward?}


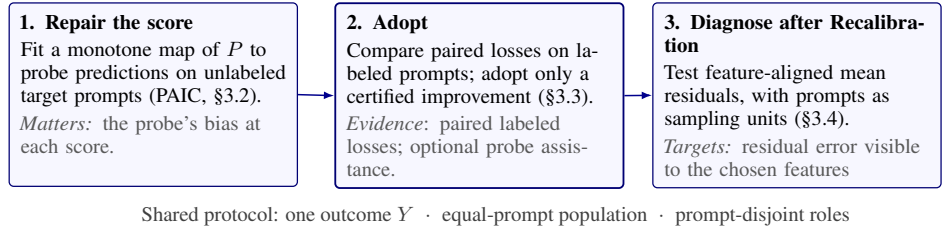
\begin{figure}[t]
\centering
\resizebox{0.9\textwidth}{!}{
\begin{tikzpicture}[
  card/.style={draw=blue!45!black,fill=blue!3,rounded corners=2pt,
    line width=.55pt,text width=4.05cm,minimum height=2.55cm,
    inner sep=5pt,align=left,font=\footnotesize,anchor=north},
  flow/.style={-{Latex[length=1.6mm]},line width=.65pt,blue!55!black}]
\node[card] (fit) at (-0.63,0) {
  \textbf{1. Repair the score}\\[2pt]
  Fit a monotone map of $P$ to probe predictions on unlabeled target
  prompts (PAIC, \S\ref{sec:paic}).\\[2pt]
  {\color{black!65}\emph{Matters:} the probe's bias at each score.\\[3.3mm]}};
\node[card, line width=0.8pt] (compare) at (4.35,0) {
  \textbf{2. Adopt}\\[2pt]
  Compare paired losses on labeled prompts; adopt only a certified
  improvement (\S\ref{sec:prce}).\\[2pt]
  {\color{black!65}\emph{Evidence}: paired labeled losses; optional probe assistance.}};
\node[card] (audit) at (9.2,0) {
  \textbf{3. Diagnose after Recalibration}\\[2pt]
  Test feature-aligned mean residuals, with prompts as sampling units
  (\S\ref{sec:skce}).\\[2pt]
  {\color{black!65}\emph{Targets:} residual error visible to the chosen features}};
\draw[flow] (fit.east) -- (compare.west);
\draw[flow] (compare.east) -- (audit.west);
\node[font=\footnotesize,align=center,text=black!75,anchor=north]
  at (4.6,-3.00) {Shared protocol: one outcome $Y$ $\;\cdot\;$ equal-prompt
  population $\;\cdot\;$ prompt-disjoint roles};
\end{tikzpicture}
}
\caption{RADAR uses separate prompt sets to fit a confidence repair, compare candidate risks, and test residual errors. All stages use the same outcome definition and weight prompts equally. PAIC fits the repair from probe predictions; comparison and auditing use held-out labels.}
\label{fig:pipeline}
\label{fig:teaser}
\end{figure}

\paragraph{Our approach.}
For \textbf{Repair}, we introduce \emph{Probe-Anchored Isotonic Calibration
(PAIC)}, which distills an independently trained correctness probe into a
monotone function of the original confidence using unlabeled target prompts. PAIC requires no target labels to fit, but held-out target labels test whether the fitted repair lowers target risk.
We show that PAIC's excess risk is controlled by the probe's
\emph{score-conditional bias} and its fitting error, rather than by overall
probe accuracy. 

For \textbf{Adopt}, we show that agreement with an imperfect probe can
misorder candidate confidence scores. We derive a paired risk comparison using
held-out labels and a betting-based procedure that controls the probability of
adopting a candidate that does not lower squared risk.

For \textbf{Diagnose}, we develop a feature-aware residual audit. Because
multiple answers generated for the same prompt may be dependent, the audit
uses cross-prompt kernel products and a prompt-level multiplier procedure to
test systematic residual structure without treating sibling answers as
independent.

\paragraph{Contributions.}
Our contributions are:
\begin{enumerate}[leftmargin=*,itemsep=1pt,topsep=2pt]
    \item \textbf{A unified framework for confidence repair, adoption, and residual diagnosis}: RADAR specifies separate data requirements and statistical guarantees for fitting a confidence repair, comparing candidate risks, and auditing residual errors. PAIC implements its repair stage using probe predictions on unlabeled target prompts.

    \item \textbf{Risk guarantees for repair and adoption.}
    PAIC performs target-label-free monotone repair with excess risk bounded by
    $(\|b_n\|_2+\sqrt{u_n})^2$
    (Theorem~\ref{thm:paic-main}), while our paired adoption rule controls the
    probability of deploying a non-improving candidate
    (Proposition~\ref{prop:betting-main}).

    \item \textbf{Residual diagnosis beyond scalar calibration.}
    We introduce a prompt-clustered kernel audit that detects
    feature-aligned residual errors while remaining valid under dependent
    answers from the same prompt
    (Theorem~\ref{thm:audit-main}).
\end{enumerate}
 Our experiments show that scalar recalibration can improve risk under some shifts, lose information relative to richer predictors in others, and fail to improve downstream decision utility even when prediction risk decreases. RADAR is designed to distinguish these regimes rather than assume that improved calibration alone is sufficient for deployment.

More broadly, this work aims to make AI systems more trustworthy not by making them answer more questions correctly, but by repairing their confidence when we can—and identifying when that confidence still should not be trusted.

\label{introend}

\section{Related Work}
\label{sec:related}

\paragraph{Grouping loss and monotone repair.}
Forecast decompositions separate reliability from resolution
\citep{degroot1983,brocker2009}. Grouping loss measures predictive information
hidden by a score, and non-injective recalibration can increase it
\citep{perezlebel2023grouping}; our $I$ is grouping loss under squared loss,
and isotonic calibration \citep{zadrozny2002,dimitriadis2021reliability} is our
monotone benchmark. Multiaccuracy, multicalibration, and semantic grouping correct
with features \citep{kim2018multiaccuracy,hebertjohnson2018multicalibration,yang2023grouping}.
For LLMs, \citet{chen2024reconfidencing} fit label-based isotonic maps within
learned groups, so serving needs group features, and
\citet{melo2026epistemic} estimate grouping loss for local audits and
deferral. PAIC instead serves a single map of the score
(Appendix~\ref{app:repair-context}).

\paragraph{Unlabeled adaptation and distillation.}
Under covariate shift, importance weighting reweights labeled source data
toward the target \citep{shimodaira2000improving}, including for calibration
\citep{park2020calibrated,wang2020transferable}; label-shift methods
re-estimate class priors \citep{lipton2018detecting,alexandari2020maximum}; and
ATC and DoC predict target accuracy from unlabeled data
\citep{garg2022leveraging,guillory2021predicting}. PAIC takes the complementary
outcome-model route: instead of a density ratio over high-dimensional features,
it transfers a probe's predictions on unlabeled target prompts into the map. It
is distillation with a student restricted to the score
\citep{hinton2015distilling}, akin to learning with privileged information
\citep{lopezpaz2016unifying}; our contribution is a guarantee identifying which
teacher errors matter.

\paragraph{Evidence for adopting a change.}
Prediction-powered inference (PPI) corrects model-based estimates with labeled
residuals \citep{angelopoulos2023ppi,angelopoulos2024ppiplusplus}; we apply it
to paired squared-loss contrasts between frozen candidates. Our adoption rule
tests bounded means by betting \citep{waudbysmith2024betting}, and
empirical-Bernstein bounds \citep{maurer2009empirical} give simultaneous
intervals (Appendix~\ref{app:corrected-comparisons}).

\paragraph{Kernel calibration tests.}
MMCE and the squared kernel calibration error (SKCE) embed residuals in a
reproducing kernel Hilbert space (RKHS)
\citep{kumar2018trainable,widmann2019calibration,marx2023distribution}. Their
estimators treat observations as independent; we handle answers clustered by
prompt, with multipliers in the spirit of wild-bootstrap kernel tests
\citep{chwialkowski2014wild}.

\paragraph{LLM confidence.}
Confidence can come from token likelihoods \citep{jiang2021know},
self-evaluation such as P(True) \citep{kadavath2022mostly}, verbalized
statements \citep{lin2022teaching,xiong2024uncertainty}, agreement across
samples \citep{wang2023selfconsistency,kuhn2023semantic,farquhar2024entropy},
or probes of internal states \citep{orgad2025intrinsic}. Our methods accept any
bounded scalar score; we evaluate likelihood-based and fixed-choice scores.
\citet{nakkiran2026concepts} relate semantic calibration to local loss
optimality; we instead repair and audit scores for fixed realized answers. 

See Appendix~\ref{app:related_detail} for a detailed and extended related work.
\label{relatedend}

\section{Repairing, Adopting, and Auditing a Confidence Score}
\label{sec:method}

\paragraph{Setup.}
Let $Z=\psi(X,A)$ collect the inputs available for answer $A$ to prompt $X$,
and let $Y\in[0,1]$ be binary correctness or a grade. A base score
$P=s_0(Z)$, a correctness probe $g(Z)$, and candidate scores $s(Z)$ all lie in
$[0,1]$. We measure squared risk, $L(s)=\E[(Y-s(Z))^2]$: the Brier score for
binary labels, and mean squared error targeting the expected grade for graded
labels. It is minimized by $\mstar(Z)=\E[Y\mid Z]$, and
\begin{equation}
\E[Y-s(Z)\mid Z]=0\ \Longleftrightarrow\ s(Z)=\mstar(Z)\quad\text{a.s.}
\label{eq:target}
\end{equation}
Scalar calibration, $\E[Y\mid s(Z)]=s(Z)$, checks averages only at each
reported value and is weaker. Several answers are sampled per prompt and may
be dependent, so expectations use the law $\mathbb Q$ that draws a prompt and
then one of its answers uniformly, and $\|\cdot\|_2$ is the corresponding
root-mean-square norm. The guarantees require sibling-preserving splits and
fitted components frozen independently of the comparison or audit prompts
(Appendix~\ref{app:assumptions}).

\paragraph{When must a correction use only the score?}
We study deployments in which the original score $P$ is available for each answer, while richer probe inputs or computation are available during offline adaptation. PAIC serves a fitted map of $P$ and preserves its ordering up to ties. When probe inputs are available at serving time, we include the direct probe in the candidate set (Table~\ref{tab:repair-inputs}).
A scalar map is the right object when deployment fixes the score's interface
while richer signals are available offline. \emph{Interface constraints:} a
threshold, triage queue, or API field consumes one score, and an
order-preserving map corrects its values while keeping its validated ranking,
up to ties. \emph{Fitting-time signals:} an expensive teacher, such as a larger
verifier, agreement among sampled answers, or execution checks, can score an
unlabeled calibration set once but is too costly for every served answer.
\emph{Access limits:} a hosted model may expose token log-probabilities but not
hidden states. When probe inputs are available at serving time, include the direct probe among the candidate scores. Declare supervision and serving inputs before comparison (Table~\ref{tab:repair-inputs}). Section~\ref{sec:emp:shift} distinguishes same-model probes from a verifier needing an extra pass per answer.

\begin{table}[t]
\centering
\footnotesize
\begin{tabularx}{\linewidth}{@{}>{\raggedright\arraybackslash}p{.31\linewidth}>{\raggedright\arraybackslash}X>{\raggedright\arraybackslash}p{.19\linewidth}@{}}
\toprule
\textbf{Correction} & \textbf{Supervision} & \textbf{Serving inputs} \\
\midrule
Source scalar maps & Source labels fit the map & $P$ \\
Importance-weighted isotonic & Source labels, reweighted toward unlabeled target prompts & $P$ \\
Source-distilled isotonic & $g$ from source labels; isotonic fit to $g$ on source prompts & $P$ \\
Target-probe Platt, beta & $g$ from source labels; parametric fit to $g$ on target prompts & $P$ \\
PAIC & $g$ from source labels; isotonic fit to $g$ on target prompts & $P$ \\
Target isotonic (yardstick) & Target labels fit the map & $P$ \\
Group maps \citep{chen2024reconfidencing} & Source labels fit groups and maps & $P$, group features \\
Direct probe & Source labels train $g$ & Probe features \\
\bottomrule
\end{tabularx}
\caption{Supervision and serving inputs in the raw-score shift experiments,
with the base scorer fixed. Source scalar maps are label-isotonic, positive Platt, monotone
beta, and a flexible nonmonotone regressor; the four rows from importance
weighting to PAIC share 500 unlabeled target and 500 unlabeled source prompts,
and target isotonic is a yardstick, not a competitor
(Appendix~\ref{app:repair-context}).}
\label{tab:repair-inputs}
\end{table}

\subsection{What a scalar map can change}
\label{sec:repair-limits}

For a recalibration $\tau:[0,1]\to[0,1]$, the best unrestricted target is
$q(P)=\E[Y\mid P]$, the expected correctness at each score. Let
$N=\E[\Var(Y\mid Z)]$ be outcome noise and $I=\E[\Var(\mstar(Z)\mid P)]$ the
variation in expected correctness hidden by the score.
Theorem~\ref{thm:decomp} gives
\begin{equation}
L(\tau\circ s_0)=N+I+\E[(\tau(P)-q(P))^2].
\label{eq:decomposition}
\end{equation}
Only the last term changes when we remap $P$. In the introductory population,
remapping $0.9$ to $0.8$ lowers risk from $0.17$ to $0.16$; if $Z$ identifies
the two families, the remaining $0.16$ is $I$ and $N=0$. The term $I$ is
grouping loss under squared loss \citep{perezlebel2023grouping}: risk below
$N+I$ requires inputs beyond $P$.

Preserving score order restricts $\tau$ to the class $\mathcal M$ of
nondecreasing maps $[0,1]\to[0,1]$. A strictly increasing map leaves every
ranking-based measure unchanged, including AUROC and the risk--coverage curve:
it moves values and the thresholds applied to them, not the order of answers.
Where $q$ decreases, a monotone map must pool distinct scores, which can create
ties. Let $\tau_{\mathrm{iso}}$ be the map in $\mathcal M$
closest to $q$ in $\|\cdot\|_2$, which weights scores by their frequency, and
let $G=\|q-\tau_{\mathrm{iso}}\|_2^2$ be the resulting pooling gap
(Appendix~\ref{app:projection-guide}).

\begin{theorem}[Monotone repair and pooling]
\label{thm:pooling-main}
Under the common bounded-information setup, $T=\tau_{\mathrm{iso}}(P)$
minimizes squared risk over $\mathcal M$ and satisfies
\[
L(T)=N+I+G,\qquad \E[Y\mid T]=T,\qquad
\E[\Var(\mstar\mid T)]=I+G.
\]
The maximum monotone gain over $P$ is $\E[(P-q(P))^2]-G$.
\end{theorem}

The best order-preserving repair is therefore calibrated yet still loses $I+G$
relative to $\mstar$, a loss that no check of its own calibration can reveal.
Theorem~\ref{thm:monotone} covers arbitrary score laws, including atoms, and
gives a three-state example with $I=0$ and $G=1/24$.

\subsection{PAIC: fitting the map to probe predictions}
\label{sec:paic}

Label-isotonic regression estimates $\tau_{\mathrm{iso}}$ from labeled pairs
$(P,Y)$. PAIC replaces the labels with a probe's predictions, so the map can be
fitted on prompts whose labels are unavailable, including prompts from the
population where the score will be served. Fit a bounded probe $g_n$ on labeled
training prompts. On $n$ calibration prompts that are independent of probe
training, each with $1\le K_j<\infty$ answers, fit
\begin{equation}
\tauhat\in\argmin_{\tau\in\mathcal M}
 \frac1n\sum_{j=1}^n\frac1{K_j}\sum_{i=1}^{K_j}
     (\tau(P_{ji})-g_n(Z_{ji}))^2,
\qquad s_{\rm PAIC}=\tauhat\circ s_0.
\label{eq:paic}
\end{equation}
Each prompt has total weight $1/n$, and serving evaluates only $\tauhat(P)$.
Conditional on probe training, PAIC's population target is the monotone
projection $\tau_{g_n}$ of $q_{g_n}(P)=\E[g_n(Z)\mid P]$. Define the
score-conditional bias $b_n=q_{g_n}-q$ and the population fitting gap
$u_n=R_{g_n}(\tauhat)-R_{g_n}(\tau_{g_n})$, where
$R_g(a)=\E[(g(Z)-a(P))^2]$; $u_n$ is a population loss gap, not the observed
training error.

\Needspace{7\baselineskip}
\begin{theorem}[PAIC under score-conditional probe bias]
\label{thm:paic-main}
With bounded probes fitted independently of the calibration prompts as above,
\[
0\le L(\tauhat\circ s_0)-(N+I+G)
\le(\|b_n\|_2+\sqrt{u_n})^2.
\]
If $\E_{\rm train}\|b_n\|_2^2\to0$ as independent calibration prompts grow,
PAIC converges in risk and in $L^2$ to the optimal monotone repair.
\end{theorem}

\paragraph{Which probe errors matter.}
Only the probe's error averaged at each base score moves PAIC's target. Errors
that cancel within a score level are harmless: in the constant-score example, a
probe predicting $0.8$ for every answer yields the optimal scalar target while
missing both families. A constant offset persists however many calibration
prompts are used. If only the unlabeled calibration inputs and a fixed probe
are available, populations whose correctness at one score differs by
$\pm\sqrt\delta$ can be indistinguishable; every scalar repair then has
worst-case excess risk at least $\delta$.
Overall accuracy therefore does not determine repair quality
(Section~\ref{sec:emp:mechanism}), and full probe consistency is sufficient but
not necessary. Theorem~\ref{thm:paic} gives $\E u_n\le25/\sqrt n$;
localization sharpens this to $\E u_n=O(n^{-2/3})$ even for a non-monotone
probe target, at the classical isotonic rate
\citep{zhang2002risk,chatterjee2015risk}. With squared score-conditional
bias bounded by $\delta$, the order $\delta+n^{-2/3}$ is minimax for
monotone scalar recalibration from exactly $n$ independent
$(P,g(Z))$ observations and no additional target-population information
(Appendix~\ref{app:paic-estimation}). The rates characterize convergence to the optimal monotone repair, which is also the population target of label-isotonic regression. Target-probe bias remains an assumption while finite-sample performance is evaluated in Section~\ref{sec:emp}.

\paragraph{Adapting to a new population.}
Theorem~\ref{thm:paic-main} holds under whichever population supplies the
calibration prompts, so drawing them from the target population repairs the
score there without target labels, with excess risk governed by the probe's
target bias $\E_{\rm tgt}[g_n-Y\mid P]$: under covariate shift, the
source-trained probe's error averaged over target inputs at each score.
Importance-weighted calibration targets the same map but needs a density ratio
over $Z$, where PAIC needs only those errors to average out.

\subsection{Deciding whether to adopt a repair}
\label{sec:prce}

Agreement with a probe cannot settle a comparison: a probe agrees perfectly
with itself, correct or not. For frozen candidates $s_a,s_b$ and probe $g$, let
$D_g(a,b)=\E[(s_a-g)^2]-\E[(s_b-g)^2]$ be their difference in squared
disagreement with $g$, the probe-relative calibration error (PRCE), and let
$\Delta_{ab}=L(s_a)-L(s_b)$ be their risk difference; negative values favor
$a$. Exactly (Proposition~\ref{prop:pairwise-risk}),
\begin{equation}
\Delta_{ab}=D_g(a,b)+2\E[(s_a-s_b)(g-Y)].
\label{eq:repair-comparison}
\end{equation}
Only probe error aligned with the candidate difference biases the comparison;
for scalar candidates it enters through $b_g(P)=\E[g-Y\mid P]$.

\paragraph{Correcting the comparison with labels.}
Freeze $g$ and $M$ candidates, including $s_0=P$, and draw a labeled pool
$\mathcal L$ and an unlabeled pool $\mathcal U$ from the same cluster law,
independent of each other and of all fitting. With $s_{a,ji}=s_a(Z_{ji})$ and
$g_{ji}=g(Z_{ji})$, define for prompt $j$
\begin{equation}
A_{j,ab}=\frac1{K_j}\sum_i
 \big[(s_{a,ji}-g_{ji})^2-(s_{b,ji}-g_{ji})^2\big],\qquad
C_{j,ab}=\frac2{K_j}\sum_i(s_{a,ji}-s_{b,ji})(g_{ji}-Y_{ji}).
\label{eq:main-contrasts}
\end{equation}
With bars denoting prompt averages over a pool, the prediction-powered
estimate $\widehat\Delta_{ab}=\overline A_{\mathcal U,ab}+\overline
C_{\mathcal L,ab}$ is unbiased for $\Delta_{ab}$ conditional on fitting, for any
bounded probe \citep{angelopoulos2023ppi}; a useful probe reduces the labeled
term's variance, and coincident pools give ordinary paired validation.

\paragraph{Adopting only certified improvements.}
Simultaneous empirical-Bernstein intervals cover every pairwise contrast
(Theorem~\ref{thm:corrected-comparison}). Their range terms exceed the gains
at the budgets in our tested controls (Section~\ref{sec:emp:adoption}). Adoption
needs only a one-sided answer per candidate $h$: does it lower risk relative to
$P$? Let $D_{jh}$ be prompt $j$'s average squared-loss difference between $h$
and $P$ on the labeled pool. Let $\delta_h\in[0,1]$ satisfy
$\delta_h\ge\sup_z|h(z)-s_0(z)|$, certified over the whole input domain;
then $|D_{jh}|\le B_h:=\delta_h(2-\delta_h)$.

\Needspace{8\baselineskip}
\begin{proposition}[Adoption by betting]
\label{prop:betting-main}
Fix $0<\alpha<1$, a nonempty finite set $\mathcal S\subset(0,1)$, and
$H\ge1$ non-original candidates with $B_h>0$, independently of comparison
outcomes. For $n$ labeled prompts drawn i.i.d.\ from the target cluster law
and independently of fitting, let
\[
E_h=\frac1{|\mathcal S|}\sum_{s\in\mathcal S}\ \prod_{j=1}^n
\Big(1-s\,\frac{D_{jh}}{B_h}\Big).
\]
Certify $h$ if $E_h\ge H/\alpha$; adopt the certified candidate with the
smallest mean $\overline D_h$ under a fixed tie order, and otherwise retain
$P$. Conditionally on
fitting, the probability that any candidate with $L(h)\ge L(P)$ is certified
is at most $\alpha$; in particular, a candidate that does not lower risk is
adopted with probability at most $\alpha$.
\end{proposition}

Each factor is positive with mean at most one when $L(h)\ge L(P)$, so
independence and Markov's inequality give the bound (Appendix~\ref{app:betting}).
This standard betting construction \citep{waudbysmith2024betting} controls
adoption, not two-sided coverage. The proxy-assisted variant additionally
requires disjoint groups of mutually independent labeled and unlabeled
prompt clusters from the same target law, independent of fitting, with its
certified range and fixed mixture (Appendix~\ref{app:betting}). We use
$\mathcal S=\{0.05,0.1,0.25,0.5,0.9\}$; choosing a procedure or budget after
inspecting a sweep is outside this guarantee.

\subsection{Auditing residual error after repair}
\label{sec:skce}

A repaired score can still leave mean errors along features that it does not
see: the losses $I+G$ in Theorem~\ref{thm:pooling-main}. In the introductory
population, the repaired score's residual $r=Y-s(Z)$ averages zero overall but
$-0.8$ within the conversion family. Let $k$ be a real measurable
positive-definite kernel with separable RKHS $\cH_k$, measurable feature map
$\Phi(z)=k(z,\cdot)$, and $k(z,z)\le\kappa<\infty$. With
independent primed copies,
\begin{equation}
\SKCE_k(s)=\|\E[(Y-s(Z))\Phi(Z)]\|_{\cH_k}^2
          =\E[rr'k(Z,Z')]
\label{eq:skce}
\end{equation}
is the squared largest residual moment over the RKHS unit ball
\citep{widmann2019calibration,marx2023distribution}
(Theorem~\ref{thm:skce-scf}). If $\cH_k$ is dense in $L^2(\mathbb Q_Z)$, zero
identifies $s=\mstar$; a finite feature map tests only its own moments, so the
kernel declares which hidden errors the audit looks for.

\paragraph{Prompts, not answers, are the sampling units.}
Answers sampled for one prompt share its content and often its mistakes, so
their residuals can be correlated even when every conditional mean is correct.
Take $J\ge2$ i.i.d.\ prompt clusters, independent of the fitted score,
kernel and representation, each with $1\le K_j<\infty$ possibly dependent
answers and bounded outcomes. Average
residual features within each prompt, and multiply only across different
prompts:
\begin{equation}
\zeta_j:=\frac1{K_j}\sum_{i=1}^{K_j}r_{ji}\,k(Z_{ji},\cdot),\qquad
\widehat{\SKCE}_{k,c}:=\frac1{J(J-1)}\sum_{j\ne\ell}
                                      \iprod{\zeta_j,\zeta_\ell}.
\label{eq:cluster-skce}
\end{equation}

\begin{theorem}[Prompt-clustered audit]
\label{thm:audit-main}
Under Assumptions~\ref{ass:kernel}--\ref{ass:audit} and the bounded setup
above, $\widehat{\SKCE}_{k,c}$ is unbiased for
$\SKCE_k(s)$ under $\mathbb Q$, for arbitrary within-prompt dependence and
unequal cluster sizes. Suppose $\SKCE_k(s)=0$ and the covariance operator of
$\zeta_j$ is nonzero. Let $x_j=\zeta_j-\bar\zeta$, draw independent
$e_j\sim\mathcal N(0,1)$, and set
$T^*_J=(J-1)^{-1}\sum_{j\ne\ell}e_je_\ell\iprod{x_j,x_\ell}$. Rejecting when
$J\,\widehat{\SKCE}_{k,c}$ exceeds the conditional $(1-\alpha)$ quantile of
$T^*_J$ has pointwise asymptotic size $\alpha$ as $J$ and the number of multiplier draws
grow.
\end{theorem}

Theorems~\ref{thm:ustat} and~\ref{thm:cluster-bootstrap} give the proofs and
the weighted chi-squared null limit. A statistic over all pairs of answers
instead has mean $\theta+\frac{K-1}{JK-1}(\gamma-\theta)$ with $K$ answers per
prompt, where $\theta=\SKCE_k(s)$ and $\gamma$ is the mean product of two sibling residual
features (Remark~\ref{rem:degeneracy}); this bias has the order of the null
statistic, so answer-level tests can reject calibrated scores whenever siblings
share errors. Non-rejection does not certify a small residual
(Proposition~\ref{prop:kernel-ucb}), and applied to the probe the statistic
bounds the comparison bias in \eqref{eq:repair-comparison}
(Proposition~\ref{prop:kernel-comparison}).
\label{methodend}

\section{Experiments}
\label{sec:emp}
We test whether unlabeled target prompts improve scalar repair, which teacher errors affect the fitted map, and whether the adoption and audit procedures distinguish useful score changes from residual failures. We then evaluate review decisions at fixed costs. Each study specifies its correctness labels, serving inputs, and independent sampling units; Appendix~\ref{app:followup} gives the full protocols.

\subsection{What richer inputs add beyond the score}
\label{sec:emp:breadth}

We compare 13 predictors for Llama 3.2 3B, Qwen 2.5 3B/7B and Mistral 7B
on 3,270 BoolQ validation and 1,172 ARC Challenge test questions; ARC fitting
uses Easy questions. Each model selects one fixed answer per question.
$P$ is its probability normalized over allowed choices. Scalar controls use
$P$; full-score controls retain the choice probabilities and selected-option
identity; rich direct also uses hidden states. Supervised candidates share
fitting labels. Rich-PAIC-$U$ and score-PAIC-$U$ project their respective
teachers on the same unlabeled target inputs, serving only a monotone map
of $P$ (Appendix~\ref{app:benchmark-breadth}).

Serving the rich predictor directly yields lower Brier risk than its monotone projection for all four BoolQ models and for ARC Qwen3B. The differences are $0.0121$-$0.0169$ on BoolQ and $0.0101$ on ARC Qwen3B, with joint intervals excluding zero. Relative to the selected full-score predictor, hidden states improve Brier risk for BoolQ Qwen3B and Mistral7B by $0.00803$ and $0.01131$; the other six comparisons are unresolved. No joint interval establishes an advantage for rich-PAIC-U over the selected scalar control or score-PAIC-U. Figure~\ref{fig:breadth-primary} shows all $32$ prespecified contrasts.


\begin{figure}[t]
\centering
\includegraphics[width=\linewidth,height=7cm]{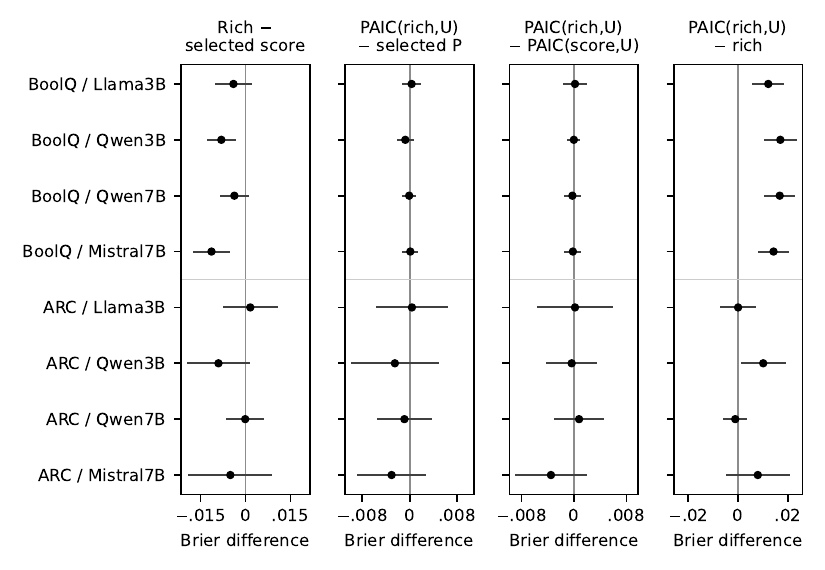}
\caption{Direct and projected confidence predictors on BoolQ and ARC. The fourth column measures the risk change from projecting each rich predictor onto a monotone function of $P$.}
\label{fig:breadth-primary}
\end{figure}

\subsection{Repair quality and transfer}
\label{sec:emp:transfer}
\label{sec:emp:shift}

\paragraph{Where the probe is wrong matters.}
\label{sec:emp:mechanism}
Equal probe MSE produces different scalar repairs. Three probes have the same excess outcome MSE, $0.0064$. With 2,000 calibration prompts, the two probes whose errors cancel within each score level yield mean excess repair risks of $3.90 \times 10^{-7}$ and $3.16 \times 10^{-5}$. A constant $0.08$ offset yields $0.0064$. Increasing the calibration sample removes fitting error but leaves the offset bias, as predicted by Theorem~\ref{thm:paic-main}.

\paragraph{Adapting a raw likelihood score under shift.}
Across difficulty, task, and subject shifts for Llama-3.2-3B and Qwen2.5-3B, tree-based PAIC reduces mean squared risk by $4.9$ - $16.4$\% relative to source label-isotonic calibration (Table~\ref{tab:repair-inputs}; Appendix~\ref{app:followup-shift}). These means use three overlapping refits, with 500 labeled source prompts and 500 unlabeled target prompts. PAIC has lower mean risk than every source-fitted scalar baseline in all six settings. Direct probes achieve lower risk in all six settings when their features are retained at serving time. With source labels and unlabeled source and target prompts matched, PAIC has lower estimated risk in 67 of 72 prespecified contrasts; approximate globally adjusted intervals favor PAIC in 47. All 18 contrasts with source-distilled isotonic favor target PAIC under that adjustment. Fitting on the target inputs therefore accounts for a substantial part of the improvement. Matching the squared-loss objective also narrows the difference from parametric target maps: PAIC's mean advantage over squared-loss beta is 0.000125–0.001368 across the six settings (Appendix~\ref{app:equal-info}).

\paragraph{Adapting a learned score to harder MATH problems.}
Holding the learned scorer and probe fixed, refitting PAIC on 777 unlabeled level-3 MATH prompts lowers Llama's target risk by $0.00633$ relative to source PAIC and by $0.00552$ relative to source label-isotonic; both adjusted intervals exclude zero. The effect varies by subject: Prealgebra risk increases by $0.00456$. For Qwen, the change from source PAIC is unresolved, and target PAIC remains $0.00083$ above source label-isotonic. The two checkpoints show different transfer outcomes under the same adaptation procedure (Appendix~\ref{app:target-adaptation}).

\paragraph{Presentation changes do not ensure repair gains.}
On constructed inventory-tracking problems, each model's answers stay fixed
across entity renaming, reordered facts and irrelevant context. Comparators
reuse 2,400 fitting labels per model across views. On 1,800 evaluation families
per model, target PAIC has higher risk than transformed-view label-isotonic.
Transformed teacher-fitting views narrow the gap, but no primary comparison
establishes a PAIC advantage: four favor the control and two remain unresolved
(Figure~\ref{fig:controlled-presentation}; Appendix~\ref{app:controlled-presentation}).

\begin{wrapfigure}{r}{7.5cm}
\centering
\includegraphics[width=\linewidth]{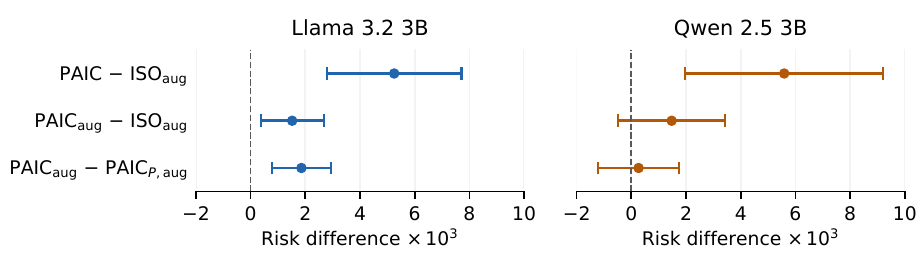}
\caption{\textbf{Controlled changes in presentation.} All six primary risk
contrasts; negative favors the first method. PAIC uses target projection;
$\mathrm{aug}$ fits transformed views with the same unique labels; $P$ marks a
score-only teacher. Approximate stratified finite-population Welch-$t$
intervals allocate $0.05/6$ per contrast, conditional on fixed fits and grades.}
\label{fig:controlled-presentation}
\end{wrapfigure}


\subsection{Adoption requires both an unbiased target and precision}
\label{sec:emp:adoption}
\label{sec:emp:controlled}

Six fixed candidates across 28 constructed settings test residual correction,
with 2,000 repetitions per setting. For a true risk difference $+0.0004$,
an offset probe gives mean proxy estimate $-0.002805$, reversing the preferred
score; labeled correction gives $+0.000397$. Corrected-to-labeled estimator
SD ratios are $0.798$ for a graded oracle control, $1.000$ for a binary blind
probe and $1.304$ for a reversed probe: assistance can improve, leave unchanged
or worsen precision (Appendix~\ref{app:controlled-comparisons}).

The simultaneous finite-sample intervals resolve no improvement in any
repetition, including with 2,000 labeled and 8,000 unlabeled prompts; their
range terms alone exceed the gains in these controls. A prespecified normal
benchmark has joint coverage $0.971$--$0.991$ (labeled) and
$0.9725$--$0.991$ (corrected), with asymptotic justification only.
Figure~\ref{fig:matched-power} isolates the adoption procedure using fixed candidates with known population risks. Under matched hypotheses, candidate families, certified ranges, and error budgets, paired betting reaches the prespecified adoption criterion with half as many comparison prompts in three populations and one quarter as many in the graded heterogeneous-gain population. Oracle-probe assistance does not reduce these crossings. Thus, correcting proxy bias and improving adoption efficiency are separate requirements; additional unlabeled predictions do not automatically reduce the labeling budget. On 47,400 cached raw-score allocations, betting adopts a beneficial candidate 20,356 times versus 7,155 for one-sided empirical Bernstein, with no false adoption estimated on held-out test prompts; most adoptions select a direct feature predictor (Appendix~\ref{app:followup-selection}).
\begin{figure}[t]
\centering
\includegraphics[width=.8\linewidth]{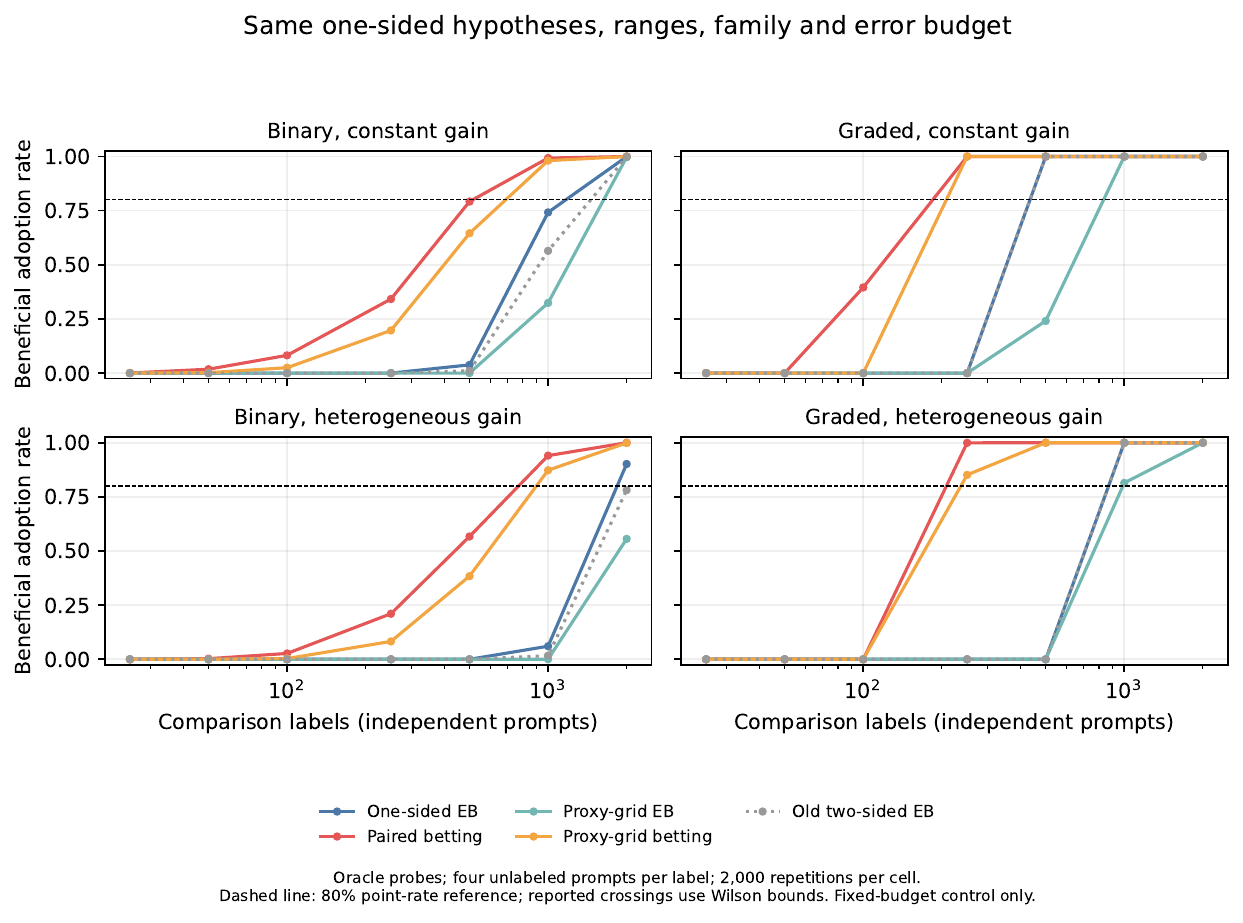}
\caption{\textbf{Adoption under matched one-sided tests.} Beneficial
adoption rate against the number of labeled comparison prompts in the four
known-risk gain populations (oracle probe, four unlabeled prompts per label,
2,000 repetitions per cell).}
\label{fig:matched-power}
\end{figure}
\subsection{Audits detect the mean errors their features expose}
\label{sec:emp:audit}
\label{sec:emp:controlled-audit}
The controlled audit fixes two score levels and two equally likely subgroups.
When subgroup errors cancel at each score, the feature test detects them while
the score-only test remains under its own null. A scalar distortion instead
favors the score-only test (Table~\ref{tab:study2-all}). Across seven calibrated
controls, feature-test rejection is $0.0485$--$0.0590$, with pointwise Wilson
intervals containing $0.05$. The learned auditor rejects 126 of 2,000
graded-null repetitions ($0.0630$, interval $[0.0532,0.0745]$). An
18,000-trial paired follow-up of that cell bounds the Gaussian test's excess
rejection over a sign-flip reference below 0.5 percentage points with exact
residuals (simultaneous 95\% bound); it neither explains the earlier rate nor
establishes nominal finite-sample size (Appendices~\ref{app:controlled-audit}
and~\ref{app:followup-audit}).
\subsection{Prediction loss and decision utility are different outcomes}
\label{sec:emp:decisions}
Perfect review costs $c$ and wrong acceptance one: accept $s\ge1-c$;
otherwise review. At $c=.10$, both adapted MATH maps review
all answers: Qwen's acceptance falls from $73.0\%$ to zero as its maximum
score crosses the threshold from $0.900015$ to $0.899516$. Rich-PAIC-$U$
accepts nothing on ARC Mistral at $c=.05$ and exceeds accept-all loss on
BoolQ Qwen3B/Mistral at $c=.50$. Rich direct has higher decision loss than its selected score control
on ARC Llama at every declared cost. On the archived
learned score (a previously recalibrated supervised score) at full budget,
comparison-fitted PAIC thresholds have higher
mean review loss than directly fitted original-score thresholds in $23/24$
settings with matched downstream labels (descriptive; Appendix~\ref{app:followup-policy}).

\section{Conclusion}
\label{sec:conclusion}
RADAR reframes confidence calibration as more than a score-correction problem. By separating confidence repair, evidence-based adoption, and residual diagnosis, it clarifies what scalar recalibration can improve, what information it cannot recover, and what evidence is needed before a repaired score should be deployed. Across models, tasks, and distribution shifts, our results show that lower prediction risk does not necessarily imply better downstream decisions. Our results argue for treating confidence not as a number to be calibrated once, but as a quantity that should be repaired, validated, and audited before it is trusted in deployment.

\label{mainend}
\clearpage

\section*{AI use statement}
\label{sec:ai-use}
Large language models were used for synthetic question and reference-solution
authoring and correctness grading, and for assistance with methodological
development, proof drafting, code, and writing.
This disclosure records uses known for the current draft, not a verification
certificate.

\section*{Reproducibility statement}
Appendix~\ref{app:theory} states all assumptions and gives complete proofs.
Appendix~\ref{app:followup} specifies the populations, splits, supervision
budgets, candidate families and uncertainty conventions for every experiment,
and Appendix~\ref{app:benchmark-breadth} gives the fixed-choice study and its
32B extension. Fitted candidates were saved before evaluation outcomes were
joined. The controlled studies, the fixed-choice and 32B studies, and the
equal-information, MATH-transfer, verifier and regeneration analyses were also
recomputed from saved predictions by a separate implementation. Code, saved
predictions and figure data will be released upon acceptance.

\section*{Ethics statement}
Automated grades can reflect grader bias or mistakes, and repeated grading
does not measure all such errors. Neither low calibration error nor audit
non-rejection establishes safety for a high-stakes use. The intended use is
to support inspection and comparison of confidence scores, not to replace
domain-specific validation or human oversight. Any data and code release
must respect model and dataset licenses and remove credentials and
identifying infrastructure information.

\bibliography{iclr2026_conference,followup_references,v3_references}
\bibliographystyle{iclr2026_conference}

\appendix
\clearpage
\etocdepthtag.toc{appendix}
\begingroup
\etocsettagdepth{main}{none}
\etocsettagdepth{appendix}{subsection}
\etocsettocdepth{subsection}
\etocsettocstyle{\section*{Appendix contents}\label{app:contents}
Section titles and page numbers link to the corresponding material.
Appendix~\ref{app:theory} gives the theory; Appendix~\ref{app:followup}
details the experiments of Section~\ref{sec:emp};
Appendix~\ref{app:benchmark-breadth} reports the fixed-choice study and its 32B
extension; Appendix~\ref{app:impact} discusses broader impact.\par\medskip}{}
\setlength{\parskip}{0pt}
\tableofcontents
\endgroup

\clearpage

\section{Theory and assumptions}
\label{app:theory}

This appendix gives the full statements and proofs for repair, comparison,
and auditing. For scalar repair, Theorem~\ref{thm:monotone} identifies
information lost through monotone pooling; Theorem~\ref{thm:paic} separates
score-conditional probe bias from the population probe-loss fitting gap.
Theorem~\ref{thm:corrected-comparison} gives observable comparison intervals
and the resulting improvement rule; Proposition~\ref{prop:betting} gives the
directional adoption test used in the main text.
Proposition~\ref{prop:kernel-comparison}
relates comparison bias to kernel residual moments; Theorems~\ref{thm:ustat}
and~\ref{thm:cluster-bootstrap} establish the clustered audit and its null
test. These results apply established conditional-expectation, projection,
model-assisted estimation, and concentration tools under the information
and sampling conventions below.

\subsection{Notation and definitions: a reader's guide}
\label{app:notation}

The two parts below explain the inputs, prediction target, and repair
geometry under the population law used throughout.
Appendix~\ref{app:assumptions} collects the assumptions.

\subsubsection{What is the information set, and which loss is used?}
\label{app:inputs-loss}

The \emph{information set} is the collection of variables available to a
predictor. We represent it by $Z=\psi(X,A)$, where $X$ is a prompt and $A$
is a generated answer. The base score $P=s_0(Z)$ and each candidate
$s(Z)$ are functions of these inputs. In the shift experiments, $Z$
contains the projected final answer state, likelihood summaries and $P$; for
the archived learned score, $Z=(W_\psi,P)$ contains a learned answer
representation and the base score. Reference-free serving
excludes the gold reference and correctness label. It does not mean that
training or evaluation is label-free, or that $Z$ contains all information
inside the LLM.

Formally, the information is the $\sigma$-algebra $\sigma(Z)$. With a
predictor fixed, two observations sharing $Z$ must receive the same score.
Previously learned components are held fixed, or conditioned on their
independent training data, throughout the population analysis.

The label $Y$ is binary correctness or a grade in $[0,1]$. The target
$\mstar(Z)=\E[Y\mid Z]$ is therefore a correctness probability or an
expected grade. For a confidence value $t\in[0,1]$ and predictor $s$, define
the loss and its population risk by
\begin{equation}
\ell_{\rm sq}(Y,t)=(Y-t)^2,\qquad
L(s)=\E[\ell_{\rm sq}(Y,s(Z))].
\label{eq:app-correctness-loss}
\end{equation}
For binary $Y$, this is the binary Brier convention with one predicted
probability. For graded $Y$, the risk $L(s)$ is mean-squared error.
The squared loss is \emph{neither negative log-likelihood nor squared
negative log-likelihood}: for $0<t<1$, binary log loss is instead
$-Y\log t-(1-Y)\log(1-t)$. Token-likelihood summaries may be predictor
inputs, but they do not change the loss used to evaluate answer correctness.

With repeated answers, $\mathbb Q$ first samples a prompt and then one of
its answers uniformly. Thus prompts, not answer rows, receive equal weight
in population targets and risks. Equation~\eqref{eq:prompt-law} defines
this law when answer counts differ or siblings are dependent.

\subsubsection{Miscalibration, monotone projection, and the two remaining losses}
\label{app:projection-guide}

The \emph{ideal correctness curve} for the base score is
$q(p)=\E[Y\mid P=p]$, defined up to the score law
$\nu=\mathcal L(P)$. Scalar calibration of $P$ means $q(P)=P$ almost
surely. More generally, a score $s$ is scalar-calibrated when
$\E[Y\mid s]=s$; the stronger condition relative to $Z$ is
$\E[Y-s\mid Z]=0$. The latter implies the former, but not conversely
(Proposition~\ref{prop:semantic-implies-scalar}).

A scalar recalibration is a measurable map $\tau:[0,1]\to[0,1]$;
its output $\tau(P)$ cannot distinguish answers sharing $P$.
Theorem~\ref{thm:decomp} separates its risk into outcome noise,
score insufficiency, and mapping error:
\[
L(\tau(P))=\underbrace{\E[\Var(Y\mid Z)]}_{N}
 +\underbrace{\E[\Var(\mstar(Z)\mid P)]}_{I}
 +\|\tau-q\|_{L^2(\nu)}^2.
\]
For the unchanged score, the last term is $\|P-q(P)\|_2^2$, its squared
scalar calibration error. The norm averages under the declared population
law; $\|\tau-q\|_{L^2(\nu)}^2=\int(\tau(p)-q(p))^2\,d\nu(p)$.
Insufficiency concerns differences in
\emph{conditional mean} correctness hidden by $P$, not merely different
realized labels at the same score. The noise term $N$ is present even for
the best predictor using $Z$.

\emph{Monotone projection} means: among all order-preserving recalibration
functions, choose the one closest to the ideal correctness curve in
population squared distance. Score values are weighted by their frequency
under $\nu$, not uniformly along the horizontal axis of a plot. With
$\mathcal M$ the nondecreasing maps $[0,1]\to[0,1]$, this is
\[
\tau_{\rm iso}\in\argmin_{\tau\in\mathcal M}
 \int(\tau(p)-q(p))^2\,d\nu(p),\qquad
G=\|q-\tau_{\rm iso}\|_{L^2(\nu)}^2.
\]
Nondecreasing means that score order cannot reverse, although distinct
scores may become tied. Pooling combines regions whose ideal correctness
values would otherwise violate this order. The population curve $q$ is
unknown; PAIC instead fits to probe predictions, with its target and
estimation error analyzed in Appendix~\ref{app:paic-estimation}.

For the ideal monotone output $T=\tau_{\rm iso}(P)$, the \emph{two losses
that can remain} are:
\begin{itemize}[leftmargin=*,itemsep=2pt,topsep=2pt]
\item $I$: differences in expected correctness already hidden by the base
score, before recalibration.
\item $G$: additional differences lost when the monotone map pools
distinct base-score values into the same output.
\end{itemize}
Theorem~\ref{thm:monotone} gives
\[
L(T)=N+I+G,\qquad \E[Y\mid T]=T,\qquad
\E[\Var(\mstar\mid T)]=I+G.
\]
Thus $I+G$ is excess risk above $L(\mstar)=N$, not residual scalar
miscalibration of $T$. The components need not both be positive. If $Z$
identifies the two families in the introduction, their constant score has
$N=G=0$ and $I=0.16$. The three-state example after
Theorem~\ref{thm:monotone} instead has $N=I=0$ and $G=1/24$.
The calibration identity is for the ideal population map; a finite PAIC
fit need not be scalar-calibrated when probe bias or fitting error remains.
For a general non-injective repair, $\|\tau-q\|_2^2$ is error relative
to the original score's ideal curve, not necessarily the scalar calibration
error of the repaired output itself.

\subsection{Information and sampling conventions}
\label{app:assumptions}

The following conditions are used only where indicated.
\begin{enumerate}[label=(T\arabic*),leftmargin=*]
\item\label{ass:information}
Let $(X,A,Y)$ have a fixed probability law $\mathbb Q$, where
$Y\in[0,1]$ and the spaces are standard Borel. A measurable representation
$Z=\psi(X,A)$ is fixed. Both $s(Z)$ and $P=s_0(Z)$ take values in $[0,1]$ and use this
same information. Write $\nu=\mathcal L(P)$ and
$\mstar(Z)=\E[Y\mid Z]$. Conditional expectations and optimal predictors
are identified up to equality almost surely. All expectations use the
declared population weighting.

\item\label{ass:probe}
For PAIC, let $\mathcal T_n$ denote standard-Borel training data,
independent of $n$ i.i.d. calibration observations from $\mathbb Q$.
The map $(\mathcal T_n,z)\mapsto g_n(z)\in[0,1]$ is jointly measurable.
For a fresh population observation independent of training, define
\[
q_{g_n}(P)=\E[g_n(Z)\mid P,\mathcal T_n],\qquad
b_n=q_{g_n}-q,\qquad
\delta_n=\E_{\mathrm{train}}\|b_n\|_{L^2(\nu)}^2.
\]
The consistency result assumes $\delta_n\to0$: probe bias averaged
\emph{at the base score} vanishes. Training randomness is held fixed inside
the conditional expectation. Training size may grow with $n$. The
representation and base score remain fixed; if learned previously, condition
on that independent upstream training. This assumption is not a property
automatically supplied by a finite probe family. Full probe consistency,
$\E_{\mathrm{train}}\|g_n-\mstar\|_2^2\to0$, is a stronger sufficient
condition by conditional-expectation contraction.

\item\label{ass:kernel}
For kernel results, $k$ is a real measurable positive-definite kernel with
a separable RKHS $\cH_k$, measurable feature map
$\Phi(z)=k(z,\cdot)$, and $k(z,z)\le\kappa<\infty$.
Consequently, $|k(z,z')|\le\kappa$ and $\|\Phi(z)\|\le\sqrt\kappa$.
For the zero-identification theorem, additionally assume that $\cH_k$ is
dense in $L^2(\mathbb Q_Z)$. A sufficient condition is that $\cZ$ is a compact
metric space and $k$ is continuous and universal. This extra condition is
\emph{not} needed for the norm identity or its upper bound.

\item\label{ass:audit}
When prompt-cluster sampling is invoked, comparison or audit data consist
of $J\ge2$ i.i.d. prompt clusters
$C_j=((Z_{ji},Y_{ji}))_{i=1}^{K_j}$, with $1\le K_j<\infty$ almost
surely. Dependence inside a cluster is unrestricted. All fitted components
used in the statistic are fixed independently of these clusters: the probe
and entire candidate family for comparisons, or the score and kernel for
auditing, including their representation. The population law $\mathbb Q$
is then the prompt-weighted answer law,
\begin{equation}
\mathbb Q f=\E_C\left[\frac1K\sum_{i=1}^K f(Z_i,Y_i)\right].
\label{eq:prompt-law}
\end{equation}
Cluster results use $\mathbb Q$ for targets, risks, norms, and any
kernel-density condition. This equals the usual answer-average weighting
when every cluster has the same size; otherwise the two weighting schemes
generally define different populations.
\end{enumerate}

\paragraph{Notation.}
Set $L(s)=\E[(Y-s(Z))^2]$, $q(P)=\E[Y\mid P]$,
$N=\E[\Var(Y\mid Z)]$, and
$I=\E[\Var(\mstar(Z)\mid P)]$.
The symbol $P$ denotes the \emph{scalar score} and $\mathbb Q$ the
population law. A composition such as $\tau\circ s_0$ is
abbreviated by $\tau(P)$ inside an expectation. Unless a law is specified,
$L^2$ norms of feature functions use $\mathbb Q_Z$, and norms of scalar
maps use $\nu$. For graded labels, the loss is mean-squared error; the
binary case is the Brier score. Results for fixed fitted objects can be
read conditionally on their joint training data; independence is required
from the subsequent sample on which a comparison or audit is evaluated.

\paragraph{Information use across stages.}
Table~\ref{tab:objects} details the label access and independence conditions
behind the workflow in Figure~\ref{fig:pipeline}. The compact main-text
comparison in Table~\ref{tab:repair-inputs} distinguishes the candidates'
fitting targets and serving inputs.

\begin{table}[t]
\centering
\small
\begin{tabularx}{\linewidth}{@{}>{\raggedright\arraybackslash}p{.22\linewidth}>{\raggedright\arraybackslash}X@{}}\toprule
\textbf{Stage} & \textbf{Inputs and label access} \\
\midrule
Fit the probe & Train $g$ on its feature inputs and $Y$. Count these labels
even when later stages use only its predictions. \\
Fit scalar repairs & Label-fitted maps use $(P,Y)$; PAIC uses $(P,g(Z))$
on repair prompts independent of probe training, without their labels. \\
Compare candidates & Evaluate the frozen candidates and $g$ on disjoint
$\mathcal L$ and $\mathcal U$; use $Y$ only on $\mathcal L$. Both pools are independent of fitting. \\
Serve the score & PAIC uses $P$; group calibration also needs a feature-based
group; direct $g$ needs its feature inputs. None uses $Y$. \\
Audit the score & Use $Z$, $Y$, and the chosen score on new independent
prompts, with a kernel fixed before the audit. \\
\bottomrule
\end{tabularx}
\caption{Information use with the representation and base scorer frozen.
Reference-free serving still incurs the cost of $P$; fitting and evaluation
use labels. Risk comparisons share one outcome and population weighting.}
\label{tab:objects}
\end{table}

\subsection{Population risk and scalar repair limits}

\begin{theorem}[Oracle PRCE equals excess squared risk]
\label{thm:prce-brier}
Under \ref{ass:information}, every measurable $s:\cZ\to[0,1]$ satisfies
\[
L(s)=N+\|s-\mstar\|_2^2,
\qquad
\PRCE(s;\mstar)=L(s)-L(\mstar).
\]
Thus $\mstar$ uniquely minimizes risk among $Z$-measurable scores.
\end{theorem}
\begin{proof}
Insert $Y-s=(Y-\mstar)+(\mstar-s)$ into the square. The cross-term is
zero because $\E[Y-\mstar\mid Z]=0$ and $\mstar-s$ is
$Z$-measurable. The first squared term is
$\E[(Y-\mstar)^2]=\E[\Var(Y\mid Z)]=N$. The remaining term is
$\|s-\mstar\|_2^2$, which is zero exactly when $s=\mstar$ almost surely.
\end{proof}

\begin{proposition}[Semantic calibration implies scalar calibration]
\label{prop:semantic-implies-scalar}
Under \ref{ass:information}, $s(Z)=\mstar(Z)$ implies
$\E[Y\mid s(Z)]=s(Z)$. The converse can fail.
\end{proposition}
\begin{proof}
The tower property gives
$\E[Y\mid s]=\E[\E[Y\mid Z]\mid s]=\E[s\mid s]=s$.
For the converse, take the two-family population from the introduction:
$Y=1$ on a family of mass $0.8$ and $Y=0$ on its complement, let $Z$
identify the family, and set $s=0.8$. The score is scalar-calibrated but
differs from $\mstar=Y$ in both families.
\end{proof}

\begin{theorem}[Total-risk decomposition]
\label{thm:decomp}
Under \ref{ass:information}, for any measurable
$\tau:[0,1]\to[0,1]$,
\begin{equation}
L(\tau\circ s_0)=N+I+\|\tau-q\|_{L^2(\nu)}^2.
\label{eq:decomp}
\end{equation}
In particular, $\PRCE(\tau\circ s_0;\mstar)=I+\|\tau-q\|_2^2$.
Subtracting the representation-relative optimum $L(\mstar)=N$ leaves
information loss and mapping error.
\end{theorem}
\begin{proof}
Since $P=s_0(Z)$, $q(P)=\E[\mstar(Z)\mid P]$.
In the squared term of Theorem~\ref{thm:prce-brier}, write
$\mstar-\tau(P)=(\mstar-q(P))+(q(P)-\tau(P))$.
Its cross-term vanishes after conditioning on $P$. The first squared
term equals $\E[\Var(\mstar\mid P)]=I$; the second is
$\|\tau-q\|_2^2$. These are applications of conditional-mean
orthogonality underlying classical forecast decompositions
\citep{degroot1983,brocker2009}.
The information term also exposes the within-score heterogeneity studied
as grouping loss \citep{perezlebel2023grouping}.
\end{proof}

\paragraph{Restricting the repair class.}
Let
\[
\mathcal M=\{\tau:[0,1]\to[0,1]:\tau\text{ is nondecreasing}\}
\]
be the set of actual bounded nondecreasing maps; these maps are Borel
measurable. Write $[\tau]_\nu$ for the equivalence class modulo
$\nu$-almost-everywhere equality, and define its quotient image
\[
\overline{\mathcal M}=\{[\tau]_\nu:\tau\in\mathcal M\}
\subset L^2(\nu).
\]
Empirical objectives are evaluated on maps in $\mathcal M$, including
their values at observed scores. Population projection takes place in
$\overline{\mathcal M}$. Choosing actual representatives specifies values
at observed scores and keeps empirical fitting distinct from population
uniqueness modulo $\nu$.

\begin{theorem}[Monotone projection and the remaining gap]
\label{thm:monotone}
The set $\overline{\mathcal M}$ is nonempty, closed, and convex.
The projection of $[q]_\nu$ onto it is a unique class
$[\tau_{\mathrm{iso}}]_\nu$. Choose and fix a representative
$\tau_{\mathrm{iso}}\in\mathcal M$ of this class. With
$G=\|q-\tau_{\mathrm{iso}}\|_2^2$,
\begin{align}
\inf_{\tau\in\mathcal M} L(\tau\circ s_0)&=N+I+G,\\
\|\tau-\tau_{\mathrm{iso}}\|_2^2
&\le\|\tau-q\|_2^2-G\quad(\tau\in\mathcal M).
\label{eq:projection-bound}
\end{align}
The gap $G$ is zero if and only if $q$ has a nondecreasing version.
The maximal improvement over the unchanged score is
$\|\mathrm{id}-q\|_2^2-G\ge0$, since the identity map belongs to
$\mathcal M$.
Moreover, for $T=\tau_{\mathrm{iso}}(P)$,
\begin{equation}
\E[Y\mid T]=T,\qquad
\E[\Var(\mstar(Z)\mid T)]=I+G.
\label{eq:isotonic-pooling}
\end{equation}
Thus $G$ is additional output grouping loss, even though the projected
score is scalar-calibrated.
\end{theorem}
\begin{proof}
Constants belong to $\mathcal M$, and convex combinations preserve
bounds and order, so $\overline{\mathcal M}$ is nonempty and convex.
For closedness, take an $L^2(\nu)$-convergent sequence in
$\overline{\mathcal M}$, choose representatives in $\mathcal M$, and
extract an almost-everywhere convergent subsequence. Its pointwise limit
on a full-measure set is bounded and nondecreasing there, and extends
monotonically to $[0,1]$. Values at atoms are retained; right-continuity
is not imposed. The limit class therefore belongs to
$\overline{\mathcal M}$. The Hilbert-space projection theorem on this
closed convex set gives existence, uniqueness modulo $\nu$, and
$\langle q-\tau_{\mathrm{iso}},\tau-\tau_{\mathrm{iso}}\rangle\le0$.
Expanding $\|\tau-q\|_2^2$ proves \eqref{eq:projection-bound}; combining
with Theorem~\ref{thm:decomp} proves the optimum and improvement formula.
Closedness gives the characterization of $G=0$.

To establish output calibration without a continuity assumption on $\nu$,
write $T=\tau_{\mathrm{iso}}(P)$ and $R=q(P)-T$. For any Lipschitz
function $h$ compactly supported in $(0,1)$, the maps
$v\mapsto v\pm\epsilon h(v)$ are nondecreasing and remain in $[0,1]$
for sufficiently small $\epsilon>0$. Composing them with
$\tau_{\mathrm{iso}}$ gives feasible perturbations in both directions.
The projection inequality implies $\E[R h(T)]=0$. Such functions
determine the finite signed residual measure on $(0,1)$, so
$\E[R\mid T]=0$ on that interval, including its atoms.

For the endpoints, use the feasible maps $\max(T,a)$ and $\min(T,1-a)$.
Divide their variational inequalities by $a>0$ and let $a\downarrow0$.
Dominated convergence gives
\[
\E[q(P)\mathbf1\{T=0\}]\le0,\qquad
\E[(1-q(P))\mathbf1\{T=1\}]\le0.
\]
Both integrands are nonnegative; hence both expectations vanish.
It follows that $\E[q(P)\mid T]=T$ everywhere almost surely, and the
tower property gives $\E[Y\mid T]=\E[\mstar\mid T]=T$.
Finally, conditional-mean orthogonality yields
\[
\E[\Var(\mstar\mid T)]
=\E[(\mstar-T)^2]
=\E[(\mstar-q(P))^2]+\E[(q(P)-T)^2]=I+G.
\]
This connects isotonic pooling to standard calibration and grouping-loss
decompositions \citep{dimitriadis2021reliability,perezlebel2023grouping}.
\end{proof}

\paragraph{Information and ordering are distinct.}
Put equal mass on three states with
$P=(0,1/2,1)$ and $Y=\mstar=q=(0,1,1/2)$. Spearman correlation is
$1/2>0$, but $q$ is not nondecreasing. Its isotonic projection is
$(0,3/4,3/4)$, so $I=0$ and $G=1/24>0$. The score retains all
predictive information here; only the monotonicity restriction prevents
perfect prediction.

\paragraph{Information versus learnability.}
\label{app:learnability}
The term $I$ uses all measurable scalar predictors. For a restricted
class $\mathcal A$ of scalar maps, the best risk is instead
$N+I+\inf_{\tau\in\mathcal A}\|\tau-q\|_2^2$, before estimation or
optimization error. On a finite or countable full-measure state set, an
injective score can retain all information in $Z$, giving $I=0$, even when
learning a useful restricted remapping is difficult. Better performance by
a fitted feature-dependent predictor therefore need not establish $I>0$.

\subsection{Probe-based comparisons of candidate repairs}
\label{app:repair-comparisons}

A probe supplies a fitted comparison target. We first separate its error
from oracle excess risk, then compare two scores directly. The paired
calculation cancels terms shared by the candidates and is therefore more
informative for a repair decision than two separate absolute-error bounds.

For fixed scores $s$ and $g$ under the same population law, define
\emph{Probe-Relative Calibration Error} by
\[
\PRCE(s;g)=\E[(s(Z)-g(Z))^2].
\]
This measures squared disagreement with $g$. The oracle identity
$\PRCE(s;\mstar)=L(s)-L(\mstar)$ follows from
Theorem~\ref{thm:prce-brier}.

\begin{theorem}[Fixed-probe discrepancy and oracle excess risk]
\label{thm:probe-gap}
Under \ref{ass:information}, for fixed $s,g:\cZ\to[0,1]$, let
$e=g-\mstar$. Then
\begin{align}
\PRCE(s;g)-\PRCE(s;\mstar)
 &=\|e\|_2^2-2\E[(s-\mstar)e],\\
\left|\PRCE(s;g)-\PRCE(s;\mstar)\right|
 &\le2\E|e|\le2\|e\|_2.
\label{eq:probe-bound}
\end{align}
An independent evaluation average is unbiased for $\PRCE(s;g)$,
conditional on the fitted $s$ and $g$.
\end{theorem}
\begin{proof}
Expand $(s-g)^2=(s-\mstar-e)^2$ for the first identity. For the bound,
factor the difference of squares:
$|(s-g)^2-(s-\mstar)^2|=|e|\,|g+\mstar-2s|\le2|e|$.
Cauchy--Schwarz bounds its expectation. Unbiasedness for the fixed-probe
quantity follows by taking the expectation of each independently evaluated
summand. Equal-prompt averages have the same property under
\eqref{eq:prompt-law}. This statement concerns a fixed candidate; selection
on the comparison sample requires simultaneous control, given below.
\end{proof}

\begin{proposition}[Pairwise transfer from probe agreement to risk]
\label{prop:pairwise-risk}
Under \ref{ass:information}, fix $s_a,s_b,g:\cZ\to[0,1]$. Write
$D_g(a,b)=\PRCE(s_a;g)-\PRCE(s_b;g)$ and
$\Delta_{ab}=L(s_a)-L(s_b)$. Then
\begin{align}
D_g(a,b)-\Delta_{ab}
 &=-2\E[(s_a-s_b)(g-\mstar)],\label{eq:pairwise-identity}\\
|D_g(a,b)-\Delta_{ab}|
 &\le2\|s_a-s_b\|_2\|g-\mstar\|_2.\label{eq:pairwise-bound}
\end{align}
In particular, if $\rho_g\ge\|g-\mstar\|_2$ and
$D_g(a,b)<-2\rho_g\|s_a-s_b\|_2$, then $L(s_a)<L(s_b)$.
For scalar candidates $s_a=\tau_a(P)$ and $s_b=\tau_b(P)$, let
$b_g(P)=\E[g(Z)-Y\mid P]$. The sharper identity and bound are
\begin{equation}
D_g(a,b)-\Delta_{ab}=-2\langle\tau_a-\tau_b,b_g\rangle_{L^2(\nu)},
\qquad
|D_g(a,b)-\Delta_{ab}|\le2\|\tau_a-\tau_b\|_2\|b_g\|_2.
\label{eq:scalar-comparison}
\end{equation}
\end{proposition}
\begin{proof}
Subtract the expansions in Theorem~\ref{thm:probe-gap} for $a$ and $b$.
The shared term $\|g-\mstar\|_2^2$ cancels, while
Theorem~\ref{thm:prce-brier} gives
$\PRCE(s_a;\mstar)-\PRCE(s_b;\mstar)=\Delta_{ab}$.
This proves \eqref{eq:pairwise-identity}. Cauchy--Schwarz gives
\eqref{eq:pairwise-bound}; the strict margin then implies
$\Delta_{ab}\le D_g(a,b)+2\rho_g\|s_a-s_b\|_2<0$.
For scalar candidates, condition the product in
\eqref{eq:pairwise-identity} on $P$ and use
$\E[g-\mstar\mid P]=\E[g-Y\mid P]=b_g$.
\end{proof}

The identity isolates the direction of probe error relevant to a comparison.
If $g-\mstar$ is orthogonal to $s_a-s_b$, the two differences agree even
when the probe is imperfect. Otherwise their signs can differ. This is a
population statement about squared correctness risk, with a common target
and law for both candidates.

\begin{corollary}[Uncorrected comparisons: a bias-dependent benchmark]
\label{cor:repair-selection}
Under \ref{ass:information} and the cluster sampling in \ref{ass:audit},
condition on a jointly fitted probe $g$ and $M\ge2$ candidates
$s_1,\ldots,s_M$, all frozen independently of the $J$ comparison clusters.
All population quantities in this corollary use $\mathbb Q$, including
$\mstar=\E_{\mathbb Q}[Y\mid Z]$. Define
\[
\widehat R_c^g=\frac1J\sum_{j=1}^J\frac1{K_j}
 \sum_{i=1}^{K_j}(s_c(Z_{ji})-g(Z_{ji}))^2,
\qquad
\widehat D_{ab}=\widehat R_a^g-\widehat R_b^g.
\]
Write $e_g=\|g-\mstar\|_2$ and $d_{ab}=\|s_a-s_b\|_2$.
For $0<\alpha<1$, set
$t_{J,M,\alpha}=\sqrt{(2/J)\log(M(M-1)/\alpha)}$.
With conditional probability at least $1-\alpha$, simultaneously for
every pair,
\begin{equation}
|\widehat D_{ab}-[L(s_a)-L(s_b)]|
\le2e_gd_{ab}+t_{J,M,\alpha}\le2e_g+t_{J,M,\alpha}.
\label{eq:finite-comparison}
\end{equation}
On the same event, any measurable empirical minimizer
$\hat c\in\argmin_{1\le c\le M}\widehat R_c^g$ satisfies
\begin{equation}
L(s_{\hat c})-\min_{1\le c\le M}L(s_c)
\le2e_g+t_{J,M,\alpha}.
\label{eq:selection-regret}
\end{equation}
\end{corollary}
\begin{proof}
For each pair, the within-prompt average of squared-loss differences lies
in $[-1,1]$, has expectation $D_g(a,b)$, and is independent across prompts.
Hoeffding's inequality \citep{hoeffding1963} gives
\[
\Pr\{|\widehat D_{ab}-D_g(a,b)|>t\}
\le2\exp(-Jt^2/2).
\]
There are $M(M-1)/2$ unordered pairs; reversing a pair only changes its
sign. A union bound with the displayed $t_{J,M,\alpha}$ and
Proposition~\ref{prop:pairwise-risk} proves
\eqref{eq:finite-comparison}. Since bounded scores give $d_{ab}\le1$,
the second inequality follows. If $c^\star$ minimizes true risk within
the family, empirical optimality gives $\widehat D_{\hat c c^\star}\le0$.
Apply the simultaneous event to this selected pair to obtain
\eqref{eq:selection-regret}. Bounded cluster averages require no moment
assumption on the finite cluster sizes.
\end{proof}

\paragraph{Interpreting a conditional comparison bound.}
The corollary allows selection \emph{within} the frozen family; fitting or
tuning candidates on these comparison prompts requires additional control,
even if their labels are unused. The probe error $e_g$ remains unknown.
For example, an honest upper confidence bound $U_g\ge L(g)$ from labeled
validation under the same law supplies
$e_g^2=L(g)-N\le U_g$. If this bound has coverage $1-\beta$, using
$\sqrt{U_g}$ in place of $e_g$ gives joint coverage at least
$1-\alpha-\beta$ by a union bound. The radius is conservative because
labeled probe risk retains outcome noise. Thus the result states the
ingredients needed for a reliable surrogate comparison; it neither
estimates probe accuracy from unlabeled agreement nor demonstrates a
selection advantage over ordinary labeled validation.

\subsection{PAIC: fixed-probe target and empirical consistency}
\label{app:paic-estimation}

Before studying increasing training and calibration samples, consider the
population repair obtained from one fixed probe. This separates error in
the anchor from error in fitting a monotone curve to finitely many examples.

\begin{proposition}[Fixed-probe projection and additional risk]
\label{prop:anchor-projection}
Under \ref{ass:information}, fix $g:\cZ\to[0,1]$ and let
$q_g(P)=\E[g(Z)\mid P]$. Choose a representative $\tau_g\in\mathcal M$
of the projection of $q_g$ onto $\overline{\mathcal M}$. This is also
the unique population minimizing class for $\E[(g(Z)-\tau(P))^2]$.
Then
\begin{align}
\|\tau_g-\tau_{\mathrm{iso}}\|_{L^2(\nu)}
 &\le\|q_g-q\|_{L^2(\nu)}\le\|g-\mstar\|_2,
 \label{eq:anchor-stability}\\
0\le L(\tau_g\circ s_0)-(N+I+G)
 &\le\|q_g-q\|_{L^2(\nu)}^2\le\|g-\mstar\|_2^2.
 \label{eq:anchor-risk}
\end{align}
For $T_g=\tau_g(P)$, $\E[g\mid T_g]=T_g$, while
\begin{equation}
\E[Y\mid T_g]-T_g=\E[\mstar-g\mid T_g].
\label{eq:anchor-calibration}
\end{equation}
\end{proposition}
\begin{proof}
Conditioning on $P$ gives
$\E[(g-\tau(P))^2]=\E[\Var(g\mid P)]+\|q_g-\tau\|_2^2$,
so the minimizing class is the projection of $q_g$.
Write $d=\tau_g-\tau_{\mathrm{iso}}$ and $h=q_g-q$ as elements of
$L^2(\nu)$. Adding the variational inequalities for the two projections
gives $\|d\|_2^2\le\langle d,h\rangle$ and hence
$\|d\|_2\le\|h\|_2$. Conditional-expectation contraction gives
$\|h\|_2=\|\E[g-\mstar\mid P]\|_2\le\|g-\mstar\|_2$.

Apply the projection inequality with target $q_g$ and competitor
$\tau_{\mathrm{iso}}$ to obtain
$\|\tau_g-q_g\|_2^2-\|\tau_{\mathrm{iso}}-q_g\|_2^2
\le-\|d\|_2^2$. By Theorem~\ref{thm:decomp},
\begin{align*}
L(\tau_g\circ s_0)-L(\tau_{\mathrm{iso}}\circ s_0)
 &=\|\tau_g-q\|_2^2-\|\tau_{\mathrm{iso}}-q\|_2^2\\
 &\le2\langle d,h\rangle-\|d\|_2^2
 =\|h\|_2^2-\|h-d\|_2^2\le\|h\|_2^2.
\end{align*}
The lower bound follows from optimality of $\tau_{\mathrm{iso}}$.
Apply the output-calibration argument of Theorem~\ref{thm:monotone} with
$g$ in place of $Y$ to obtain $\E[g\mid T_g]=T_g$. Subtract this
identity from $\E[Y\mid T_g]=\E[\mstar\mid T_g]$ for
\eqref{eq:anchor-calibration}.
\end{proof}

Consequently, the fixed-probe population repair improves on the base score
whenever $\|\mathrm{id}-q\|_2^2-G>\|q_g-q\|_2^2$: the available
monotone gain then exceeds the bound on anchor error. This is a sufficient
population condition. A fitted finite-sample repair also incurs calibration
error, addressed next.

\paragraph{Measurable population and empirical maps.}
A regular conditional kernel $K_Z(p,dz)$ gives the jointly measurable
version $q_g(t,p)=\int g(t,z)K_Z(p,dz)$. The space $L^2(\nu)$ is
separable and projection onto $\overline{\mathcal M}$ is continuous, so
the projected class is measurable in training data. One jointly measurable
monotone representative of a projected class $f$ is its prefix essential
supremum, equivalently
\[
f^\sharp(p)=\sup\left(\{0\}\cup
 \left\{r\in\mathbb Q\cap[0,1]:
        \int_{[0,p]}(f(u)-r)_+\,d\nu(u)>0\right\}\right).
\]
The integrals depend only on the class and are jointly measurable in the
class and $p$. For a monotone class this representative agrees
$\nu$-almost everywhere and retains every atom; empty prefixes take value
zero. Empirical minimizers, in contrast, retain their actual observed values.
Stable sorting, pooling tied scores, and a specified monotone interpolation
give a jointly measurable empirical map.

\begin{lemma}[Uniform loss deviation for an independently fitted anchor]
\label{lem:monotone-gc}
Let $\mathbb Q_n$ be the empirical measure of $n$ i.i.d. calibration
observations under \ref{ass:information}. For each fixed measurable
$g:\cZ\to[0,1]$, there is a bounded measurable $\Delta_n(g)$ with
\[
\sup_{a\in\mathcal M}
 \left|(\mathbb Q_n-\mathbb Q)(a(P)-g(Z))^2\right|
\le\Delta_n(g),\qquad
\E_{\rm cal}\Delta_n(g)\le\frac{C_0}{\sqrt n},\quad C_0=\frac{25}{2}.
\]
The constant is uniform over fixed bounded anchors and all score laws,
including laws with atoms. The statement therefore applies conditionally
to the independent, possibly changing $g_n$ in \ref{ass:probe}.
\end{lemma}
\begin{proof}
Let $\mathcal H_{\uparrow}$ contain the open and closed upper half-lines in
$[0,1]$, including the empty and full sets. For fixed measurable
$0\le w(Z)\le1$, define
\[
D_n(w)=\sup_{H\in\mathcal H_{\uparrow}}
 |(\mathbb Q_n-\mathbb Q)w(Z)\mathbf1\{P\in H\}|.
\]
One-sided limits of rational cuts recover open and closed half-lines, so
the supremum is measurable. Symmetrization with independent Rademacher
signs $\varepsilon_i$ and decreasing-score ordering gives
\[
\E D_n(w)\le\frac2n\E\max_{0\le l\le n}
 \left|\sum_{i=1}^l\varepsilon_{(i)}w(Z_{(i)})\right|
\le\frac4{\sqrt n}.
\]
For the last inequality, condition on the observations. The weighted
partial sums form a martingale, and Doob's $L^2$ maximal inequality bounds
their expected maximum by $2(\sum_i w(Z_i)^2)^{1/2}\le2\sqrt n$.
Ties only reduce the attainable prefixes.
For any bounded nondecreasing $a$, every superlevel set
$\{p:a(p)>u\}$ is such a half-line. The layer-cake identities
\[
a(p)^2=\int_0^1 2u\,\mathbf1\{a(p)>u\}\,du,
\qquad
a(p)=\int_0^1\mathbf1\{a(p)>u\}\,du
\]
imply a uniform bound with
\[
\Delta_n(g)=D_n(1)+2D_n(g)+|(\mathbb Q_n-\mathbb Q)g^2|.
\]
The final term has expectation at most $1/(2\sqrt n)$ because
$g^2\in[0,1]$. Adding the bounds gives $C_0=25/2$.
Uniformity is over monotone maps for each fixed anchor, not over anchors
chosen using the calibration observations. Training independence allows
conditioning on $g_n$ before applying the bound.
\end{proof}

\begin{theorem}[PAIC risk and consistency under score-conditional bias]
\label{thm:paic}
Under \ref{ass:information} and \ref{ass:probe}, let
$\tauhat\in\mathcal M$ be a jointly measurable empirical minimizer of
$n^{-1}\sum_i(\tau(P_i)-g_n(Z_i))^2$ over actual maps
$\tau\in\mathcal M$. Let $\tau_{g_n}$ be the population projection of
$q_{g_n}$, and define, conditionally on training,
\[
R_{g_n}(a)=\mathbb Q(g_n(Z)-a(P))^2,\qquad
u_n=R_{g_n}(\tauhat)-R_{g_n}(\tau_{g_n})\ge0.
\]
Then almost surely
\begin{equation}
\|\tauhat-\tau_{\mathrm{iso}}\|_2^2
\le L(\tauhat\circ s_0)-(N+I+G)
\le(\|b_n\|_2+\sqrt{u_n})^2,
\qquad u_n\le2\Delta_n(g_n).
\label{eq:paic-risk-transfer}
\end{equation}
In particular,
\begin{equation}
\E_{\mathrm{train,cal}}\big[L(\tauhat\circ s_0)-(N+I+G)\big]
\le\left(\sqrt{\delta_n}+\sqrt{\frac{2C_0}{\sqrt n}}\right)^2
\le2\delta_n+\frac{4C_0}{\sqrt n}.
\label{eq:paic-consistency}
\end{equation}
Thus $\delta_n\to0$ implies risk convergence to $N+I+G$ and squared
$L^2(\nu)$ convergence of the map to $\tau_{\mathrm{iso}}$, both in
mean and in probability. The rate is a non-sharp uniform consistency bound,
not an observable confidence interval or a guarantee at $\nu$-null scores.
\end{theorem}
\begin{proof}
Condition on training. Empirical optimality and
Lemma~\ref{lem:monotone-gc} imply $u_n\le2\Delta_n(g_n)$.
Write $d=\tau_{g_n}-\tau_{\mathrm{iso}}$, $e=\tauhat-\tau_{g_n}$,
and $b=b_n$. Projection at target $q_{g_n}$ gives
\[
\|e\|_2^2\le u_n,\qquad
R_{g_n}(\tau_{g_n})-R_{g_n}(\tau_{\mathrm{iso}})\le-\|d\|_2^2.
\]
Differences of label risks equal differences of probe risks plus
$2\langle\text{map difference},b\rangle$. Therefore
\begin{align*}
L(\tauhat\circ s_0)-L(\tau_{\mathrm{iso}}\circ s_0)
&\le u_n-\|d\|_2^2+2\langle d+e,b\rangle\\
&=u_n+\|b\|_2^2-\|d-b\|_2^2+2\langle e,b\rangle\\
&\le(\sqrt{u_n}+\|b\|_2)^2.
\end{align*}
Projection at the true target $q$ gives the lower inequality in
\eqref{eq:paic-risk-transfer}. Taking expectations, using
$\E u_n\le2C_0/\sqrt n$ and Cauchy--Schwarz on
$\E[\|b_n\|_2\sqrt{u_n}]$ proves \eqref{eq:paic-consistency}.
Markov's inequality gives convergence in probability. This argument uses
projection inequalities rather than linearity of the projection.
\end{proof}

\paragraph{Sharper fitting rate and scalar-calibration minimax order.}
\label{rem:minimax}
Condition on probe training and write $\ell_a=(a(P)-g_n(Z))^2$.
The loss differences $f_a=\ell_a-\ell_{\tau_{g_n}}$, $a\in\mathcal M$, satisfy
$\E f_a^2\le4\|a-\tau_{g_n}\|_2^2\le4\E f_a$: squared loss is
$2$-Lipschitz on $[0,1]$, and projection at $q_{g_n}$ gives the second
inequality. The monotone class has $L^2(\mu)$ bracketing entropy of order
$1/\eta$ uniformly over probability measures $\mu$
\citep[Thm.~2.7.5, p.~159]{vdvw1996}. For the loss differences with
$\E f_a^2\le r$, chaining bounds the expected local Rademacher complexity
by a constant times $n^{-1/2}r^{1/4}+n^{-2/3}$. This sub-root bound has
fixed point of order $n^{-2/3}$; \citet[Thm.~3.3(1)]{bartlett2005local}
therefore gives $\E[u_n\mid\mathcal T_n]\le Cn^{-2/3}$ for a universal $C$,
without requiring $q$ or $q_{g_n}$ to be monotone. Related fixed-design
bounds for error relative to the monotone projection under misspecification
are given by \citet[Thm.~6.1]{chatterjee2015risk}; here localization
controls the full excess probe loss $u_n$.
Theorem~\ref{thm:paic} therefore gives
\[
\E L(\tauhat\circ s_0)-(N+I+G)
\le\bigl(\sqrt{\delta_n}+\sqrt C\,n^{-1/3}\bigr)^2.
\]
For i.i.d. prompt clusters, put $F_a(C)=K^{-1}\sum_i f_a(Z_i)$.
Jensen gives $\E F_a^2\le4\|a-\tau_{g_n}\|_2^2\le4\E F_a$ and
$\|F_a-F_b\|_{L^2(\mathbb Q_J)}\le2\|a-b\|_{L^2(\mu_J)}$, where
$\mu_J=J^{-1}\sum_jK_j^{-1}\sum_i\delta_{P_{ji}}$ is a probability measure.
Thus the same entropy argument applies with $J$ in place of $n$, under the
prompt-weighted law and calibration independence conditions of
Appendix~\ref{app:paic-clusters}.

For a precise minimax statement, consider the experiment observing only
$D_n=((P_i,V_i))_{i=1}^n$, i.i.d., with $V_i=g(Z_i)\in[0,1]$;
an estimator returns a bounded nondecreasing map $a_{D_n}(P)$.
No target outcomes, additional target draws or population oracle are
available. Any auxiliary input is independent of $D_n$ and has the same
law for every population in the comparison. Let $\mathfrak P_\delta$
be the class of laws of $(P,V,Y)\in[0,1]^3$ with
$\|\E[V-Y\mid P]\|_2^2\le\delta$, for $0\le\delta\le1/4$.
Write $L_{\mathbb Q}(a)=\E_{\mathbb Q}(a(P)-Y)^2$ and
$L_{\mathbb Q}^*=\inf_{a\in\mathcal M}L_{\mathbb Q}(a)=N+I+G$. Then
\[
\inf_{a_{D_n}\in\mathcal M}\sup_{\mathbb Q\in\mathfrak P_\delta}
\E\!\left[L_{\mathbb Q}(a_{D_n})-L_{\mathbb Q}^*\right]
\asymp\delta+n^{-2/3}.
\]
The infimum is over measurable estimators; expectation includes calibration
and estimator randomness.
The upper bound follows above. Both lower constructions use the same
fixed representation $Z=(P,V)$ and known probe $g(P,V)=V$, with auxiliary
information identical across alternatives. For bias, take $P=V=1/2$
and $Y\sim\mathrm{Bernoulli}(1/2\pm\sqrt\delta)$. The observed law is
identical under both signs, and the average excess risk is
$\E[(a_{D_n}(1/2)-1/2)^2]+\delta\ge\delta$
(Le Cam's two-point method; \citealp{yu1997assouad}).

For fitting, put $k=\lceil n^{1/3}\rceil$ and let $P$ be uniform on
$p_j=j/(k+1)$, $j=1,\ldots,k$. Set $Y=V$ and
$V\mid P=p_j\sim\mathrm{Bernoulli}(\theta_j)$, with
$\theta_j=1/4+(j-1/2)/(2k)+\varepsilon_j/(8k)$,
$\varepsilon_j\in\{-1,1\}$. These curves are monotone, lie in $[1/4,3/4]$,
and have zero probe bias. Scalar excess risk is
$k^{-1}\sum_j(a_{D_n}(p_j)-\theta_j)^2$.
Neighboring sign models differ by $1/(4k)$ at one score and have
$\mathrm{KL}(D_n^+,D_n^-)\le n/(3k^3)\le1/3$.
Their testing error is bounded away from zero, so summing the coordinate
errors gives a lower bound of order $k^{-2}\asymp n^{-2/3}$
\citep{yu1997assouad}, the isotonic minimax order
\citep{zhang2002risk,guntuboyina2018shape}.
The maximum of the two lower bounds controls their sum up to a constant.
This is a worst-case result for the specified scalar-calibration experiment,
not for serving the richer probe directly, and it supplies neither a bias
certificate nor usable constants for a particular fitted probe.

\begin{corollary}[PAIC's limit and the role of the anchor]
\label{cor:paic}
Under Theorem~\ref{thm:paic}, PAIC's risk converges to $N+I+G$.
It converges to oracle PRCE zero when $I=G=0$. If $q$ is nondecreasing,
the remaining excess risk is $I$; otherwise pooling contributes $G$.
Isotonic regression on labels has the same population minimizing class
in $\overline{\mathcal M}$ as isotonic regression on $\mstar$.
\end{corollary}
\begin{proof}
For any two bounded maps, the difference of their squared losses is
bounded in expectation by twice their $L^1$ distance. Hence convergence
in \eqref{eq:paic-consistency} gives convergence of risk to
$L(\tau_{\mathrm{iso}})=N+I+G$. The final claim follows from
$\E(Y-a(P))^2=N+\E(\mstar-a(P))^2$: the term $N$ does not depend on
$a$. Probe anchoring can change finite-sample bias and variance, but the
population projection argument does not guarantee dominance over
label-anchored isotonic regression \citep{zadrozny2002}.
\end{proof}

The weaker assumption can hold without full probe consistency: take
$Z=Y$ a fair bit, $P=1/2$, and $g_n=1/2$. Then $b_n=0$ and
$\delta_n=0$, but $\|g_n-\mstar\|_2^2=1/4$ for every $n$.
The probe is adequate for the scalar target while ignoring within-score
information. Training independence remains essential: a probe that equals
one only at continuously distributed calibration observations can have zero
population error against $\mstar=0$ while producing all-one fit targets.

\paragraph{Extension to repeated answers.}
\label{app:paic-clusters}
For i.i.d. calibration clusters, independent of probe training, replace
$\mathbb Q_n$ by the prompt-level empirical law
$\mathbb Q_J f=J^{-1}\sum_j K_j^{-1}\sum_i f(Z_{ji},Y_{ji})$ and use
the target and norms under \eqref{eq:prompt-law}. For independent uniform
answer indices $I_j$ within the observed clusters, conditional Jensen gives
\[
\sup_H\left|\sum_j\varepsilon_j\frac1{K_j}
 \sum_i w(Z_{ji})\mathbf1\{P_{ji}\in H\}\right|
\le\E_{I\mid C}\sup_H\left|\sum_j\varepsilon_j
 w(Z_{jI_j})\mathbf1\{P_{jI_j}\in H\}\right|.
\]
After averaging over clusters and indices, the selected answers are i.i.d.
under $\mathbb Q$. The same symmetrization bound applies with $J$ in place
of $n$; the $g^2$ term is also a bounded independent cluster average.
Hence Lemma~\ref{lem:monotone-gc} and Theorem~\ref{thm:paic} hold with
the same $C_0$ and $J$ as sample size. Finite random cluster sizes need
no moment bound. All training/calibration independence conditions remain.
Pooling out-of-fold predictions from different fitted representations requires a separate aggregation argument and compatible feature geometry.

\subsection{Residual-corrected risk comparisons and recommendations}
\label{app:corrected-comparisons}

The following construction applies prediction-powered mean estimation
\citep{angelopoulos2023ppi} to paired losses. Unlike an uncorrected PRCE
comparison, it uses labels to estimate the probe-error direction relevant
to the decision. It does not require probe consistency.

Condition on a fitted probe $g$ and $M\ge2$ bounded candidate scores,
including $s_0$. Take two mutually independent samples from the same
prompt-cluster law in \ref{ass:audit}: $J_L\ge2$ labeled prompts and
$J_U\ge2$ prompts whose labels are unobserved. Both samples are independent
of all fitting, feature selection, and candidate construction. For a generic
cluster $C=((Z_i,Y_i))_{i=1}^K$, define
\begin{align}
A_{ab}(C)&=\frac1K\sum_i\big[(s_a(Z_i)-g(Z_i))^2
                          -(s_b(Z_i)-g(Z_i))^2\big],\\
D_{ab}(C)&=\frac1K\sum_i\big[(Y_i-s_a(Z_i))^2
                          -(Y_i-s_b(Z_i))^2\big].
\label{eq:cluster-contrasts}
\end{align}
Then $\E A_{ab}=D_g(a,b)$ and $\E D_{ab}=\Delta_{ab}$ under the
common induced law $\mathbb Q$. The exact residual correction is
\begin{equation}
D_{ab}(C)-A_{ab}(C)
=\frac2K\sum_i(s_a(Z_i)-s_b(Z_i))(g(Z_i)-Y_i).
\label{eq:residual-correction}
\end{equation}
For a fixed $\lambda\in[0,1]$, write
\begin{equation}
\widehat\Delta_{ab}^{\lambda}
=\overline{(D_{ab}-\lambda A_{ab})}_L
 +\lambda\overline A_{U,ab}.
\label{eq:weighted-comparison}
\end{equation}
Here and below bars average prompts. The choice $\lambda=1$ is the
residual-corrected comparison; $\lambda=0$ is ordinary paired labeled
validation. For each fixed $\lambda$, cancellation of the two proxy
expectations proves $\E\widehat\Delta_{ab}^{\lambda}=\Delta_{ab}$.
In contrast, $\PRCE(g;g)=0$ automatically makes $g$ an uncorrected
PRCE minimizer if it is an eligible candidate. Independent probe fitting
does not remove this self-reference.

\paragraph{An observable mean bound.}
For $n\ge2$, unbiased sample variance $v$ with divisor $n-1$, known
range diameter $R$, and $0<\epsilon<1$, let
\begin{equation}
B(n,v,R,\epsilon)
=\sqrt{\frac{2v\log(4/\epsilon)}{n}}
 +\frac{7R\log(4/\epsilon)}{3(n-1)}.
\label{eq:empirical-bernstein}
\end{equation}
For i.i.d. observations in an interval of diameter $R$, their mean is
within this radius of its expectation with probability at least
$1-\epsilon$. This follows by rescaling Theorem~4 of
\citet{maurer2009empirical} and applying its one-sided bound to both
tails at error level $\epsilon/2$. Their pairwise sample variance equals
the usual variance with divisor $n-1$.

\begin{theorem}[Simultaneous corrected comparisons]
\label{thm:corrected-comparison}
Under the fixed-fitting and two-pool assumptions above, let
$\Lambda\subset[0,1]$ be a nonempty finite set fixed before comparison.
Set $H=M(M-1)/2$ and $\epsilon=\alpha/(2H|\Lambda|)$ for
$0<\alpha<1$. Let $v_{L,ab}^{\lambda}$ be the sample variance of
$D_{ab}-\lambda A_{ab}$ across labeled prompts, and $v_{U,ab}$ the
sample variance of $A_{ab}$ across unlabeled prompts. Define
\begin{equation}
r_{ab}^{\lambda}
=B(J_L,v_{L,ab}^{\lambda},2(1+\lambda),\epsilon)
 +\lambda B(J_U,v_{U,ab},2,\epsilon).
\label{eq:corrected-radius}
\end{equation}
Conditionally on fitting, with probability at least $1-\alpha$,
\begin{equation}
|\widehat\Delta_{ab}^{\lambda}-\Delta_{ab}|
\le r_{ab}^{\lambda}
\quad\text{for every candidate pair and every }\lambda\in\Lambda.
\label{eq:corrected-coverage}
\end{equation}
The event covers selection within the frozen family and grid. It does not
make a post-selected point estimate unbiased.
\end{theorem}
\begin{proof}
Squared losses lie in $[0,1]$, so $A_{ab},D_{ab}\in[-1,1]$ and
$D_{ab}-\lambda A_{ab}\in[-(1+\lambda),1+\lambda]$.
Apply \eqref{eq:empirical-bernstein} to the labeled and unlabeled prompt
averages for every unordered pair and every grid value. There are at most
$2H|\Lambda|$ events, each failing with probability at most $\epsilon$.
Their simultaneous success has probability at least $1-\alpha$ by a
union bound; events for different pairs need not be independent.
The triangle inequality, expectation cancellation in
\eqref{eq:weighted-comparison}, and reversal of pair signs prove
\eqref{eq:corrected-coverage}. Arbitrary within-prompt dependence and
finite random $K$ are allowed because contributions are bounded cluster
averages.
\end{proof}

For the main construction use $\Lambda=\{1\}$, giving
$\epsilon=\alpha/(2H)$ and labeled/unlabeled range diameters $4$ and $2$.
For a prespecified labeled-only procedure no unlabeled sample is needed:
use $\overline D_L$ with radius $B(J_L,v_D,2,\alpha/H)$.
Comparing several separately reported methods does not authorize selecting
one afterward while retaining its unadjusted coverage level.

\begin{corollary}[A recommendation with controlled false improvement]
\label{cor:certified-repair}
Fix $\eta_{\min}\ge0$. For a fixed permitted weight, select the candidate
$a\ne0$ with smallest
$U_{a0}=\widehat\Delta_{a0}^{\lambda}+r_{a0}^{\lambda}$, using a
predeclared tie order. Recommend it only when $U_{a0}<-\eta_{\min}$;
otherwise retain $s_0$ and report an inconclusive comparison. Under
Theorem~\ref{thm:corrected-comparison},
\[
\Pr\{\text{recommend }a\ne0\text{ with }
       L(s_a)\ge L(s_0)-\eta_{\min}\mid\text{fitting}\}\le\alpha.
\]
Selection of both the candidate and a weight in the jointly covered
$\Lambda$ obeys the same guarantee.
\end{corollary}
\begin{proof}
On the simultaneous event, $\Delta_{a0}\le U_{a0}<-\eta_{\min}$
for every accepted recommendation. A false declaration therefore implies
failure of that event. This is an unconditional false-declaration probability
over comparison samples, not a conditional error rate among recommendations.
\end{proof}

\paragraph{When the correction can help.}
For a fixed weight, independence of the two pools gives exactly
\begin{equation}
\Var(\widehat\Delta_{ab}^{\lambda})
=\frac{\Var(D_{ab}-\lambda A_{ab})}{J_L}
 +\frac{\lambda^2\Var(A_{ab})}{J_U}.
\label{eq:comparison-variance}
\end{equation}
Thus additional unlabeled prompts reduce the proxy contribution, while
probe usefulness depends on the residual contribution. If $\Var(A_{ab})>0$,
the population variance-minimizing weight in $[0,1]$ is
\[
\lambda_{ab}^{\star}
=\operatorname{clip}_{[0,1]}
 \frac{\operatorname{Cov}(D_{ab},A_{ab})}
 {(1+J_L/J_U)\Var(A_{ab})}.
\]
When $\Var(A_{ab})=0$, all weights have the same variance; choose zero.
This is the familiar control-variate calculation underlying power tuning
\citep{angelopoulos2024ppiplusplus}. A plug-in weight estimated on the
comparison data is not automatically covered by a fixed-weight theorem;
use an independently fixed weight or the simultaneous finite-grid result.
A common weight gives coherent differences of candidate-risk estimates;
pair-specific weights instead define pairwise comparisons.

The finite-sample radius is observable, not necessarily narrow. Its variance
term can adapt to small paired differences, but its range term can dominate
at small $J_L$. A valid known bound $\|s_a-s_b\|_\infty\le d\le1$
permits replacing the absolute bound $1$ for $A_{ab},D_{ab}$ by
$2d-d^2$: factor the difference of squares and maximize over two scores
at distance at most $d$ and a target in $[0,1]$. This can tighten the
range diameters, but an observed maximum does not bound unobserved values.
If labeled and unlabeled pools are the same observations with the same
weights, \eqref{eq:weighted-comparison} reduces algebraically to
$\overline D_L$ for every $\lambda$; the independent-pool variance
formula no longer applies. Neither correction nor a small probe error
guarantees a label-efficiency or runtime advantage.

\subsection{Directional adoption by betting}
\label{app:betting}

The intervals above cover every pairwise contrast in both directions. Adopting
a candidate requires less: evidence that it lowers risk relative to the
unchanged score. The following test uses a standard betting construction
\citep{waudbysmith2024betting}. It controls false adoption.

Condition on the fitted probe and the frozen candidate family, and let $H$
count the non-original candidates. For candidate $h$ and an independent
labeled prompt $j$ with $K_j$ answers, write
\begin{align*}
D_{jh}&=\frac1{K_j}\sum_{i=1}^{K_j}\big[(Y_{ji}-h(Z_{ji}))^2-(Y_{ji}-P_{ji})^2\big],\\
A_{jh}&=\frac1{K_j}\sum_{i=1}^{K_j}\big[(h(Z_{ji})-g(Z_{ji}))^2-(P_{ji}-g(Z_{ji}))^2\big],
\end{align*}
so that $\E D_{jh}=L(h)-L(P)$ under $\mathbb Q$. Let $\delta_h\in[0,1]$ satisfy
$\delta_h\ge\sup_z|h(z)-s_0(z)|$; this must be a certified bound over the whole
input domain, not an observed maximum. Two scores within $\delta_h$ of each
other and a target in $[0,1]$ have squared losses differing by at most
$B_h=\delta_h(2-\delta_h)$: factor the difference of squares and maximize, as
in Appendix~\ref{app:corrected-comparisons}. Hence $|D_{jh}|,|A_{jh}|\le B_h$.

\begin{proposition}[Adoption by betting]
\label{prop:betting}
Condition on fitting, and let the $n$ labeled prompts be i.i.d.\ and
independent of fitting. Fix $0<\alpha<1$ and a nonempty finite set
$\mathcal S\subset[0,1)$. For each non-original candidate with $B_h>0$, let
\[
E_h(s)=\prod_{j=1}^n\Big(1-s\,\frac{D_{jh}}{B_h}\Big),\qquad
E_h=\frac1{|\mathcal S|}\sum_{s\in\mathcal S}E_h(s).
\]
Then
\[
\Pr\big\{\exists h:\ L(h)\ge L(P)\ \text{and}\ E_h\ge H/\alpha\big\}\le\alpha.
\]
Consequently, the rule that certifies every candidate with $E_h\ge H/\alpha$,
adopts the certified candidate with the smallest mean $\overline D_h$ under a
fixed tie order, and otherwise retains $P$, adopts a candidate with
$L(h)\ge L(P)$ with probability at most $\alpha$.
\end{proposition}
\begin{proof}
Fix $h$ with $L(h)\ge L(P)$, so $\E D_{jh}\ge0$. For $s\in\mathcal S$, each
factor satisfies $1-sD_{jh}/B_h\ge1-s>0$ because $|D_{jh}|\le B_h$, and its
expectation is $1-s\E D_{jh}/B_h\le1$. Prompts are independent, so
$\E E_h(s)=\prod_j\E(1-sD_{jh}/B_h)\le1$, and averaging over $\mathcal S$ gives
$\E E_h\le1$. Markov's inequality gives $\Pr\{E_h\ge H/\alpha\}\le\alpha/H$,
and a union bound over at most $H$ such candidates proves the claim. A
candidate with $B_h=0$ coincides with $P$ and is never certified.
\end{proof}

The main text and experiments use $\mathcal S=\{0.05,0.1,0.25,0.5,0.9\}$.
The guarantee applies to one prespecified method and comparison; choosing a
method or budget after inspecting a sweep is not covered.

\paragraph{A proxy-assisted variant.}
Condition on fitting and prespecify disjoint groups. Each group contains
one labeled prompt and a fixed positive number $r$ of unlabeled prompts.
All these prompt clusters are mutually independent, sampled from the same
target cluster law, and independent of fitting; their prompt-weighted answer
law is $\mathbb Q$. For fixed $\lambda\in[0,1]$, form
\[
\widetilde D_{jh,\lambda}=D_{jh}-\lambda A_{jh}+\lambda\overline A^{\mathcal U}_{jh},
\]
where $\overline A^{\mathcal U}_{jh}$ averages $A$ over the group's unlabeled
prompts. Then $\E\widetilde D_{jh,\lambda}=\E D_{jh}$ and
$|\widetilde D_{jh,\lambda}|\le B_h+2\lambda\delta_h$: the residual part
$D-\lambda A=(1-\lambda)D+\lambda(D-A)$ is bounded by
$(1-\lambda)B_h+2\lambda\delta_h$, because $|D-A|\le2\delta_h$, and the
unlabeled part by $\lambda B_h$. Replacing $D_{jh}/B_h$ by
$\widetilde D_{jh,\lambda}/(B_h+2\lambda\delta_h)$ and averaging over a fixed
uniform mixture of $s\in\mathcal S$ and $\lambda\in\{0,.25,.5,.75,1\}$, the same
proof applies. The zero-weight component gives a valid fallback but pays a
mixture penalty.

\subsection{Kernel identities and their scope}
\label{app:kernel-definitions}

The kernel step maps a conditional-mean residual into a collection of
testable moments. For attribution, take the product kernel
$\widetilde k((z,y),(z',y'))=k(z,z')yy'$ in the distribution-matching
framework of \citet{marx2023distribution}. Compare the law of $(Z,Y)$
with a law of $(Z,\widetilde Y)$ having the same $Z$ marginal and
$\E[\widetilde Y\mid Z]=s(Z)$, with $\widetilde Y\in[0,1]$.
Their difference of mean embeddings is $\E[(Y-s(Z))\Phi(Z)]$.
Its squared norm is the SKCE below; for graded outcomes this specializes
distribution matching to conditional means.

Define the corresponding residual functional by
\begin{equation}
\SCF_k(s):=\sup_{\substack{f\in\cH_k\\\|f\|_{\cH_k}\le1}}
          |\E[(Y-s(Z))f(Z)]|.
\label{eq:scf-definition}
\end{equation}

\begin{theorem}[Residual norm identity and zero identification]
\label{thm:skce-scf}
Under \ref{ass:information} and the bounded-RKHS part of
\ref{ass:kernel}, put $r=Y-s(Z)$ and $\mu_s=\E[r\Phi(Z)]$.
For independent primed copies,
\[
\SKCE_k(s)=\E[rr'k(Z,Z')]
 =\|\mu_s\|_{\cH_k}^2=\SCF_k(s)^2.
\]
If $\cH_k$ is dense in $L^2(\mathbb Q_Z)$, then
$\SKCE_k(s)=0$ if and only if $s=\mstar$ almost surely.
\end{theorem}
\begin{proof}
The bound $\|r\Phi(Z)\|\le\sqrt\kappa$ makes the Bochner mean well
defined. The reproducing property and independence give
\[
\E[rr'k(Z,Z')]
=\E\langle r\Phi(Z),r'\Phi(Z')\rangle
=\langle\mu_s,\mu_s\rangle.
\]
For $f\in\cH_k$, $\E[rf(Z)]=\langle\mu_s,f\rangle$.
Cauchy--Schwarz gives a supremum of $\|\mu_s\|$ over the unit ball;
when $\mu_s\ne0$ it is attained at $\mu_s/\|\mu_s\|$, and when
$\mu_s=0$ every value is zero. These are the usual mean-embedding
identities \citep{gretton2012kernel,muandet2017kernel}.
Related kernel calibration criteria use distribution matching
\citep{marx2023distribution}; the identity here specifically concerns
the residual-weighted mean embedding.

Conditioning on $Z$ writes the moment as
$\E[(\mstar-s)f(Z)]$. If it vanishes for every RKHS function, density
and $\mstar-s\in L^2(\mathbb Q_Z)$ imply $\|\mstar-s\|_2^2=0$.
The converse is immediate. On a compact metric space, continuous
universality supplies this density through density of continuous
functions in $L^2$ \citep{sriperumbudur2011universality}. The relevant
residual measure is signed, so injectivity stated only for probability
measures is not, by itself, the required argument.
\end{proof}

\begin{theorem}[Upper bound by oracle PRCE and continuity]
\label{thm:skce-prce}
Under the bounded-RKHS conditions above, for bounded $Z$-measurable
scores $s,t$,
\begin{align}
0\le\SKCE_k(s)&\le\kappa\PRCE(s;\mstar),\\
|\sqrt{\SKCE_k(s)}-\sqrt{\SKCE_k(t)}|
&\le\sqrt\kappa\,\|s-t\|_2.
\label{eq:skce-continuity}
\end{align}
Universality is unnecessary for either statement.
\end{theorem}
\begin{proof}
$\E[(Y-\mstar)\Phi(Z)]=0$, so
$\mu_s=\E[(\mstar-s)\Phi(Z)]$. Consequently,
\[
\|\mu_s\|\le\sqrt\kappa\,\E|\mstar-s|
\le\sqrt\kappa\,\|\mstar-s\|_2.
\]
Square to obtain the first bound. For the second, use the reverse
triangle inequality followed by
$\|\mu_s-\mu_t\|=\|\E[(t-s)\Phi(Z)]\|
\le\sqrt\kappa\|s-t\|_2$.
\end{proof}

\begin{proposition}[Kernel control of comparison bias]
\label{prop:kernel-comparison}
Under \ref{ass:information} and the bounded-RKHS conditions in
\ref{ass:kernel}, fix $s_a,s_b,g$ under one population law. If
$w=s_a-s_b\in\cH_k$, then
\begin{equation}
|D_g(a,b)-\Delta_{ab}|
\le2\|w\|_{\cH_k}\sqrt{\SKCE_k(g)}.
\label{eq:kernel-comparison}
\end{equation}
For any $v\in\cH_k$ approximating $w$ in $L^2$,
\[
|D_g(a,b)-\Delta_{ab}|
\le2\|v\|_{\cH_k}\sqrt{\SKCE_k(g)}
       +2\|w-v\|_2\|g-\mstar\|_2.
\]
\end{proposition}
\begin{proof}
For $\mu_g=\E[(Y-g)\Phi(Z)]$, the paired identity gives
$D_g-\Delta_{ab}=2\E[(Y-g)w(Z)]=2\langle\mu_g,w\rangle$.
Apply RKHS Cauchy--Schwarz and Theorem~\ref{thm:skce-scf}.
For the second statement, split $w=v+(w-v)$ and condition the remainder
on $Z$, then use $L^2$ Cauchy--Schwarz. This is a decision-specific
specialization of residual/kernel loss control
\citep{marx2023distribution}; it audits the \emph{probe}, not the chosen
candidate's residuals.
\end{proof}

A finite Nystr\"om space need not contain nonlinear or isotonic candidate
contrasts. One alternative, fixed after candidate fitting but before
comparison, is a finite kernel spanning those contrasts:
$k_{\mathcal S}(z,z')=H^{-1}\sum_{a<b}w_{ab}(z)w_{ab}(z')$.
It has diagonal at most one and RKHS contrast norms at most $\sqrt H$.
Its SKCE tests precisely the selected residual moments, not every
feature-conditioned error. A quantitative use of
\eqref{eq:kernel-comparison} needs an upper bound on SKCE, not merely
failure to reject its zero null.

\begin{theorem}[Limiting SKCE after PAIC]
\label{thm:paic-skce}
Under Theorem~\ref{thm:paic} and \ref{ass:kernel},
\[
\SKCE_k(\tauhat\circ s_0)
\longrightarrow\SKCE_k(\tau_{\mathrm{iso}}\circ s_0)
\le\kappa(I+G)
\]
in probability and in mean. This is a limiting upper bound, not a
guarantee that recalibration decreases SKCE.
\end{theorem}
\begin{proof}
Apply \eqref{eq:skce-continuity} to the two composed scores. Each square
root is at most $\sqrt\kappa$, so the difference of the SKCE values is
at most $2\kappa\|\tauhat-\tau_{\mathrm{iso}}\|_2$.
Theorem~\ref{thm:paic} gives the asserted convergence. The limiting bound
uses Theorem~\ref{thm:skce-prce} and
$\PRCE(\tau_{\mathrm{iso}}\circ s_0;\mstar)=I+G$.
\end{proof}

\begin{remark}[Kernel dependence, approximation, and tolerance claims]
\label{rem:scope}
Exact zero identification does not make nonzero tolerances kernel-free.
For equiprobable $Z\in\{0,1\}$, let $Y=Z$ and $s=1/2$. The normalized
kernel matrix
\[
K_\rho=\begin{pmatrix}1&\rho\\\rho&1\end{pmatrix},\qquad 0\le\rho<1,
\]
is universal on this finite space, yet
$\SKCE_{K_\rho}(s)=(1-\rho)/8$ while
$\PRCE(s;\mstar)=1/4$. The kernel can make the former arbitrarily
small without changing the latter. Thus neither normalization nor
universality supplies a kernel-uniform inverse bound.
For a fixed finite feature map, including a Nystr\"om approximation,
the norm identity continues to hold but zero tests only those feature
moments. Full identification does not follow automatically.

PAIC need not minimize SKCE either. Let $Y=Z\in\{0,1\}$ with state
probabilities $(1/4,3/4)$, $s_0=9/10$, and
$k(z,z')=\mathbf1\{z=z'\}$. With oracle anchor $g=\mstar$, population
PAIC gives $3/4$, lowering
risk from $21/100$ to $3/16$, while increasing SKCE from $9/160$ to
$9/128$. The residual IPM and its square have the same minimizers.
Squaring provides the quadratic form used by the unbiased pair estimator;
the norm identity computes the witness supremum directly.
\end{remark}

\subsection{Inference with independent prompt clusters}
\label{app:cluster-inference}

Let $\zeta_j=K_j^{-1}\sum_i(Y_{ji}-s(Z_{ji}))\Phi(Z_{ji})$,
$\mu=\E\zeta_j$, and
$C=\E[(\zeta_j-\mu)\otimes(\zeta_j-\mu)]$.
Here $u\otimes v$ is the operator $w\mapsto\langle v,w\rangle u$.
Because $\|\zeta_j\|\le\sqrt\kappa$, $C$ is trace-class. In this
subsection $\theta=\|\mu\|^2$ is SKCE under the law in
\eqref{eq:prompt-law}.

\begin{theorem}[Prompt-cluster U-statistic]
\label{thm:ustat}
Under \ref{ass:kernel} and \ref{ass:audit},
\[
U_J=\frac1{J(J-1)}\sum_{j\ne\ell}\langle\zeta_j,\zeta_\ell\rangle
\]
is unbiased for $\theta$. If
$\sigma^2=4\Var\langle\zeta_1,\mu\rangle>0$, then
$\sqrt J(U_J-\theta)\Rightarrow\mathcal N(0,\sigma^2)$.
Under $\mu=0$, with $(\lambda_b)$ the eigenvalues of $C$,
\begin{equation}
JU_J\ \Rightarrow\ W_C:=\sum_{b\ge1}\lambda_b(G_b^2-1),
\qquad G_b\stackrel{\mathrm{iid}}\sim\mathcal N(0,1).
\label{eq:cluster-null}
\end{equation}
The eigenvalues are those of the residual-weighted \emph{cluster}
covariance, not of the bare semantic kernel.
\end{theorem}
\begin{proof}
Different prompt clusters are independent, hence every off-diagonal
summand has expectation $\langle\mu,\mu\rangle$. Write
$v_j=\zeta_j-\mu$. The Hoeffding decomposition
\citep{hoeffding1948} is
\[
U_J-\theta=\frac2J\sum_j\langle v_j,\mu\rangle
       +\frac1{J(J-1)}\sum_{j\ne\ell}\langle v_j,v_\ell\rangle.
\]
The second term has variance
$2\E\langle v_1,v_2\rangle^2/[J(J-1)]$; multiplied by $\sqrt J$,
it vanishes in mean square. The ordinary scalar CLT applies to the
first term when its variance is positive.

Under the null, $\mu=0$ and the exact identity is
\[
JU_J=\frac J{J-1}\left[
 \left\|J^{-1/2}\sum_j\zeta_j\right\|^2
 -\frac1J\sum_j\|\zeta_j\|^2\right].
\]
Finite-dimensional projection CLTs and the tail bound
$\E\|J^{-1/2}\sum_j(I-\Pi_d)\zeta_j\|^2
=\operatorname{tr}((I-\Pi_d)C)\to0$
give the Hilbert-space Gaussian limit with covariance $C$.
The strong law gives convergence of the subtracted average to
$\operatorname{tr}C$. Diagonalizing $C$ yields
\eqref{eq:cluster-null}. The series converges since
$\sum_b\lambda_b=\operatorname{tr}C<\infty$; in particular the
centered series converges in mean square.
\end{proof}

\begin{remark}[Degeneracy, diagonals, and sibling answers]
\label{rem:degeneracy}
The U-statistic can be negative although $\theta\ge0$. The empirical
squared mean $V_J=\|J^{-1}\sum_j\zeta_j\|^2$ instead satisfies
\[
U_J=\frac J{J-1}V_J-
             \frac{\sum_j\|\zeta_j\|^2}{J(J-1)}.
\]
Under the null, $JV_J$ has an uncentered weighted chi-squared limit;
$JU_J$ has the centered limit in \eqref{eq:cluster-null}. Agreement of
two computations of $V_J$ does not verify this diagonal correction.
Also, $\theta>0$ alone does not ensure $\sigma^2>0$: for equiprobable
binary $Z$, $Y=1$, $s=0$, and $k(z,z')=\mathbf1\{z=z'\}$,
$\theta=1/2$ but $\langle\zeta,\mu\rangle=1/2$ is constant when
each cluster contains one answer.

If there are $K\ge2$ answers per prompt, an all-row-pairs statistic also
includes sibling pairs. Let $\gamma$ be the average expected summand
over two distinct answers within one prompt. Counting ordered pairs gives
\[
\E[U_{\mathrm{rows}}]
=\theta+\frac{K-1}{JK-1}(\gamma-\theta).
\]
This bias is generally nonzero and is of the same order as the null
statistic. Excluding within-prompt pairs matches the estimator to the
independent sampling unit; grouping cross-validation folds alone does not
change which pairs enter the statistic.
\end{remark}

The construction below uses one multiplier per independent prompt, in
the spirit of wild-bootstrap tests for kernel statistics
\citep{chwialkowski2014wild}. Its validity is proved for the present
centered cluster statistic rather than inferred from a row-level test.

\begin{theorem}[A centered Gaussian multiplier test]
\label{thm:cluster-bootstrap}
Under the null and conditions of Theorem~\ref{thm:ustat}, let
$\bar\zeta=J^{-1}\sum_j\zeta_j$, $x_j=\zeta_j-\bar\zeta$, and
draw one independent $e_j\sim\mathcal N(0,1)$ per prompt. Set
\[
T_J^*=\frac1{J-1}\sum_{j\ne\ell}e_je_\ell\langle x_j,x_\ell\rangle.
\]
Its conditional law converges weakly almost surely to the law of $W_C$.
If $C\ne0$,
rejecting when $JU_J$ exceeds its conditional $(1-\alpha)$ quantile
has pointwise asymptotic size $\alpha$, for $0<\alpha<1$.
For Monte Carlo quantiles, the number of multiplier draws must also grow
to obtain this exact asymptotic statement.
\end{theorem}
\begin{proof}
Put $\widehat C_J=J^{-1}\sum_j x_j\otimes x_j$ and
$B_J^*=J^{-1/2}\sum_j e_jx_j$. Conditional on the observations,
$B_J^*$ is Gaussian with covariance $\widehat C_J$, and
\[
T_J^*=\frac J{J-1}\left[
 \|B_J^*\|^2-\operatorname{tr}\widehat C_J
 -\frac1J\sum_j(e_j^2-1)\|x_j\|^2\right].
\]
Empirical covariance operators converge to $C$ in trace norm almost
surely. To see this here, apply the strong law to finite-rank matrix
entries and to squared norms. The trace in the complementary projection
can be made arbitrarily small; off-diagonal blocks are controlled by
Cauchy--Schwarz. Subtracting the empirical mean contributes a rank-one
term converging to the corresponding population term.
Thus the conditional law of the first two terms converges weakly almost
surely to
$\|B\|^2-\operatorname{tr}C=W_C$ for $B\sim\mathcal N(0,C)$,
by the same finite-projection and trace-tail argument.
The last term has conditional variance
\[
\frac2{J^2}\sum_j\|x_j\|^4\le\frac{32\kappa^2}{J},
\]
so it vanishes. If $C\ne0$, at least one nonzero
$\lambda_bG_b^2$ has a density positive on $(0,\infty)$. The sum of
the remaining nonnegative terms has positive probability arbitrarily
near zero: make finitely many Gaussian coordinates small, then use
summability of the eigenvalues and Markov's inequality for the
independent tail. Convolution therefore gives a continuous CDF strictly
increasing on $(-\operatorname{tr}C,\infty)$, and hence unique interior
quantiles. Quantile convergence and Theorem~\ref{thm:ustat} give
asymptotic size.
If $C=0$ under the null, $\zeta_j=0$ almost surely; use strict rejection
and never reject the identically zero statistic.
\end{proof}

\begin{proposition}[A finite-sample upper bound for a fixed kernel]
\label{prop:kernel-ucb}
Under \ref{ass:kernel} and \ref{ass:audit}, let
$V_J=\|J^{-1}\sum_j\zeta_j\|^2$ and
$\theta=\|\E\zeta_j\|^2$. For $0<\alpha<1$, with probability at
least $1-\alpha$,
\begin{equation}
\theta\le\left[
 \sqrt{V_J}+\sqrt{\frac{\kappa}{J}}
        \left(1+\sqrt{2\log(1/\alpha)}\right)\right]^2.
\label{eq:kernel-ucb}
\end{equation}
This bound concerns the declared kernel only and is not an inverse bound
on oracle PRCE.
\end{proposition}
\begin{proof}
Write $F=\|J^{-1}\sum_j\zeta_j-\mu\|$. Independence and centering
give $\E F\le\sqrt{\E F^2}\le\sqrt{\kappa/J}$.
Replacing one cluster changes $F$ by at most $2\sqrt\kappa/J$.
The bounded-differences inequality \citep{mcdiarmid1998concentration}
therefore gives $F\le\sqrt{\kappa/J}+
\sqrt{2\kappa\log(1/\alpha)/J}$ with probability at least
$1-\alpha$. Since $\|\mu\|\le\sqrt{V_J}+F$, squaring proves the claim.
\end{proof}

\paragraph{Interpreting an audit.}
The point-null test asks whether the selected feature moments reveal
residual error. Establishing a tolerance $\theta<\varepsilon$ instead
requires a one-sided procedure for $H_0:\theta\ge\varepsilon$, with a
prespecified kernel, threshold, and valid uncertainty control. Non-rejection
of the point null does not establish that tolerance or small oracle PRCE.
Proposition~\ref{prop:kernel-ucb} supplies one conservative alternative:
declare $\theta<\varepsilon$ only when its upper bound is below the
prespecified $\varepsilon$. Combine its error budget with that of any
simultaneous candidate comparisons when using
Proposition~\ref{prop:kernel-comparison}.
Repeated testing over time needs a multiple-testing or sequential policy;
the theorem conditions on a fixed score, kernel, and population law.

\subsection{Repair classes and relation to prior work}
\label{app:repair-context}

\paragraph{Inputs define what a repair can change.}
Suppose $Z$ contains every input used by the compared predictors and all
risks use the same outcome and population law. A global scalar map
$\tau(P)$ assigns the same output whenever $P$ agrees. A group-dependent
map $\tau_{\gamma(Z)}(P)$ additionally uses a group assignment $\gamma(Z)$.
When the group is not determined by $P$, it can distinguish answers that
no function of $P$ alone can distinguish. Each group map may be
nondecreasing without preserving order across groups. Both outputs are
scalars: the distinction is their input restriction, not output dimension.
If a candidate uses features outside the declared $Z$, augment the
information set before applying the common-target identities.

\citet[Section~4]{chen2024reconfidencing} learn a supervised tree partition
from LLM top-layer query embeddings, using a squared-loss criterion on
correctness labels, and then fit a separate label-based isotonic regressor
within each group. Their groups are learned, not restricted to supplied
demographic categories. The deployed correction needs the group assignment.
PAIC instead fits one global map of $P$ to predictions from an independently
trained correctness probe. Table~\ref{tab:repair-inputs} separates these
choices from direct feature-based prediction.

\citet{melo2026epistemic} develop partition-based estimators of grouping
loss and excess decision risk, with local confidence audits and LLM
cascades. Their work already connects errors hidden by average calibration
to downstream decisions. Our guarantees address a different set of objects:
the conditional probe bias affecting a global monotone map, paired risk
contrasts between frozen candidates, and residual moments with dependent
answers grouped by prompt.

All rows require accounting for fitting and tuning labels. PAIC's map does
not see its fitting prompts' labels, but its probe is label-trained.
The table describes the correction's inputs, not the cost of obtaining $P$,
which may already require learned features. A score-only constraint must be
declared before comparisons. Enlarging a population predictor class cannot worsen its optimal
risk, but finite fitted methods can differ in approximation and estimation
error. A richer fitted predictor beating a scalar fit does not identify $I$
(Appendix~\ref{app:learnability}).

\paragraph{Grouping loss and monotone pooling.}
For any measurable repair output $T_\tau=\tau(P)$, conditional variance
decomposition gives
\[
\E[\Var(\mstar\mid T_\tau)]
=I+\E[\Var(q(P)\mid T_\tau)].
\]
This is the squared-loss form of the recalibration identity in
\citet[Appendix~C.10]{perezlebel2023grouping}. A noninjective map can add
grouping loss; an invertible map preserves it. For the ideal monotone output
$T=\tau_{\rm iso}(P)$, Theorem~\ref{thm:monotone} establishes
$\E[q(P)\mid T]=T$, so the additional term is exactly $G$.
The result connects monotone pooling to this established information-loss
identity.

\paragraph{Probe targets and the role of the guarantee.}
Label-isotonic and PAIC optimize over the same monotone class but fit
different targets. Conditional on probe training,
$q_g(P)=\E[g(Z)\mid P]$ determines the population probe-loss minimizer.
Theorem~\ref{thm:paic} bounds the label-risk excess over
$N+I+G$ by $(\|b_n\|_2+\sqrt{u_n})^2$ and controls the fitting term
under independent calibration sampling. Vanishing score-conditional bias
is sufficient for convergence to the monotone benchmark; it neither removes
$I+G$ nor guarantees finite-sample superiority over label-isotonic or
group-dependent calibration. In particular, greater agreement with the
training probe is not evidence of such superiority.

\paragraph{Comparison and audit are separate from candidate construction.}
The labeled residual correction applies prediction-powered inference
\citep{angelopoulos2023ppi,angelopoulos2024ppiplusplus} to paired squared
losses; Theorem~\ref{thm:corrected-comparison} instantiates empirical
Bernstein bounds~\citep{maurer2009empirical} for frozen candidates and
independent prompt pools. It can compare scalar and feature-dependent
candidates without requiring any one of them to be correct. Its coverage
does not promise narrow intervals or lower label cost.
The residual embedding is an established kernel construction
\citep{widmann2019calibration,marx2023distribution}. Our
Theorems~\ref{thm:ustat} and~\ref{thm:cluster-bootstrap} specify
cross-prompt products and centered multipliers for dependent sibling
answers. Groupwise fitting uses labels to change predictions; an independent
audit uses labels to test mean-error structure in the selected predictor.
These different roles require separate evidence, not a common proxy score
treated as ground truth.


\clearpage
\section{Experimental details and additional results}
\label{app:followup}

This appendix gives the design, results and uncertainty conventions for the
experiments of Section~\ref{sec:emp}, in the order of the main text.
Appendix~\ref{app:benchmark-breadth} reports the fixed-choice study of
Section~\ref{sec:emp:breadth}.

\subsection{Data, scores and uncertainty}
\label{app:followup-design}

\paragraph{Cached answers and loss.}
We evaluate confidence scores for eight sampled answers per prompt from
Llama-3.2-3B-Instruct and Qwen2.5-3B-Instruct on GSM8K ($7{,}353$ prompts per
model), MATH ($5{,}729$) and a mixed-domain collection of $10{,}000$ questions.
GSM8K and MATH labels are binary final-answer matches from a fixed parser; 10K
labels are automated grades in $[0,1]$. Risk averages squared loss within each
prompt and then equally across prompts. Each labeled prompt therefore supplies
eight answer labels and forms one sampling unit. The loss is Brier loss for
binary outcomes and mean squared error for graded outcomes.

\paragraph{Three score regimes.}
The \emph{raw likelihood score} is $P=\exp(-\text{mean answer-token NLL})$,
where NLL denotes negative log-likelihood. Its probes use a fixed,
label-independent 64-coordinate random-sign projection of the final answer
state plus eight likelihood summaries, standardized on source fitting data.
A \emph{learned GSM8K-trained scorer}, with 32 learned coordinates, is
transferred to MATH in Appendix~\ref{app:deployment-transfer}.
The \emph{archived learned score} is a previously trained and
temperature-recalibrated supervised score on GSM8K and 10K. It tests the
additional benefit of scalar repair after supervised calibration. These regimes
use different fitted functions and should be interpreted separately.

\paragraph{Data separation.}
Fitting, unlabeled adaptation, labeled comparison, discovery, audit and test
roles are prompt-disjoint; all answers from a prompt stay together. Fitted
candidates, maps and thresholds were saved before evaluation outcomes were
joined. The cached answer pools had already been used in earlier analyses,
so these evaluations are retrospective.

\paragraph{Uncertainty.}
The three refits, repeated allocations and seed/fold blocks overlap; their
means and ranges are descriptive. Unless stated otherwise, paired
prompt-bootstrap intervals are pointwise and conditional on the fitted
functions. The archived-score analyses cover four dataset/model cells, three
seeds, five folds and three nested fitting budgets, giving 180 settings.
Each block contains 147--200 comparison prompts, 73--100 discovery prompts
and 75--100 audit prompts.

\subsection{Controlled PAIC fitting with equally accurate probes}
\label{app:controlled-mechanism}

We test why probes with the same mean squared error can produce different
scalar repairs. The key distinction is whether their errors cancel within
each score level or leave a persistent bias. Using three constructed probes
with known population risks, we separate score-conditional bias from
finite-sample fitting error in Theorem~\ref{thm:paic-main}. This controlled
study uses no learned probes or LLM predictions.

\paragraph{Population and targets.}
The four states $(S,H)\in\{0.35,0.65\}\times\{-1,1\}$ are equally likely.
The variable $S$ determines the score level, while $H$ distinguishes two groups
at that level. Their correctness probability is $m=S+0.10H$, and
$Y\mid S,H\sim\operatorname{Bernoulli}(m)$. The reported score
$P=S+0.02$ assigns the same confidence to both groups. We evaluate outcome
risks analytically, without sampling outcomes. Here $N=0.2175$, $I=0.01$,
and $G=0$. The optimal scalar and monotone map is $P\mapsto S$, with
risk $N+I=0.2275$; the unchanged score has risk $0.2279$.
The three probes are
\[
g_-=m-0.08H,\qquad g_{\rm offset}=m+0.08,\qquad g_+=m+0.08H.
\]
All lie in $[0,1]$ and have the same excess outcome MSE
$\E[(g-m)^2]=0.0064$, hence the same outcome MSE $0.2239$.
Their conditional targets $q_g(P)=\E[g\mid P]$ are respectively
$S$, $S+0.08$, and $S$. These targets are already nondecreasing.
Thus the population PAIC excess risks are $0$, $0.0064$, and $0$:
the two hidden-state errors average away at each score, whereas the
constant offset remains.

\paragraph{Frozen finite-sample procedure.}
For each calibration count $C\in\{20,50,100,500,2000\}$, we generate
2,000 independent samples. Each sample is a multinomial count vector
over the four equiprobable states, shared by all three probes.
Different values of $C$ use independent samples. A separate 50-repetition
pilot uses distinct random streams and does not tune this design.
We fit weighted isotonic regression at $P\in\{0.37,0.67\}$.
If both scores occur, their mean probe targets are retained unless they
cross, in which case they are pooled using the observed counts.
If one score is absent, both fitted values equal the observed group's
mean target. This extension rule is fixed before sampling. Integer
cross-products of the targets, expressed in hundredths, distinguish
strict crossings from exact ties.

For a fitted map $f$, all expectations below are over the known population,
conditional on its calibration sample. With $b\in\{0,0.08,0\}$ for the
three probes, we calculate
\[
\begin{gathered}
L(f)-(N+I)=\E[(f(P)-S)^2],\qquad
u=\E[(f(P)-(S+b))^2],\\
L(f)-(N+I)\le (|b|+\sqrt{u})^2.
\end{gathered}
\]
Here $u$ is the population probe-loss gap above its monotone minimum,
not empirical fitting loss. Exact evaluation removes outcome-sampling
noise; the reported Monte Carlo uncertainty comes only from calibration
samples.

\begin{table}[t]
\centering
\small
\setlength{\tabcolsep}{5pt}
\caption{PAIC excess risk above the scalar optimum $0.2275$.
Entries are mean (Monte Carlo standard error), \emph{both multiplied by
$10^3$}, over 2,000 repetitions. Standard errors are the empirical
standard deviation divided by $\sqrt{2000}$ and are pointwise, not
simultaneous intervals. The unchanged score's excess is $0.4$ in these units.}
\label{tab:controlled-paic-mechanism}
\begin{tabular}{rccc}
\toprule
$C$ & $g_-$ & $g_{\rm offset}$ & $g_+$ \\
\midrule
20   & $0.04347\;(0.00095)$   & $7.6023\;(0.0874)$ & $3.5208\;(0.0767)$ \\
50   & $0.01649\;(0.00036)$   & $6.8077\;(0.0535)$ & $1.3354\;(0.0294)$ \\
100  & $0.00791\;(0.00017)$   & $6.6049\;(0.0366)$ & $0.64096\;(0.01358)$ \\
500  & $0.001527\;(0.000035)$ & $6.4237\;(0.0157)$ & $0.12366\;(0.00287)$ \\
2000 & $0.000390\;(0.000009)$ & $6.4011\;(0.0079)$ & $0.03156\;(0.00073)$ \\
\midrule
Population limit & $0$ & $6.4$ & $0$ \\
\bottomrule
\end{tabular}
\end{table}

\paragraph{Results and scope.}
Table~\ref{tab:controlled-paic-mechanism} separates fitting variation from
the persistent offset bias. At $C=2000$, mean excess risks are approximately
$3.90\times10^{-7}$, $0.0064011$, and $3.16\times10^{-5}$.
Equal probe MSE does not imply equal finite-sample variation: the
within-score target amplitudes are $0.02$, $0.10$, and $0.18$.
More calibration samples bring the offset probe's repair toward its
biased population target, not toward the scalar optimum.

All 30,000 fitted maps satisfy the bound to numerical tolerance $10^{-12}$. For the two zero-offset probes, the bound is an algebraic equality. Appendix B.5 examines the same mechanism in learned probes.


\subsection{Raw-score repair under distribution shift}
\label{app:followup-shifts}
\label{app:followup-shift}

The shift study crosses two models, three shifts, three refits, source fitting
budgets $F\in\{100,500,1000\}$ and unlabeled adaptation budgets
$U\in\{100,500,1000\}$. Difficulty transfer uses MATH levels 1--2 as source and
level 3 as target; task transfer uses GSM8K as source and MATH as target; subject
transfer uses 10K general knowledge and history as source and science as target,
excluding mathematics. Test sets have 598, 3,085 and 1,300 prompts per refit.
The primary comparison fixes $F=U=500$.

Platt and beta minimize outcome-targeted cross-entropy, with soft targets for
grades; beta uses $\operatorname{sigmoid}(a\log P-b\log(1-P)+c)$ with
nonnegative bounded $a,b$ \citep{kull2017beta}. The flexible scalar model uses
60 histogram-boosting iterations, eight leaves, minimum leaf size 32 and unit
$L_2$ regularization. Ridge uses penalty 10; the tree probe, prespecified as
primary, uses 48 ExtraTrees with maximum depth 10, minimum leaf size 32 and
feature fraction .7. The group method fits an at-most-eight-leaf feature tree on
one source half and within-leaf isotonic maps on the other, following
\citet{chen2024reconfidencing} with projected answer states instead of their
query embeddings. Scalar methods and PAIC serve only $P$; direct probes need the
72 probe inputs. The target-isotonic diagnostic uses labeled target prompts and
no source labels, so it is a yardstick rather than an equally supervised
competitor.

\begin{table}[t]
\centering\footnotesize\setlength{\tabcolsep}{1.7pt}
\caption{Primary shift risks at $F=U=500$, averaged over three overlapping
refits (lower is better). ``Group'' is the adapted group calibrator of
\citet{chen2024reconfidencing}, which serves $P$ and group features; direct
trees serve the probe's features. ``Target'' is the isotonic diagnostic with
250 target labels and no source labels; every other fitted method uses 500
source labels, and PAIC also uses 500 unlabeled target prompts.}
\label{tab:followup-shifts}
\begin{tabular}{llrrrrrrrrrr}
\toprule
Shift & Model & Original & \shortstack{Source\\iso.} & Platt & Beta & \shortstack{Non-\\mono.} & \shortstack{PAIC\\ridge} & \shortstack{PAIC\\trees} & Group & \shortstack{Direct\\trees} & Target\\
\midrule
Difficulty & Llama & .296066 & .251714 & .251608 & .251833 & .253394 & .237544 & .238951 & .249526 & .193440 & .236763\\
 & Qwen & .229475 & .210177 & .209458 & .209722 & .211439 & .194298 & .196389 & .206683 & .138683 & .194654\\
Task & Llama & .247051 & .238183 & .237836 & .236746 & .238634 & .234934 & .218722 & .246771 & .199584 & .219169\\
 & Qwen & .175655 & .185457 & .185327 & .182487 & .182422 & .166735 & .165191 & .184328 & .148355 & .159922\\
Subject & Llama & .325099 & .099458 & .103151 & .102349 & .099871 & .083310 & .083183 & .099059 & .054461 & .082144\\
 & Qwen & .253062 & .122240 & .129469 & .128146 & .120107 & .117090 & .116237 & .125438 & .074457 & .112542\\
\bottomrule
\end{tabular}
\end{table}

With either probe, PAIC has lower mean risk than every source-fitted scalar map
and the unchanged score in all six cells (Table~\ref{tab:followup-shifts}). Tree PAIC improves on source-fitted beta in
all 18 refits, with all 18 pointwise paired intervals below zero; ridge PAIC
does so in 17. Against the flexible scalar map, Qwen subject transfer is less
decisive: one of three ridge intervals and two of three tree intervals exclude
zero. Relative to source label-isotonic, tree and ridge PAIC lower mean risk by
$4.9$--$16.4\%$ and $1.4$--$16.2\%$ and close $62$--$102\%$ and $17$--$102\%$ of
the gap to the target-isotonic yardstick (median $87\%$ over both probes and six
cells). Direct probes have lower mean risk in all six cells. PAIC improves the deployed map of P; retaining the probe inputs enables further risk reduction. The group
calibrator, the closest published comparator, improves on source
label-isotonic in four of six cells but is above PAIC in all six and above the
unchanged score under Qwen task transfer.


\paragraph{Learning curves and boundaries.}
With $U=500$, increasing $F$ from 100 to 1,000 lowers difficulty-shift
ridge-PAIC risk from .241287 to .237034 for Llama and from .194314 to .193718
for Qwen. Under Llama task transfer, ridge PAIC worsens from .234934 at $F=500$
to .249177 at $F=1000$, above the unchanged score (.247051), while tree PAIC
stays stable (.218722 to .218441); probe and variant choice is therefore left to
the labeled adoption test of Appendix~\ref{app:followup-selection}. More
unlabeled target data changes risk less: Qwen task-transfer ridge risks are
.165542/.166735/.166094 at $U=100/500/1000$. Within a single domain, at
$F=U=500$, ridge PAIC and label-isotonic differ by less than .0004 in every
GSM8K, MATH and 10K cell mean, as expected when source and target share one
score--correctness curve. On the archived learned score, every scalar map
changes risk by less than $10^{-3}$ (Table~\ref{tab:followup-original}); PAIC
improves on label-isotonic by .000108--.000401 in the four full-budget means.

\begin{table}[t]
\centering\small\setlength{\tabcolsep}{4pt}
\caption{Archived learned score at the full downstream fitting budget.
Entries average 15 overlapping seed/fold blocks per cell.}
\label{tab:followup-original}
\begin{tabular}{llrrrrrr}
\toprule
Dataset & Model & Original & \shortstack{Label\\iso.} & \shortstack{Flexible\\scalar} & Platt & PAIC & Direct $g$\\
\midrule
GSM8K & Llama & .104379 & .105213 & .104840 & .104688 & .104812 & .106254\\
 & Qwen & .085551 & .085418 & .085300 & .085266 & .085281 & .086493\\
10K & Llama & .049106 & .048715 & .048548 & .048779 & .048607 & .048427\\
 & Qwen & .054499 & .054428 & .054160 & .054534 & .054235 & .054067\\
\bottomrule
\end{tabular}
\end{table}

\subsection{Equally informed adaptation}
\label{app:equal-info}

\paragraph{Design.}
The comparison reuses the six model--shift settings and three overlapping refits
above. Every adaptation method is allowed 500 labeled source prompts, 500
unlabeled target prompts and an independent pool of 500 source prompts whose
labels are masked; each prompt keeps its eight siblings with equal weight. The
tree and ridge teachers are reused; trees are primary and ridge is a
prespecified sensitivity. Target-probe Platt and beta fit exactly the target
prompts and teacher predictions used by PAIC, minimizing soft cross-entropy.
Source-distilled isotonic fits the same teacher's predictions on the masked
source pool, isolating the value of target adaptation. Importance-weighted (IW)
isotonic fits the 500 source outcomes with prompt-level density ratios estimated
by five-fold domain classification \citep{park2020calibrated}; the classifier
receives the prompt means of the 72 probe inputs and the mean original
confidence, but no outcomes or task labels. Logistic regularization is fixed at
$C=0.1$, odds are clipped to $[0.1,10]$, and one weight applies to all siblings
of a prompt. The group calibrator uses no unlabeled target prompts and serves
group features, so its risks appear in Table~\ref{tab:followup-shifts}.

\begin{table}[t]
\centering\footnotesize\setlength{\tabcolsep}{3.4pt}
\caption{Equally informed adaptation: target risks averaged over three
overlapping refits (lower is better). CE and SQ denote cross-entropy and
squared-loss fitting of the target-probe maps; the SQ maps are a posthoc
sensitivity with its own contrast family.}
\label{tab:equal-info-risks}
\begin{tabular}{llrrrrrrr}
\toprule
 & & & \multicolumn{2}{c}{Target-probe Platt} & \multicolumn{2}{c}{Target-probe beta} & \shortstack{Source-} & \\
\cmidrule(lr){4-5}\cmidrule(lr){6-7}
Shift & Model & PAIC & CE & SQ & CE & SQ & \shortstack{distilled} & IW iso.\\
\midrule
Difficulty & Llama & .238951 & .240000 & .240173 & .239466 & .239557 & .251711 & .237170\\
 & Qwen & .196389 & .198206 & .198034 & .196995 & .196972 & .207772 & .197005\\
Task & Llama & .218722 & .219675 & .219552 & .219042 & .218932 & .236728 & .235725\\
 & Qwen & .165191 & .167125 & .166397 & .165519 & .165316 & .180906 & .184300\\
Subject & Llama & .083183 & .084982 & .084411 & .084557 & .084132 & .099035 & .092045\\
 & Qwen & .116237 & .118592 & .118227 & .117870 & .117605 & .123411 & .124173\\
\bottomrule
\end{tabular}
\end{table}

\paragraph{Paired contrasts.}
The 72 prespecified primary contrasts compare tree-PAIC with the four
comparators in each of the 18 panels. Losses average siblings within a prompt,
and uncertainty uses 1,999 prompt bootstraps with pointwise intervals and
approximate normal-Bonferroni intervals across all 72 contrasts
(Table~\ref{tab:equal-info-contrasts}, Figure~\ref{fig:equal-info-contrasts}).
Within-panel four-contrast max-$t$ intervals favor PAIC in 53 of 72 contrasts
and importance weighting in one; no contrast favors a comparator under the
global adjustment. PAIC's estimate is higher in five contrasts: against beta by
.000046 (Qwen task) and .000316 (Llama difficulty), and against IW by .003894
and .001453 (Llama difficulty) and .000373 (Qwen difficulty), each in one refit.
Only the .003894 difference has a pointwise interval excluding zero,
$[.001080,.006801]$. Direct tree predictors remain more accurate than PAIC in all
six settings.

\begin{table}[t]
\centering\small\setlength{\tabcolsep}{5pt}
\caption{Counts of the 18 refit-level contrasts per comparator in which PAIC
has the lower estimate, and in which pointwise or globally adjusted intervals
favor PAIC. The squared-loss rows form a separate 36-contrast family.}
\label{tab:equal-info-contrasts}
\begin{tabular}{lccc}
\toprule
Comparator & \shortstack{Lower PAIC\\estimate} & \shortstack{Pointwise interval\\favors PAIC} & \shortstack{Global interval\\favors PAIC}\\
\midrule
Target-probe Platt (CE) & 18 & 15 & 10\\
Target-probe beta (CE) & 16 & 10 & 8\\
Source-distilled isotonic & 18 & 18 & 18\\
Importance-weighted isotonic & 15 & 12 & 11\\
Total & 67/72 & 55/72 & 47/72\\
\midrule
Target-probe Platt (SQ) & 18 & 13 & 9\\
Target-probe beta (SQ) & 15 & 9 & 7\\
\bottomrule
\end{tabular}
\end{table}

\begin{figure}[t]
\centering
\includegraphics[width=\linewidth]{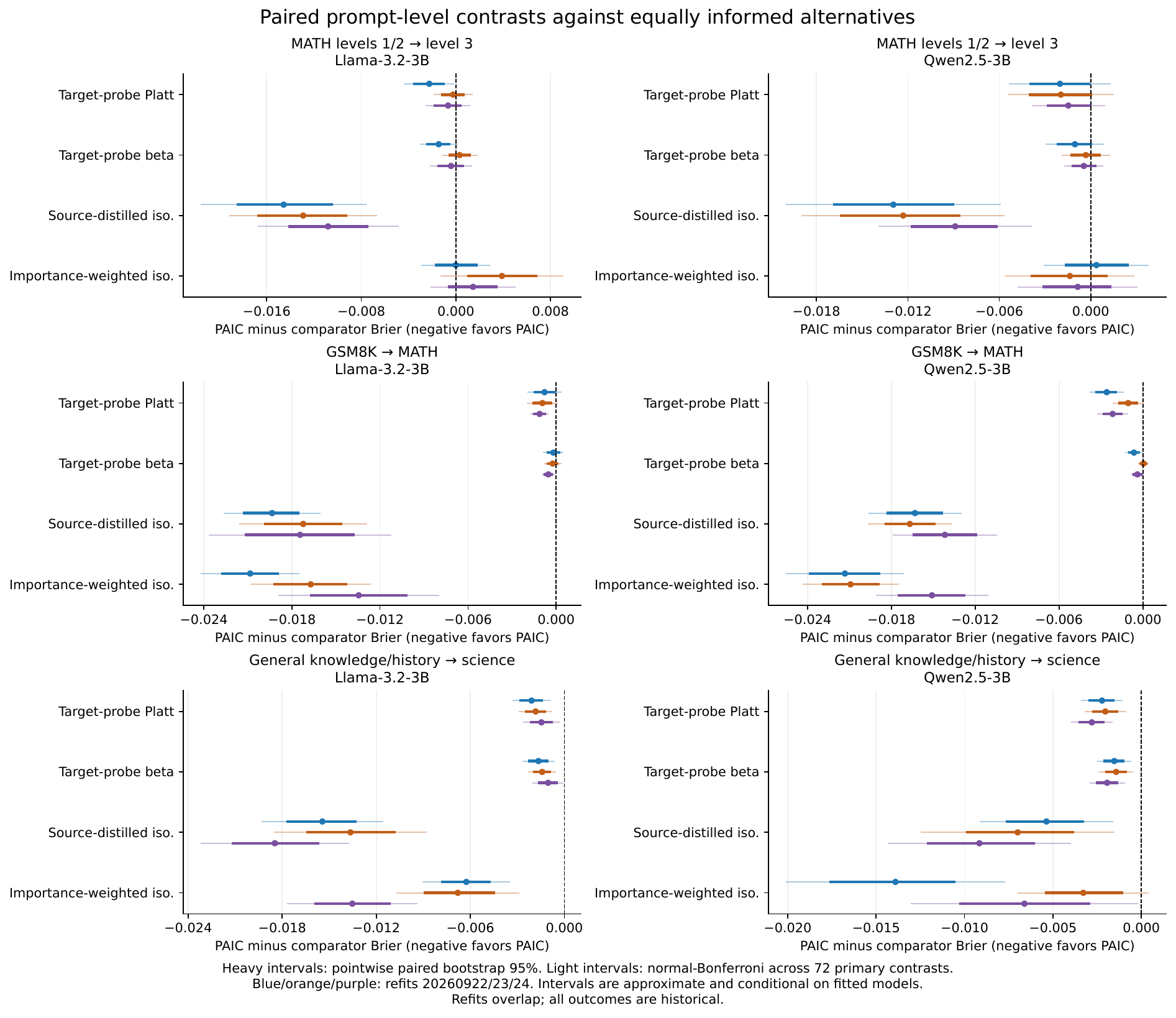}
\caption{\textbf{Paired contrasts against equally informed alternatives.}
PAIC minus comparator Brier risk for each refit (negative favors PAIC). Heavy
bars are pointwise paired-bootstrap $95\%$ intervals; light bars are
normal-Bonferroni intervals across the 72 primary contrasts. Intervals are
conditional on the fitted models.}
\label{fig:equal-info-contrasts}
\end{figure}

\paragraph{Fitting objective.}
The cross-entropy maps and isotonic PAIC minimize different objectives. A
sensitivity analysis fits 36 further Platt and beta maps that minimize squared
error to the same target teacher values, with the same $10^{-6}$ parameter
penalty and bounds and two fixed starts chosen by adaptation loss alone. Every
selected map lowers adaptation teacher error relative to its cross-entropy
version, and on test data squared-loss beta improves on cross-entropy beta in 15
of 18 refits (Platt in 14). PAIC still has the lower estimate against
squared-loss Platt in 18 of 18 refits and against squared-loss beta in 15 of 18,
with mean gains over the latter of .000125--.001368 per cell. The objective thus
explains part of the gap to parametric maps; the remainder is a modest mapping
benefit that holds in most refits.

\paragraph{Overlap and importance weighting.}
Natural-shift overlap differs sharply (Table~\ref{tab:equal-info-overlap}).
Clipping inflates the effective sample size (ESS) by equalizing tiny weights but
does not restore missing support. Task and subject shifts have nearly separable source and target inputs. We therefore also evaluate randomized GSM8K shifts with known overlap. Estimated and oracle importance weighting give similar risks in these controls, and all target-aware maps have similar mean risks (Table~\ref{tab:equal-info-overlap}). Six
randomized GSM8K controls therefore assign each prompt to source or target with
known probability $e(X)=0.2+0.6\,\operatorname{sigmoid}(z)$, where $z$
standardizes the prompt's mean original confidence; this provides overlap and
preserves the conditional outcome law by design. In these controls all
target-aware maps nearly tie, PAIC has the lowest mean risk, and estimated and
oracle ($e/(1-e)$) weighting agree.

\begin{table}[t]
\centering\small\setlength{\tabcolsep}{4pt}
\caption{Top: overlap diagnostics for importance weighting, averaged over
three refits; ESS counts source prompts out of 500. Bottom: target risks in
the randomized-overlap GSM8K controls.}
\label{tab:equal-info-overlap}
\begin{tabular}{llcccc}
\toprule
Shift & Model & Domain AUC & Raw ESS & Clipped ESS & Lower-clipped fraction\\
\midrule
Difficulty & Llama & .643 & 210.0 & 225.3 & .026\\
 & Qwen & .672 & 195.3 & 229.7 & .023\\
Task & Llama & .999 & 46.0 & 347.8 & .956\\
 & Qwen & .998 & 26.8 & 137.2 & .902\\
Subject & Llama & .994 & 24.6 & 56.3 & .808\\
 & Qwen & .993 & 7.9 & 43.3 & .798\\
Randomized GSM8K & Llama & .586 & 275.8 & 280.9 & .007\\
 & Qwen & .585 & 300.1 & 304.9 & .019\\
\bottomrule
\end{tabular}

\medskip
\begin{tabular}{lcccccc}
\toprule
Model & Source iso. & Estimated IW & Oracle IW & PAIC & Target-probe beta & Source-distilled\\
\midrule
Llama & .144594 & .144691 & .144531 & .144009 & .144090 & .144074\\
Qwen & .107361 & .107121 & .107149 & .106985 & .107043 & .106991\\
\bottomrule
\end{tabular}
\end{table}

\subsection{Conditional bias of learned probes}
\label{app:followup-bias}

The raw-score diagnostic compares ridge probes with 16 or 72 inputs and
ExtraTrees, fitted on identical source labels.
Bounded perturbations
separate score-dependent bias from cancelling error: independent-sign
perturbations raise probe MSE by .006470 on average but PAIC risk by only
.000029, whereas score-dependent perturbations raise MSE by .000459 and PAIC
risk by .001161.

On the archived learned score, four probes (ridge penalties .1 and 10; 32
shallow and 64 deeper ExtraTrees) and nested unlabeled calibration prefixes give
720 probe fits and 2,880 maps. Reversals persist: on GSM8K Qwen (seed 42, fold 0,
quarter budget), shallow trees have lower direct risk than ridge (.090788 versus
.094128) but a worse PAIC map (.089870 versus .087777; paired interval for the
gap $[.000632,.003630]$).F
Coarse bins approximate,
rather than identify, the conditional-bias norm of Theorem~\ref{thm:paic-main}:
the reversals support the mechanism, while reliable target-label-free probe
selection remains open, and the labeled adoption test is the safeguard.

\subsection{A learned score transferred to harder MATH problems}
\label{app:deployment-transfer}
\label{app:target-adaptation}

This study uses a learned scorer rather than the raw likelihood score. Each
model's final scorer, trained on GSM8K, is applied to cached MATH answers with
eight answers per problem \citep{hendrycksmath2021}; levels 1--2 supply 3,110
source problems and level 3 supplies the target. Difficulty is observational:
the levels also differ in subject mix and answer length. Labels are the
original parser matches, replayed exactly for all $91{,}664$ answers. A fixed
hash splits the source problems into probe fitting $F$ (1,555), scalar-map
calibration $C$ (777) and source evaluation (778). The probe takes the scorer's
32 learned coordinates and $P$ (a 256-component Nystr\"om ridge model).
Label-isotonic is fitted on $F$ outcomes and source PAIC on $C$ with the
probe's predictions, so both use the same 12,440 source labels; the direct probe
serves the richer inputs.

\paragraph{Transfer.}
Both scalar repairs lower squared loss relative to $P$ on source and on target,
and all eight corresponding Bonferroni-adjusted ($99.6875\%$) intervals exclude
zero. On source, the PAIC--isotonic difference is unresolved; on all 2,619
level-3 problems, label-isotonic is lower by $0.00087$ for Llama and $0.00094$
for Qwen, both intervals excluding zero. The repairs' gain over $P$ shrinks from source to
target for Llama ($0.09391$ to $0.05793$ for PAIC) and grows for Qwen
($0.00466$ to $0.01242$).

\paragraph{Adaptation with unlabeled target answers.}
Keeping the scorer and probe fixed, target PAIC refits only the monotone map,
on 777 level-3 problems whose labels are not used; the remaining 1,842 level-3
problems form a common evaluation set (Table~\ref{tab:target-adaptation-levels}).
For Llama, target PAIC lowers risk by $0.00633$ relative to source PAIC and by
$0.00552$ relative to source label-isotonic (Bonferroni $98.75\%$ intervals
$[-0.00792,-0.00479]$ and $[-0.00689,-0.00418]$). The gain is not uniform
across subjects: Prealgebra risk rises by $0.00456$ (pointwise interval
$[0.00194,0.00715]$), while Intermediate Algebra and Precalculus improve most.
For Qwen, the change relative to source PAIC is unresolved ($-0.00008$,
interval $[-0.00026,0.00010]$), and target PAIC remains $0.00083$ above source
label-isotonic ($[0.00019,0.00149]$). In the highest score bin, the fixed probe
overpredicts correctness by $0.068$ for Llama and $0.049$ for Qwen; such coarse
bins are consistent with score-conditional bias but neither estimate $b_n$ nor
isolate the cause of any risk difference.

\paragraph{Review decisions.}
With perfect review at cost $c=0.10$ (accept when the score is at least
$0.90$), both target PAIC maps review every answer and so match the
always-review loss $0.10$. For Qwen, acceptance falls from $73.00\%$ under
source PAIC to zero because the largest adapted score, $0.899516$, lies just
below the threshold, whereas source PAIC reaches $0.900015$. A small change in
squared risk near a threshold can thus change many decisions.

\begin{table}[!ht]
\centering
\caption{A learned GSM8K-trained score on MATH level 3: squared risk and
squared disagreement with the fixed probe on the common 1,842-prompt evaluation
set (eight answers per prompt), with pointwise 95\% intervals.}
\label{tab:target-adaptation-levels}
\small
\setlength{\tabcolsep}{3pt}
\begin{tabular}{llrr}
\toprule
Model & Candidate & Squared risk & Probe disagreement \\
\midrule
Llama & Unchanged score & $0.30109\;[0.29077, 0.31163]$ & $0.10114\;[0.09632, 0.10616]$ \\
Llama & Source label-isotonic & $0.23767\;[0.23206, 0.24340]$ & $0.01821\;[0.01730, 0.01913]$ \\
Llama & Source PAIC & $0.23848\;[0.23284, 0.24424]$ & $0.01858\;[0.01764, 0.01954]$ \\
Llama & Target PAIC & $0.23215\;[0.22695, 0.23740]$ & $0.01725\;[0.01636, 0.01816]$ \\
Llama & Direct probe & $0.21456\;[0.20900, 0.22012]$ & $0.00000\;[0.00000, 0.00000]$ \\
\midrule
Qwen & Unchanged score & $0.16308\;[0.15296, 0.17326]$ & $0.01505\;[0.01397, 0.01618]$ \\
Qwen & Source label-isotonic & $0.15026\;[0.14161, 0.15896]$ & $0.00249\;[0.00222, 0.00279]$ \\
Qwen & Source PAIC & $0.15117\;[0.14249, 0.15991]$ & $0.00191\;[0.00167, 0.00218]$ \\
Qwen & Target PAIC & $0.15109\;[0.14239, 0.15983]$ & $0.00191\;[0.00166, 0.00218]$ \\
Qwen & Direct probe & $0.15123\;[0.14253, 0.15993]$ & $0.00000\;[0.00000, 0.00000]$ \\
\bottomrule
\end{tabular}
\par\smallskip
\begin{minipage}{0.98\linewidth}\footnotesize The direct probe serves richer inputs. Source fits are fixed; each target PAIC map is fitted on 777 separate level-3 prompts without their labels.\end{minipage}
\end{table}

\subsection{Controlled presentation changes}
\label{app:controlled-presentation}

This study asks whether an unlabeled target projection improves a scalar
confidence when a problem's presentation changes but its answer and label do
not. Constructed problems track integer counts of red tokens through additions,
removals and transfers among crates. Eight strata cross 4/8/12/16 operations
with 4/6 crates, and each structural family has two numerical instances. Entity
renaming, fact reordering and irrelevant context supply three known
transformations. Llama-3.2-3B-Instruct and Qwen2.5-3B-Instruct each answer the
canonical problem once, and every presentation reuses that answer, so only the
confidence changes. The design tests held-out structural families under known
transformations.

\paragraph{Score, teachers and supervision.}
Here $P$ is the model's probability of the Yes verification sequence,
normalized over Yes and No for the question and fixed answer; it differs from
both the answer-token likelihood and the fixed-choice score of
Appendix~\ref{app:benchmark-breadth}. The rich teacher uses three prefix
hidden-state blocks (last, middle and quarter) plus $P$; its matched score-only
teacher uses $P$ alone. Both are Nystr\"om kernel ridge models with 256
components, ridge penalty 1 and unit kernel scale, fitted with five family
folds. Per model, the family-disjoint roles are fitting $F$ (1,200 families;
2,400 graded answers), canonical projection $C$ (600 families), target
projection $U$ (600 families, one transformed view each, balanced within
strata), and evaluation $T$ (1,800 families; 3,600 graded answers). Canonical
fitting uses the 2,400 labeled rows. Augmented fitting, denoted
$\mathrm{aug}$, uses the three transformed views of the same answers: 7,200
rows with the same 2,400 labels. Thus $\mathrm{ISO}_{\mathrm{aug}}$ fits
label-isotonic to transformed scores; target PAIC projects the canonical rich
teacher on $U$; $\mathrm{PAIC}_{\mathrm{aug}}$ projects a transformed-trained
rich teacher; and $\mathrm{PAIC}_{P,\mathrm{aug}}$ projects its score-only
counterpart. Source PAIC projects on canonical $C$.

\paragraph{Evaluation.}
Every evaluation answer has a binary solver-based grade; answers that cannot
be parsed count as incorrect (49 of 3,600 for Llama, including 37 truncated,
and 16 for Qwen, including 10 truncated). Primary risk averages the two
instances and three transformed presentations within each family and applies
fixed stratum weights: $0.05$ for each four-operation stratum and $0.15$ for
each other stratum. Families, not presentation rows, are the sampling units.
The six primary contrasts in Figure~\ref{fig:controlled-presentation} use
approximate stratified finite-population Welch-$t$ intervals with
$\alpha=0.05/6$ each, conditional on the fits and grades. All six point
estimates favor the comparator; four intervals exclude zero and the two Qwen
augmented comparisons are unresolved (Table~\ref{tab:controlled-presentation-risk}).
In point estimates, transformed-view fitting helps the direct rich teacher
most, while no projected map recovers that gain.

\begin{table}[!ht]
\centering\footnotesize
\caption{Controlled presentation changes: target Brier risk for all twelve
candidates, averaging the three transformed views and two instances within
each family with fixed stratum weights. All fitted candidates share 2,400
unique fitting labels per model; the six primary paired contrasts are in
Figure~\ref{fig:controlled-presentation}.}
\label{tab:controlled-presentation-risk}
\begin{tabular}{lrr}
\toprule
Candidate & Llama 3.2 3B & Qwen 2.5 3B \\
\midrule
Original $P$ & 0.295817 & 0.377312 \\
Canonical label-isotonic & 0.247903 & 0.241805 \\
Source PAIC & 0.246479 & 0.244422 \\
Target PAIC & 0.250821 & 0.247072 \\
Direct rich teacher & 0.243782 & 0.231066 \\
Direct score-only teacher & 0.244477 & 0.242851 \\
Target score-only PAIC & 0.245427 & 0.242856 \\
$\mathrm{ISO}_{\mathrm{aug}}$ & 0.245572 & 0.241482 \\
Direct rich teacher, transformed fitting & 0.235877 & 0.223031 \\
Direct score-only teacher, transformed fitting & 0.244191 & 0.242760 \\
$\mathrm{PAIC}_{\mathrm{aug}}$ & 0.247091 & 0.242958 \\
$\mathrm{PAIC}_{P,\mathrm{aug}}$ & 0.245234 & 0.242692 \\
\bottomrule
\end{tabular}
\end{table}

\subsection{An offline verifier as teacher}
\label{app:verifier-projection}
\label{app:teacher}

Can a score-only map keep part of the benefit of an extra verifier pass? For
Qwen2.5-3B-Instruct under the MATH difficulty shift, a fixed hash selects one
cached answer per prompt: 500 labeled level-1/2 source prompts, 500 unlabeled
level-3 adaptation prompts and 598 level-3 test prompts. A second pass of the
same checkpoint scores \texttt{True} against \texttt{False} for the question and
fixed answer, without a reference answer. A positive Platt map fitted to the 500
source grades calibrates this verifier, and the calibrated verifier is the
teacher. PAIC fits a nondecreasing map of the likelihood score to the teacher on
the adaptation prompts. Target Platt and monotone beta maps receive the same
inputs and teacher values, fitted by cross-entropy (CE) or Brier loss; source
isotonic and Platt use the source grades. Supervision is thus shared, while the
information and computation needed at serving differ.

PAIC lowers squared risk from $0.21274$ to $0.19718$ while serving only the
original score, and it improves on both target Platt variants
(Table~\ref{tab:verifier-projection-risk}). Comparisons with source isotonic,
source Platt, the raw verifier and both target beta variants are unresolved.
The calibrated verifier reaches $0.15506$: the projection keeps part of the
teacher's benefit and leaves a gap of $0.04212$ ($[0.02548,0.05914]$).

\begin{table}[!ht]
\centering
\caption{Offline verifier as teacher: squared risk and paired differences on
the 598 test prompts. $\Delta=L(\mathrm{PAIC})-L(\mathrm{candidate})$; negative
values favor PAIC. Intervals are pointwise 95\% paired prompt-bootstrap
intervals (1,999 draws), conditional on the fitted maps.}
\label{tab:verifier-projection-risk}
\footnotesize
\setlength{\tabcolsep}{4pt}
\begin{tabular}{lrrl}
\toprule
Candidate & MSE & $\Delta$ & 95\% interval for $\Delta$ \\
\midrule
Original score & 0.21274 & $-0.01556$ & $[-0.02088,\,-0.01031]$ \\
Source isotonic & 0.19748 & $-0.00030$ & $[-0.00863,\,0.00767]$ \\
Source Platt & 0.19622 & $0.00096$ & $[-0.00351,\,0.00574]$ \\
Raw verifier & 0.19672 & $0.00046$ & $[-0.02123,\,0.02271]$ \\
Calibrated verifier & 0.15506 & $0.04212$ & $[0.02548,\,0.05914]$ \\
PAIC & 0.19718 & $0.00000$ & $[0.00000,\,0.00000]$ \\
Target Platt (CE) & 0.20071 & $-0.00353$ & $[-0.00640,\,-0.00086]$ \\
Target beta (CE) & 0.19853 & $-0.00135$ & $[-0.00339,\,0.00059]$ \\
Target Platt (Brier) & 0.20020 & $-0.00302$ & $[-0.00577,\,-0.00048]$ \\
Target beta (Brier) & 0.19817 & $-0.00099$ & $[-0.00293,\,0.00086]$ \\
\bottomrule
\end{tabular}
\end{table}

\paragraph{Review decisions and cost.}
Under the perfect-review policy, PAIC accepts only 6 of 598 answers at $c=0.10$
(loss $0.10401$, just above reviewing everything), and all four target
Platt/beta variants review every answer; at $c=0.25$ and $0.50$ the calibrated
verifier has the lowest loss (Table~\ref{tab:verifier-projection-policies}).
On a separate sample of 16 newly generated answers, the extra verifier pass took
$0.0401$ seconds against $11.203$ seconds for generation and original
confidence, about $0.36\%$ of the recorded time. The separate timing experiment measures the verifier’s incremental inference cost for this fixed teacher.

\begin{table}[!ht]
\centering
\caption{Offline verifier as teacher: review policies at three review costs.
Each pair reports mean decision loss and automated acceptance ($A$, percent)
under perfect review; scoring latency is not included.}
\label{tab:verifier-projection-policies}
\footnotesize
\setlength{\tabcolsep}{3.5pt}
\begin{tabular}{lrrrrrr}
\toprule
& \multicolumn{2}{c}{$c=0.10$} & \multicolumn{2}{c}{$c=0.25$} & \multicolumn{2}{c}{$c=0.50$} \\
Candidate & Loss & $A$ & Loss & $A$ & Loss & $A$ \\
\midrule
Original score & 0.14732 & 58.03 & 0.25962 & 97.83 & 0.26756 & 100.00 \\
Source isotonic & 0.14448 & 49.16 & 0.24582 & 95.32 & 0.26505 & 97.83 \\
Source Platt & 0.12659 & 33.61 & 0.22952 & 88.46 & 0.26672 & 99.50 \\
Raw verifier & 0.18846 & 88.80 & 0.20778 & 89.80 & 0.23495 & 90.47 \\
Calibrated verifier & 0.11488 & 55.35 & 0.19064 & 84.62 & 0.23328 & 89.80 \\
PAIC & 0.10401 & 1.00 & 0.24540 & 95.48 & 0.26672 & 99.83 \\
Target Platt (CE) & 0.10000 & 0.00 & 0.26756 & 100.00 & 0.26756 & 100.00 \\
Target beta (CE) & 0.10000 & 0.00 & 0.25962 & 97.83 & 0.26756 & 100.00 \\
Target Platt (Brier) & 0.10000 & 0.00 & 0.26756 & 100.00 & 0.26756 & 100.00 \\
Target beta (Brier) & 0.10000 & 0.00 & 0.25836 & 97.66 & 0.26756 & 100.00 \\
Review all & 0.10000 & 0.00 & 0.25000 & 0.00 & 0.50000 & 0.00 \\
Accept all & 0.26756 & 100.00 & 0.26756 & 100.00 & 0.26756 & 100.00 \\
\bottomrule
\end{tabular}
\end{table}

\subsection{Controlled tests of residual-corrected comparisons}
\label{app:controlled-comparisons}

These known-population controls evaluate the comparison procedure of
Appendix~\ref{app:corrected-comparisons}; they contain no fitted repairs or LLM
predictions. Independent prompts have uniformly distributed
$(S,H)\in\{0.35,0.65\}\times\{-1,1\}$, conditional mean $m=S+dH$ and starting
score $P=S+c$, with $d=0.1$ and $c=0.02$ by default. The six candidates are $P$,
$S$, $S+0.1H$, $1/2$, $S-0.1H$ and $g$, and the probe is one of $m$, $m+0.08$,
$S$ or $1-m$. The 28 settings comprise 12 singleton binary controls (four
probes, $J_L\in\{100,500,2000\}$, $J_U=4J_L$), four balanced-pool controls
($J_L=J_U=500$), four no-useful-repair controls ($d=c=0$), six sibling controls
($K=8$ with $\rho\in\{0,0.8\}$, or $K=2$ versus $16$ by subgroup) and two graded
controls; each has 2,000 repetitions. We compare paired labels, the uncorrected
proxy comparison and correction with $\lambda=1$, using the global
empirical-Bernstein ranges over the 15 pairs at $\alpha=0.05$.

\paragraph{Correction removes proxy bias.}
Outside the no-useful-repair controls, the true contrast is $L(P)-L(S)=0.0004$,
favoring $S$. For the offset binary control with $(J_L,J_U)=(500,2000)$, mean
estimates over 2,000 repetitions are $0.000412584$ for paired labels,
$-0.002805102$ for the proxy and $0.000397034$ after correction (Monte Carlo
standard errors $0.000019079$, $0.000002000$ and $0.000018762$): the proxy
precisely estimates the wrong target, $-0.0028$. Corrected-to-labeled estimator
SD ratios are $0.9801$, $1.0000$ and $1.3044$ for the binary oracle, blind and
reversed probes and $0.7979$ for the graded oracle, so precision gains depend on
the probe.

\paragraph{Two-sided bounds cover but do not decide.}
Both empirical-Bernstein procedures have $100\%$ observed simultaneous coverage
and retain $P$ in every repetition. At $(J_L,J_U)=(2000,8000)$ the range-only
half-widths are about $0.01655$ for paired labels and $0.04088$ for correction,
exceeding the largest true improvement, $0.0104$. A normal benchmark with
Bonferroni critical values has observed simultaneous coverage $0.971$--$0.991$
for paired labels and $0.9725$--$0.991$ for correction across the 28 settings,
with asymptotic justification only. This motivates the one-sided adoption test
of Proposition~\ref{prop:betting-main}.

\subsection{Adoption tests with matched finite inference}
\label{app:followup-selection}

\paragraph{Procedures.}
For non-original candidate $h$ and prompt $j$, $D_{jh}$ is the prompt-averaged
squared-loss difference from $P$. A certificate $\delta_h\ge\sup|h-P|$ gives the
range $[-B_h,B_h]$ with $B_h=\delta_h(2-\delta_h)$; observed maxima are not
certificates. For a scalar map, $\sup_p|\hat\tau(p)-p|$ is computed exactly from
the fitted map, while a feature-based candidate that admits no tighter bound
takes the certificate implied by its output range, which widens $B_h$. Both
primary tests use the one-sided null that $h$ does not lower expected squared
loss, the same certified ranges, the same family of $H$ candidates, the same
fixed comparison budget and familywise $\alpha=0.05$. The one-sided
empirical-Bernstein (EB) test \citep{maurer2009empirical} certifies $h$ when
\[
\overline D_h+\sqrt{\frac{2s_h^2\log(2H/\alpha)}{n}}
+\frac{7(2B_h)\log(2H/\alpha)}{3(n-1)}<0,
\]
and betting certifies $h$ when the e-value of Proposition~\ref{prop:betting},
with $\mathcal S=\{.05,.1,.25,.5,.9\}$, reaches $H/\alpha$. Both adopt the
certified candidate with the lowest paired mean and otherwise retain $P$.
Assisted variants add a frozen probe's contrasts on independent unlabeled
prompts (Appendix~\ref{app:betting}), with the weight grid
$\{0,.25,.5,.75,1\}$ and candidate-by-weight Bonferroni correction. Two-sided EB
intervals, which use $\log(4H/\alpha)$, and same-sample PPI++
\citep{angelopoulos2024ppiplusplus}, with weights clipped to $[0,1]$, serve as
references; PPI++ is asymptotic. Guarantees apply to a prespecified budget and
procedure, not to choices made after inspecting the sweeps.

\paragraph{Known-risk simulations.}
The matrix crosses six finite-support populations (binary or graded null,
constant gain and heterogeneous gain), oracle or reversed probes,
$n\in\{25,50,100,250,500,1000,2000\}$ labeled prompts, and $r\in\{0,1,4\}$
unlabeled prompts per label. Each of the 252 cells has 2,000 datasets and the
same five candidates. A prompt state is uniform over four conditional means
$\mu=(.2,.4,.6,.8)$ with $z=(1,-1,1,-1)$. The original score is $P=\mu+o$, where
$o=.04$ under the null, $.16$ under constant gain, and $(.02,.02,.18,.18)$ under
heterogeneous gain; the alternatives are $\mu-.04$, $\mu+.04z$, $\mu-.04z$ and
$\mu+o+.005$. Binary siblings are Bernoulli($\mu$); graded siblings are
$\mu\pm.04$ with equal probability. With probability .8 all eight siblings share
one draw. The null therefore has three risk-tied repairs and one harmful repair,
and adopting a tie counts as a false improvement. Across 168,000 null decisions
per method, paired betting certifies a tie 50 times and assisted betting 26
times; neither ever selects a harmful candidate, the largest cell rate is $.2\%$,
and the EB tests make no null certification. Asymptotic PPI++ made 3,120 false
adoptions, 41 of them harmful (largest cell rate $3.25\%$).

\begin{table}[t]
\centering\small\setlength{\tabcolsep}{4pt}
\caption{First tested label count at which beneficial adoption has Wilson
lower bound at least $80\%$ and false improvement Wilson upper bound at most
$5\%$ (oracle probe, $r=4$; ``--'': no crossing through 2,000). The middle four
columns share one-sided hypotheses, certified ranges, candidate family, and
error budget; PPI++ is asymptotic.}
\label{tab:followup-selection}
\begin{tabular}{lcccccc}
\toprule
Population & \shortstack{Two-sided\\EB} & \shortstack{One-sided\\EB} & Betting & \shortstack{Assisted\\EB} & \shortstack{Assisted\\betting} & PPI++\\
\midrule
Binary constant gain & 2,000 & 2,000 & 1,000 & 2,000 & 1,000 & 500\\
Graded constant gain & 500 & 500 & 250 & 1,000 & 250 & 25\\
Binary heterogeneous gain & -- & 2,000 & 1,000 & -- & 1,000 & 1,000\\
Graded heterogeneous gain & 1,000 & 1,000 & 250 & 2,000 & 250 & 25\\
\bottomrule
\end{tabular}
\end{table}

Matched to the same hypotheses, betting reaches the criterion with half the
labels in three populations and a quarter in the fourth
(Table~\ref{tab:followup-selection}, Figure~\ref{fig:matched-power}). Paired
betting and one-sided EB do not use the probe, so their crossings are the same
for oracle and reversed probes and for every unlabeled ratio: across all 24
alternative settings, betting crosses at 250--1,000 labels and one-sided EB at
500--2,000. The two-sided intervals overstate the binary heterogeneous gap,
since the matched one-sided EB test crosses at 2,000. Proxy assistance does not
improve betting's crossings; constant-gain oracle contrasts have no variation for
a proxy to remove.


\paragraph{Cached LLM panels.}
The replay uses all feasible allocations from 30 raw-score panels (16 shared
candidates, including five PAIC maps, two of them controlled perturbations) and
180 archived-score panels (eight shared candidates), with learned or
reversed auxiliary probes; candidates and allocations are unchanged and nothing
is refitted. False adoption is classified by a separate test set
(Table~\ref{tab:matched-real}). At 250 comparison prompts, beneficial adoption is
$16.81\%$ for one-sided EB versus $66.73\%$ for betting, and $37.82\%$ versus
$79.28\%$ at 500; at least one PAIC map is certified in $14.40\%$ versus
$50.10\%$ and $32.73\%$ versus $67.55\%$ of allocations. Feasible panels and
unlabeled ratios differ with budget, so these pooled percentages are not a
common-population learning curve. Richer direct predictors account for
$99.29\%$ of betting's adoptions. Counting the same 500 source fitting labels, 14
of 36 panel/ratio combinations in which both tests reach $80\%$ beneficial
adoption show at least twofold fewer total labels for betting. On the archived
learned score, all five procedures retain the original in all 117,000
learned-probe allocations per method: its comparison pools hold 25--100 labels
per block, and its scalar-map gains are below $10^{-3}$.

\begin{table}[t]
\centering\small\setlength{\tabcolsep}{4pt}
\caption{Matched procedures on 47,400 raw-score allocations per method
(learned probes, 16 candidates). ``PAIC certified'' counts allocations in which
at least one PAIC map is certified; false adoption is estimated on a separate
test set.}
\label{tab:matched-real}
\begin{tabular}{lrrrrr}
\toprule
Procedure & \shortstack{Beneficial\\adoption} & \shortstack{PAIC\\certified} & \shortstack{PAIC\\selected} & \shortstack{Richer\\selected} & \shortstack{False\\adoption}\\
\midrule
One-sided EB & 7,155 (15.09\%) & 5,794 (12.22\%) & 47 & 7,108 & 0\\
Betting & 20,356 (42.95\%) & 15,635 (32.99\%) & 125 & 20,212 & 0\\
Assisted EB & 6,610 (13.95\%) & 5,440 (11.48\%) & 35 & 6,575 & 0\\
Assisted betting & 19,012 (40.11\%) & 14,598 (30.80\%) & 183 & 18,815 & 0\\
Two-sided EB & 6,358 (13.41\%) & 5,207 (10.99\%) & 130 & 6,228 & 0\\
\bottomrule
\end{tabular}
\end{table}

\subsection{Controlled residual audits: what each test can detect}
\label{app:controlled-audit}

Opposite subgroup errors can cancel at each score level. We ask which
audit features can detect these errors under dependent answers and unequal
answer counts. Scores are fixed, subgroup information is
observed, and conditional means are known. A rerun of all $37$
configurations reproduces every rejection count
(Appendix~\ref{app:followup-audit}).


\paragraph{Population, controls and label budgets.}

Each independent prompt draws base rate $S\in\{0.35,0.65\}$ and
subgroup $H\in\{-1,1\}$ independently and uniformly. Binary outcomes
have conditional mean $S+\delta H$. The reported score is either $s=S$ or
the fixed invertible distortion
$s=f(S)=\operatorname{sigmoid}(\operatorname{logit}(S)/0.7)$.
The $30$ core configurations cross both scores with
$\delta\in\{0,0.05,0.10,0.20,0.35\}$ and
$J\in\{100,500,2000\}$, using one answer per prompt.
With $s=S$, positive $\delta$ introduces hidden subgroup error while
preserving scalar calibration; with $s=f(S)$, $\delta=0$ introduces
scalar error alone. Positive $\delta$ and distortion combine both.

Seven additional configurations hold $J=500$ and $s=S$. Four use
$K=8$ answers per prompt, $\rho\in\{0,0.8\}$ and $\delta\in\{0,0.10\}$.
Here $\rho$ is the probability of one shared outcome for all answers
to a prompt; otherwise outcomes are independent conditional on $S,H$.
Two configurations use the same mechanism with $\rho=0.8$,
$\delta\in\{0,0.10\}$, and $K=2$ for $H=-1$ versus $K=16$ for
$H=1$. The final, calibrated graded control uses $K=8$, $\rho=0.8$
and $Y=S+\sigma(H)\epsilon$, where $\epsilon$ is a fair sign,
$\sigma(-1)=0.05$ and $\sigma(1)=0.15$. Its values already lie in
$[0.2,0.8]$, so the implementation's clipping leaves the declared mean
unchanged. The shared-outcome mechanism also applies to this control.

All methods receive the same $J$ labeled audit prompts. In each repetition,
the learned auditor also uses $500$ independent, single-answer training
prompts from the same outcome model. It fits residuals
$Y-s$ by ridge regression on the four joint-state indicators, with
summed squared loss, penalty $1$ and no intercept. Its predictions
$h(S,H)$ are clipped to $[-1,1]$ and frozen before auditing. These
additional training labels are counted separately.

\paragraph{A common statistic, four different null hypotheses.}

Each test asks whether mean residuals align with its features, defining
its own null (Table~\ref{tab:study2-nulls}). The one-hot vectors
$e_{(S,H)}$ and $e_S$ indicate joint-state and score membership.
All kernels use
$k(z,z')=\langle\Phi(z),\Phi(z')\rangle$, without a bandwidth search.
The feature, subgroup and learned tests observe $S,H$ directly, without learning
subgroup representations from LLM outputs.

\begin{table}[!ht]
\centering\small
\begin{tabular}{@{}p{0.18\linewidth}p{0.22\linewidth}p{0.52\linewidth}@{}}
\toprule
Test & Feature map $\Phi$ & Null hypothesis under the prompt-weighted law \\
\midrule
Feature & $e_{(S,H)}$ & $\E[Y-s\mid S,H]=0$ in every joint state. \\
Score-only & $e_S$ & $\E[Y-s\mid S]=0$ at each score level. \\
Subgroup & $H$ & $\E[(Y-s)H]=0$. \\
Learned auditor & $h(S,H)$ & $\E[(Y-s)h(S,H)]=0$, conditional on the fitted $h$. \\
\bottomrule
\end{tabular}
\caption{\textbf{An audit's null depends on its features.}
Hidden-only error satisfies the score-only null; scalar-only error
satisfies the subgroup null. Their non-rejection in these cases is not
failure to detect an alternative to their own null.}
\label{tab:study2-nulls}
\end{table}

For every map, the residual contribution first averages within a prompt:
\[
 a_j=\frac1{K_j}\sum_{i=1}^{K_j}(Y_{ji}-s_{ji})\Phi(Z_{ji}),
 \qquad
 U_J=\frac{\|\sum_j a_j\|^2-\sum_j\|a_j\|^2}{J(J-1)}.
\]
This weights prompts equally even when $K$ depends on $H$ and excludes
all within-prompt pairs. The observed $U_J$ is uncentered and may be
negative. For calibration, $x_j=a_j-\bar a$ and independent Gaussian
prompt multipliers give
\[
 T_b^*=\frac{\|\sum_j e_{bj}x_j\|^2-
                     \sum_j e_{bj}^2\|x_j\|^2}{J-1},
 \qquad
 p=\frac{1+\sum_{b=1}^{1999}\mathbf1\{T_b^*\ge JU_J\}}{2000}.
\]
Methods share multiplier draws within each repetition;
each rejects at $p\le0.05$. Theorems~\ref{thm:ustat}
and~\ref{thm:cluster-bootstrap} establish the population target and
asymptotic justification. With fixed $J$ and $1{,}999$ multiplier draws,
finite-sample rejection rates remain an empirical question.

Each configuration has $2{,}000$ repetitions. A fixed master seed, with
configuration and repetition indices, determines separate simulation and
multiplier streams. Separate child streams generate
audit and auditor-training data. Reported $95\%$ Wilson intervals
describe each marginal Monte Carlo rejection rate. They are pointwise,
not simultaneous intervals or paired intervals for method differences.

\paragraph{Results and interpretation.}

Feature choice determines which residual moments are tested. Hidden subgroup
errors cancel in the score-only moment; scalar-only errors cancel in
the subgroup moment. Each satisfies the corresponding test's null.

Across the seven fully calibrated configurations, feature-test rejection
ranges from $0.0485$ to $0.0590$; each pointwise interval contains
$0.05$. This includes the stated sibling, unequal-$K$ and graded controls,
without establishing validity outside them. The graded-null learned
auditor rejects $126/2000=0.0630$ (pointwise Wilson interval
$[0.05317,0.07451]$); with a new seed it rejects $4.70\%$
(Appendix~\ref{app:followup-audit}). Across the $46$ method-specific null
cells, $45$ pointwise intervals contain $0.05$; this count is neither
a simultaneous validity statement nor an equivalence test.

Under hidden-only error at $\delta=0.05$, feature-test rejection
increases from $0.1265$ to $0.4445$ to $0.9790$ as $J$ increases from
$100$ to $500$ to $2{,}000$. For scalar-only error at $J=500$, the
score-only test has higher observed rejection: $0.6770$, versus
$0.5575$ for the feature test.
These descriptive differences support choosing features for the errors
of interest.
Table~\ref{tab:study2-all} retains every configuration.

\paragraph{Checks and scope.}
Deterministic fixtures check the cross-prompt formula and centered
multiplier arithmetic against explicit pair sums. Audit-label budgets are $J$ prompts per repetition; the auditor's
additional $500$ training labels are counted separately. These known-state
simulations do not validate an LLM representation or a real-data repair, and
non-rejection concerns the specified moment only.

\begingroup
\footnotesize
\setlength{\tabcolsep}{3pt}
\renewcommand{\arraystretch}{1.0}
\begin{longtable}{@{}lcccc@{}}
\caption{All 37 controlled audit configurations: rejection rates from 2,000 repetitions per configuration at nominal level 0.05. Pointwise 95\% Wilson half-widths are about 0.01 near 0.05 and at most 0.022. A superscript $\circ$ marks a method-specific population null by construction; for the learned auditor this applies under full calibration.}
\label{tab:study2-all}\\
\toprule
Configuration & Feature & Score only & Subgroup & Auditor \\
\midrule
\endfirsthead
\multicolumn{5}{l}{\textit{Table \thetable{} continued. $\circ$: method-specific population null.}}\\
\toprule
Configuration & Feature & Score only & Subgroup & Auditor \\
\midrule
\endhead
\midrule
\multicolumn{5}{r}{\textit{Continued on next page}}\\
\endfoot
\bottomrule
\endlastfoot
\multicolumn{5}{l}{\textbf{Identity score $s=S$, binary outcomes, $K=1$; configuration: $(J,\delta)$}}\\
$(100,\ 0)$ & 0.0590$^\circ$ & 0.0540$^\circ$ & 0.0570$^\circ$ & 0.0510$^\circ$ \\
$(100,\ 0.05)$ & 0.1265 & 0.0490$^\circ$ & 0.1940 & 0.1475 \\
$(100,\ 0.1)$ & 0.3625 & 0.0510$^\circ$ & 0.5950 & 0.5500 \\
$(100,\ 0.2)$ & 0.9475 & 0.0505$^\circ$ & 0.9930 & 0.9915 \\
$(100,\ 0.35)$ & 1.0000 & 0.0515$^\circ$ & 1.0000 & 1.0000 \\
$(500,\ 0)$ & 0.0510$^\circ$ & 0.0460$^\circ$ & 0.0500$^\circ$ & 0.0490$^\circ$ \\
$(500,\ 0.05)$ & 0.4445 & 0.0540$^\circ$ & 0.6540 & 0.4680 \\
$(500,\ 0.1)$ & 0.9735 & 0.0455$^\circ$ & 0.9975 & 0.9925 \\
$(500,\ 0.2)$ & 1.0000 & 0.0430$^\circ$ & 1.0000 & 1.0000 \\
$(500,\ 0.35)$ & 1.0000 & 0.0520$^\circ$ & 1.0000 & 1.0000 \\
$(2,000,\ 0)$ & 0.0545$^\circ$ & 0.0445$^\circ$ & 0.0480$^\circ$ & 0.0435$^\circ$ \\
$(2,000,\ 0.05)$ & 0.9790 & 0.0580$^\circ$ & 0.9980 & 0.8880 \\
$(2,000,\ 0.1)$ & 1.0000 & 0.0495$^\circ$ & 1.0000 & 1.0000 \\
$(2,000,\ 0.2)$ & 1.0000 & 0.0525$^\circ$ & 1.0000 & 1.0000 \\
$(2,000,\ 0.35)$ & 1.0000 & 0.0495$^\circ$ & 1.0000 & 1.0000 \\
\midrule
\multicolumn{5}{l}{\textbf{Distorted score $s=f(S)$, binary outcomes, $K=1$; configuration: $(J,\delta)$}}\\
$(100,\ 0)$ & 0.1230 & 0.1720 & 0.0570$^\circ$ & 0.1645 \\
$(100,\ 0.05)$ & 0.1900 & 0.1665 & 0.1945 & 0.2625 \\
$(100,\ 0.1)$ & 0.4550 & 0.1795 & 0.5880 & 0.6045 \\
$(100,\ 0.2)$ & 0.9595 & 0.1600 & 0.9930 & 0.9880 \\
$(100,\ 0.35)$ & 1.0000 & 0.1535 & 1.0000 & 1.0000 \\
$(500,\ 0)$ & 0.5575 & 0.6770 & 0.0535$^\circ$ & 0.6140 \\
$(500,\ 0.05)$ & 0.8425 & 0.6990 & 0.6645 & 0.8820 \\
$(500,\ 0.1)$ & 0.9960 & 0.6705 & 0.9980 & 0.9990 \\
$(500,\ 0.2)$ & 1.0000 & 0.6665 & 1.0000 & 1.0000 \\
$(500,\ 0.35)$ & 1.0000 & 0.6685 & 1.0000 & 1.0000 \\
$(2,000,\ 0)$ & 0.9980 & 0.9995 & 0.0525$^\circ$ & 0.9620 \\
$(2,000,\ 0.05)$ & 1.0000 & 0.9995 & 0.9950 & 0.9990 \\
$(2,000,\ 0.1)$ & 1.0000 & 0.9990 & 1.0000 & 1.0000 \\
$(2,000,\ 0.2)$ & 1.0000 & 0.9995 & 1.0000 & 1.0000 \\
$(2,000,\ 0.35)$ & 1.0000 & 0.9990 & 1.0000 & 1.0000 \\
\midrule
\multicolumn{5}{l}{\textbf{Sibling and graded controls: $J=500$, $s=S$}}\\
$K=8,\ \rho=0$, $\delta=0$ & 0.0485$^\circ$ & 0.0535$^\circ$ & 0.0500$^\circ$ & 0.0525$^\circ$ \\
$K=8,\ \rho=0$, $\delta=0.1$ & 1.0000 & 0.0485$^\circ$ & 1.0000 & 1.0000 \\
$K=8,\ \rho=0.8$, $\delta=0$ & 0.0500$^\circ$ & 0.0485$^\circ$ & 0.0570$^\circ$ & 0.0490$^\circ$ \\
$K=8,\ \rho=0.8$, $\delta=0.1$ & 0.9940 & 0.0520$^\circ$ & 1.0000 & 0.9965 \\
$K=2/16,\ \rho=0.8$, $\delta=0$ & 0.0505$^\circ$ & 0.0490$^\circ$ & 0.0465$^\circ$ & 0.0525$^\circ$ \\
$K=2/16,\ \rho=0.8$, $\delta=0.1$ & 0.9915 & 0.0475$^\circ$ & 0.9995 & 0.9980 \\
$K=8,\ \rho=0.8$, $\delta=0,\ \mathrm{graded}$ & 0.0545$^\circ$ & 0.0500$^\circ$ & 0.0580$^\circ$ & 0.0630$^\circ$ \\
\end{longtable}
\endgroup

\subsection{Audit reruns and real-data audits}
\label{app:followup-audit}

\paragraph{Reruns of the controlled audits.}
Reruns of all 37 configurations of Appendix~\ref{app:controlled-audit} and nine
diagnostics reproduce all 148 original rejection counts. The graded-null learned
auditor rejects at $6.30\%$; another seed gives $4.70\%$, and the excess does not
survive the declared Holm family. A binary 100-prompt Gaussian cell rejects
$7.25\%$ and does survive that correction; replaying the identical observations
against an exact conditional binomial reference gives $6.60\%$, so that
realization is rejection-rich even for a valid exact test. A 48-cell,
96,000-trial study varies sample size, outcome type, witness rank and
multiplier count; at 100 prompts, binary Gaussian rates with 1,999 multipliers
are $4.80/5.25\%$ across two seeds versus $4.25/4.30\%$ for the exact reference,
and graded rates are $5.95/6.20\%$ versus $5.25/5.25\%$ for symmetric sign flips.
No excess survives the 280-comparison Holm family, which is not a proof of
uniform validity; exact or sign-flip references are the safer choice for small
audits when the null permits.

\paragraph{A paired diagnostic of the graded-null cell.}
\label{app:paired-auditor-diagnostic}
We ran 18,000 further repetitions of the graded-null cell with 500 audit
prompts, eight answers per prompt, sibling dependence $\rho=.8$ and 1,999
reference draws. Each repetition trains and freezes its witness before auditing.
Three rejection indicators share the same outcomes: $G_L$ uses the Gaussian
procedure with residuals computed as in the controlled study (explicit outcome
construction, averaging and score subtraction); $G_I$ uses the Gaussian
procedure with exact, sign-symmetric residuals; and $S_I$ uses exact residuals
with a sign-flip reference. The two Gaussian arms share multiplier draws. The prespecified
contrasts are $\theta_G=\mathbb E[G_I-S_I]$ and $\theta_F=\mathbb E[G_L-G_I]$.
For each, we bound its two discordance probabilities separately and subtract
the bounds; eight one-sided exact Clopper--Pearson tails at $.00625$ give
simultaneous coverage at least $95\%$ for both intervals.

\begin{center}
\small
\begin{tabular}{@{}lccc@{}}
\toprule
Contrast & Left/right-only rejections & Estimate & Simultaneous interval \\
\midrule
$\theta_G$ & $73/54$ & $0.1056$ & $[-0.1210,\,0.3319]$ \\
$\theta_F$ & $0/0$ & $0$ & $[-0.0282,\,0.0282]$ \\
\bottomrule
\end{tabular}
\end{center}
Estimates and intervals are in percentage points. The bounds exclude a
$0.5$-point Gaussian excess, and the floating-point residual computation changes
no decision in this cell. Marginal rejections are $959/18{,}000$ for each
Gaussian arm and $940/18{,}000$ for $S_I$, descriptive rather than a test of
absolute size. The diagnostic does not identify the cause of the earlier
126/2,000 rate, and it concerns this cell only.

\paragraph{Audits of repaired LLM scores.}
The raw-score audit fixes discovery-trained kernels, ridge witnesses, score bins
and metadata groups before evaluating independent audit prompts. For ridge
PAIC, approximate feature tests reject in 30 of 30 panels, score tests in 7 and
metadata tests in 19; conservative bounded tests reject none. A partition audit
adapts the local grouping-loss approach of \citet{melo2026epistemic}: discovery
fits honest trees, optionally after spending its first $20\%$ on Platt
calibration; debiased leaf energy subtracts within-leaf sampling variance; and
only finite-rejecting leaves receive an audit-estimated residual offset, fixed
before test evaluation. Across 300 raw-score plans, seven corrections in five
panels trigger; all seven lower test point risk, by .02474--.05914, with
favorable pointwise normal intervals, though none passes the conservative
finite-sample certificate. On the archived learned score, 1,440 plans across 180
panels with 75--100 audit prompts each produce no finite rejection;
non-rejection does not establish sufficiency.

\paragraph{A discovered family}
All seven triggering leaves depend on answer length. One Qwen task-transfer leaf
contains 510 test answers from 179 prompts, with mean raw confidence .91565,
correctness .01176 and mean length 1,023.93 tokens. Inspecting raw states, all 20
sampled long answers hit the 1,024-token cap without an end-of-sequence token or
closed boxed answer, whereas all 20 shorter siblings terminate with a closed box.

\paragraph{A mechanical screen.}
The learned rule flags answers longer than 1,018 tokens; a mechanical rule flags
exactly 1,024 tokens. The two groups nearly coincide (Table~\ref{tab:screen}). On
the 572-prompt, one-answer audit sample, the mechanical rule selects 32 prompts
and the learned rule 33, with Platt-adjusted two-sided Hoeffding $p=.01398$ and
$.01049$. Mechanical membership needs no labels; testing its residual bias uses
114 discovery labels for Platt calibration and 572 audit labels, and learning the
tree uses 458 further discovery labels. The audit surfaced an unanticipated
failure family; once found, a length rule describes it nearly exactly.

\begin{table}[t]
\centering\small\setlength{\tabcolsep}{4pt}
\caption{Learned audit rule (length $>1{,}018$ tokens) versus a mechanical
screen (exactly 1,024 tokens) over all eight siblings per prompt.}
\label{tab:screen}
\begin{tabular}{lrrrrrr}
\toprule
Partition & Prompts & \shortstack{Learned\\answers} & \shortstack{Mechanical\\answers} & Jaccard & \shortstack{Learned-group\\error} & \shortstack{Mechanical-\\group error}\\
\midrule
Discovery & 572 & 200 & 196 & 98.00\% & 97.00\% & 98.98\%\\
Audit & 572 & 243 & 238 & 97.94\% & 99.59\% & 100.00\%\\
Final evaluation & 1,435 & 510 & 498 & 97.65\% & 98.82\% & 99.20\%\\
\bottomrule
\end{tabular}
\end{table}

\subsection{Decisions and incremental scoring cost}
\label{app:followup-policies}
\label{app:followup-policy}

At prespecified review costs $c\in\{.05,.10,.20,.30,.40,.50\}$, accepting incurs
$1-Y$ and reviewing incurs $c$, assuming perfect review, and a value policy
accepts at score $1-c$. A strictly monotone remap preserves ranking and isotonic
pooling can add ties, so these experiments assess threshold choices and their
transfer rather than risk--coverage curves. Baselines fit original-score
thresholds directly on the same labels, fixed before test outcomes are
evaluated. On the archived learned score, 21,600 policies over 180 settings give
useful acceptance without a consistent advantage over directly fitted
thresholds (Table~\ref{tab:followup-policy}). For GSM8K Qwen at $c=.20$, PAIC
accepts $76.87\%$ with $5.10\%$ accepted error and loss .085398, versus .086033
for the directly fitted threshold; at $c=.10$ its loss .055525 exceeds that
baseline's .055109. At the full fitting budget, when both policies also use
comparison labels, the comparison-fitted PAIC threshold has higher mean loss in
23 of 24 cell--cost settings than the original-score threshold fitted on $F$
plus comparison labels. At $c=.05$, Llama 10K PAIC accepts $.56\%$ of answers on
average, with loss .050141, slightly worse than reviewing everything. Matching
downstream label budgets matters: recalibration need not improve these
decisions.

\begin{table}[t]
\centering\small\setlength{\tabcolsep}{4pt}
\caption{Archived learned score, full fitting budget: mean PAIC value-policy
loss minus the original-score threshold fitted directly on the same $F$ labels.
Negative favors PAIC. Each entry averages 15 overlapping blocks and is
descriptive, not an independent-refit interval.}
\label{tab:followup-policy}
\begin{tabular}{llrrrrrr}
\toprule
Dataset & Model & $c=.05$ & $.10$ & $.20$ & $.30$ & $.40$ & $.50$\\
\midrule
GSM8K & Llama & $-.000539$ & $-.000410$ & $+.000367$ & $-.000301$ & $+.000067$ & $-.000320$\\
 & Qwen & $+.000101$ & $+.000415$ & $-.000634$ & $-.000034$ & $+.000359$ & $-.000267$\\
10K & Llama & $+.000100$ & $-.000110$ & $-.000308$ & $-.000318$ & $-.000056$ & $-.000207$\\
 & Qwen & $+.000027$ & $+.000120$ & $-.000282$ & $+.000055$ & $-.000179$ & $+.000308$\\
\bottomrule
\end{tabular}
\end{table}

On cached inputs, with one thread, five warmups, twenty repetitions and
randomized order, median batch-one latency is 14.5/8.7 microseconds for
Llama/Qwen PAIC versus 263.1/142.6 for direct ridge and 3,616.0/1,863.9 for direct
trees; at batch 256, PAIC takes .137/.113 microseconds per answer. These are
incremental costs on cached inputs, excluding generation and state extraction.

\clearpage
\section{Fixed-choice confidence on BoolQ and ARC}
\label{app:benchmark-breadth}

\paragraph{Question and design.}
The study asks whether hidden states help predict correctness beyond the full
choice scores, and whether that information survives when deployment is
restricted to a monotone function of one scalar score. Each model selects one
option once per question, and every candidate predicts whether that fixed
selection is correct. The design crosses four checkpoints, Llama 3.2 3B
Instruct \citep{meta2024llama32}, Qwen 2.5 3B and 7B Instruct
\citep{qwen2025technical} and Mistral 7B Instruct v0.3
\citep{mistral2024instruct03}, with two benchmarks. The same questions occur
across models, so the eight cells are not independent replications, and
pretraining exposure to these public benchmarks is unknown.

\paragraph{Tasks and data separation.}
BoolQ \citep{clark2019boolq} supplies passage-based yes/no questions; its public
validation split is the evaluation population $T$, and its training data supply
development $D$, fitting $F$ and two unlabeled projection sets $C$ and $U$ from
the same population. ARC \citep{clark2018arc} supplies a difficulty shift: Easy
training supplies $D/F$, Easy validation $C$ and Easy test a source evaluation
set $S$; Challenge training and validation supply unlabeled target inputs $U$,
and Challenge test is $T$ (Table~\ref{tab:breadth-populations}). Groups of
questions linked by normalized exact matches of questions, passages or IDs are
kept within one role, and all predictions were saved before evaluation labels
were joined.

\begin{table}[H]
\centering\small
\caption{Question counts after outcome-blind overlap filtering, shared by all
models. BoolQ has no separate source-evaluation role.}
\label{tab:breadth-populations}
\begin{tabular}{lrrrrrr}
\toprule
Benchmark & $D$ & $F$ & $C$ & $U$ & $S$ & $T$ \\
\midrule
BoolQ & 400 & 4,393 & 1,597 & 1,598 & --- & 3,270 \\
ARC & 112 & 2,126 & 570 & 1,410 & 2,375 & 1,172 \\
\bottomrule
\end{tabular}
\end{table}

\paragraph{Scores and candidates.}
Each allowed option letter is one prefix-stable token with next-token logit
$\ell_j$; the choice probabilities are $q_j=\exp(\ell_j)/\sum_{k}\exp(\ell_k)$,
the model selects the first maximizer $\widehat j$, and the scalar confidence is
$P=q_{\widehat j}$, conditional on the allowed choices. Every supervised
candidate uses the same $F$ questions, and five label-blind folds that keep
linked questions together select hyperparameters by out-of-fold Brier loss. The
seven direct candidates are raw $P$, label-isotonic, monotone Platt, a flexible
$P$-only predictor, ridge and tree predictors of the 24 choice-score features,
and a rich ridge predictor that adds every coordinate of the middle and final
hidden states. Six further candidates project three teachers (the selected
scalar, score and rich predictors) onto nondecreasing functions of $P$ by
isotonic regression on $C$ or $U$, without their labels. Supervised labels are
therefore matched, while PAIC additionally uses unlabeled inputs.

\paragraph{Primary comparisons.}
At the full fitting budget on $T$, each of the eight cells contributes four
Brier-risk differences (first minus second; negative favors the first): rich
direct minus the selected score predictor; rich-PAIC-$U$ minus the selected
scalar predictor; rich-PAIC-$U$ minus score-PAIC-$U$; and rich-PAIC-$U$ minus
rich direct. All 32 effects form one family of approximate simultaneous $95\%$
intervals from 19,999 dependency-component bootstrap draws, conditional on the
fits and fitting-based selections (Figure~\ref{fig:breadth-primary}). Failure to
resolve a difference does not establish equivalence.

\paragraph{Results.}
Hidden states add resolved accuracy over the full choice scores in two of eight
cells: rich ridge minus the selected score control is $-0.00803$ for BoolQ
Qwen3B (joint interval $[-0.01243,-0.00362]$) and $-0.01131$ for BoolQ
Mistral7B ($[-0.01702,-0.00559]$); the other six differences are unresolved.
Rich-PAIC-$U$ versus the selected scalar control is unresolved in all eight cells (point differences
negative in five), and so is rich-PAIC-$U$ versus score-PAIC-$U$. Serving the
rich predictor directly is more accurate than projecting it for each BoolQ model
(differences $0.0121$--$0.0169$) and for ARC Qwen3B ($0.0101$); the other three
ARC differences are unresolved. BoolQ's projection sets share its fitting
population, and ARC's $C$ and $U$ differ in both population and size, so this
design does not isolate the effect of adaptation.

\paragraph{Decisions.}
For costs $c\in\{.05,.10,.20,.30,.50\}$, a fixed rule accepts when $s\geq1-c$ and
otherwise requests perfect review. On BoolQ, every candidate beats both
reference policies at $c=.20$, but the pattern changes at other costs: at
$c=.50$, rich-PAIC-$U$ equals the accept-all loss for Llama3B and Qwen7B and
exceeds it for Qwen3B and Mistral7B. On ARC Mistral7B at $c=.05$, rich-PAIC-$U$
accepts nothing. The rich direct predictor has its own limits: on ARC Llama3B
its decision loss exceeds that of the selected score-ridge control at every
cost, although their Brier contrast is unresolved.

\subsection{A larger checkpoint}
\label{app:e13-larger-model}

We repeat the protocol for Qwen2.5-32B-Instruct, fitting the same 13 candidates
at 25\%, 50\% and 100\% of $F$ and forming a separate family of eight
simultaneous intervals for the four questions on each benchmark. The selected
full-score control is score trees on BoolQ and score ridge on ARC; isotonic is
the selected scalar control on both. Six of eight intervals cross zero
(Table~\ref{tab:e13-primary}): the richer direct predictor resolves no gain
over the full-score control on either benchmark, and rich-PAIC-$U$ resolves no
gain over the scalar control or score-PAIC-$U$. The only resolved contrasts
compare rich-PAIC-$U$ with its direct rich teacher, with opposite signs:
projection increases risk on BoolQ and lowers it on ARC. The value of serving a
richer teacher through one monotone scalar thus remains population-dependent at
this scale; one checkpoint does not establish a trend with model size.

\begin{table}[t]
\centering\small
\setlength{\tabcolsep}{4pt}
\begin{tabular*}{\linewidth}{@{\extracolsep{\fill}}llrrrl@{}}
\toprule
Benchmark & Contrast & Effect & Lower & Upper & Direction \\
\midrule
BoolQ & Rich $-$ score & $-0.00297949$ & $-0.00637632$ & $0.00041735$ & unresolved \\
BoolQ & PAIC-R-$U$ $-$ scalar & $0.00010225$ & $-0.00079915$ & $0.00100366$ & unresolved \\
BoolQ & PAIC-R-$U$ $-$ PAIC-S-$U$ & $-0.00032620$ & $-0.00092227$ & $0.00026986$ & unresolved \\
BoolQ & PAIC-R-$U$ $-$ rich & $0.01234145$ & $0.00720879$ & $0.01747412$ & first higher risk \\
\midrule
ARC & Rich $-$ score & $0.00257319$ & $-0.00239375$ & $0.00754012$ & unresolved \\
ARC & PAIC-R-$U$ $-$ scalar & $0.00109614$ & $-0.00309369$ & $0.00528598$ & unresolved \\
ARC & PAIC-R-$U$ $-$ PAIC-S-$U$ & $0.00042009$ & $-0.00386352$ & $0.00470370$ & unresolved \\
ARC & PAIC-R-$U$ $-$ rich & $-0.00234922$ & $-0.00461969$ & $-0.00007876$ & first lower risk \\
\bottomrule
\end{tabular*}
\caption{\textbf{Qwen2.5-32B primary contrasts.} Approximate joint 95\%
intervals from 19,999 question-weighted dependency-component draws, conditional
on the fitted candidates and the $F$-based control selection. R/S abbreviate
rich/score teachers.}
\label{tab:e13-primary}
\end{table}


\section{Extended Related Work}
\label{app:related_detail}

\paragraph{Classical calibration and probabilistic prediction.}
Calibration asks whether predictive probabilities agree with empirical outcome
frequencies. Classical forecast decompositions separate reliability,
resolution, and irreducible uncertainty
\citep{murphy1973new,degroot1983,brocker2009}, while proper scoring rules
provide decision-theoretic objectives for probabilistic prediction
\citep{gneiting2007strictly}. In machine learning, post-hoc methods such as
Platt scaling, isotonic regression, temperature scaling, beta calibration, and
Dirichlet calibration modify predictive probabilities without retraining the
underlying model
\citep{platt1999probabilistic,niculescu2005predicting,zadrozny2002,
guo2017calibration,kull2017beta,kull2019dirichlet}. Reliability diagrams and
Expected Calibration Error (ECE) are widely used summaries, although their
interpretation depends on binning and finite-sample estimation
\citep{naeini2015obtaining,vaicenavicius2019evaluating,
kumar2019verified,roelofs2022mitigating}. More recent studies also emphasize
that calibration behavior depends on architecture, scale, pretraining, and
distribution shift rather than following a universal relationship with
accuracy~\citep{minderer2021revisiting,ovadia2019trust}. Our analysis differs
from work that asks only whether the reported scalar is calibrated: we ask
which part of correctness risk remains recoverable after the prediction has
already been compressed into that scalar.

\paragraph{Grouping loss, sufficiency, and feature-dependent calibration.}
A scalar confidence score can be calibrated while hiding systematic
differences among examples assigned the same value. Grouping loss formalizes
the predictive information lost through this compression
\citep{perezlebel2023grouping}; under squared loss, our information term
$I=\E[\Var(m^\star(Z)\mid P)]$ is precisely this quantity under the declared
representation $Z$. Non-injective recalibration may further merge distinct
score values and thereby increase grouping loss
\citep{perezlebel2023grouping}. Related approaches enrich the conditioning
information rather than repairing only the original scalar. Multiaccuracy
corrects systematic residuals over identifiable subpopulations
\citep{kim2018multiaccuracy}, while multicalibration requires calibration over
many computationally identifiable groups
\citep{hebertjohnson2018multicalibration}. Semantic and representation-aware
variants similarly use features to expose errors hidden by an aggregate score
\citep{yang2023grouping}. Our decomposition separates this pre-existing
information loss from an additional monotone-pooling term $G$: even the
population-optimal order-preserving repair can therefore be scalar-calibrated
while remaining insufficient relative to $Z$.

\paragraph{Group-dependent calibration for LLMs.}
Recent LLM work has begun to examine this distinction directly.
\citet{chen2024reconfidencing} use learned representation-based groups and fit
label-supervised isotonic calibrators within those groups. Their deployed
prediction therefore depends on both the original confidence and the learned
group assignment. \citet{melo2026epistemic} use grouping-loss estimates to
identify local epistemic uncertainty and support auditing and deferral. These
approaches exploit richer information at serving time. PAIC addresses a
different constraint: richer features or a correctness model may be available
during fitting, while deployment must retain a \emph{single map of the
original score}. Our theory therefore asks how much of a richer predictor can
be transferred into that restricted interface, and which errors of the richer
predictor affect the resulting scalar map.

\paragraph{Calibration under distribution shift.}
Calibration commonly deteriorates when the deployment distribution differs
from the calibration distribution
\citep{ovadia2019trust,minderer2021revisiting}. Under covariate shift,
importance weighting reweights labeled source observations to approximate
target expectations~\citep{shimodaira2000improving}, and analogous ideas have
been developed specifically for calibration
\citep{park2020calibrated,wang2020transferable}. Under label shift, class
priors can instead be estimated and corrected
\citep{lipton2018detecting,alexandari2020maximum}. Other work uses unlabeled
target data to estimate aggregate target performance, including ATC and
difference-of-confidence approaches
\citep{garg2022leveraging,guillory2021predicting}. These methods primarily
estimate target performance or reconstruct the target calibration objective
through assumptions on distribution shift. PAIC takes an outcome-model route:
a source-trained correctness probe predicts outcomes on unlabeled target
examples, and these predictions define the target-population scalar repair.
Accordingly, its relevant assumption is not an accurate target density ratio
but small \emph{score-conditional probe bias} on the target population.

\paragraph{Distillation and privileged information.}
PAIC is also related to knowledge distillation, where a richer teacher provides
soft supervision for a restricted student
\citep{hinton2015distilling}, and to learning using privileged information,
where additional features are available during training but unavailable at
test time~\citep{lopezpaz2016unifying}. Standard distillation typically asks
whether the student can approximate the teacher's predictions or improve task
accuracy. In our setting, the student hypothesis class is deliberately narrow:
it is a nondecreasing function of a pre-existing scalar score. Consequently,
global teacher accuracy is not the quantity governing transfer. Only the
teacher's systematic error after conditioning on the student input $P$ shifts
the population projection. The PAIC guarantee makes this distinction explicit
and shows why two probes with similar global prediction error can induce very
different scalar repairs.

\paragraph{Selective prediction, abstention, and decision utility.}
Confidence is often useful because it determines whether a model prediction is
accepted, rejected, reviewed, or deferred. Selective classification studies
this accuracy--coverage trade-off and constructs predictors that abstain on
low-confidence inputs~\citep{geifman2017selective}; related reject-option and
risk--coverage formulations likewise emphasize that uncertainty estimates are
valuable through the decisions they support. These objectives are distinct
from calibration itself: a monotone transformation can change probabilities
while preserving ranking, whereas a fixed decision threshold can change
behavior substantially after recalibration. Our decision experiments therefore
evaluate prediction risk and downstream review utility separately rather than
assuming that lower Brier risk implies a better selective policy. This
distinction is particularly important when the application assigns explicit
costs to acceptance and review.

\paragraph{Evidence for adopting a fitted repair.}
A lower empirical loss on the data used to construct a repair is not evidence
that the repair improves population risk. Prediction-powered inference (PPI)
combines abundant model predictions with a smaller labeled sample by correcting
the model-based estimate with observed residuals
\citep{angelopoulos2023ppi}; PPI++ adapts the prediction contribution to
improve statistical efficiency
\citep{angelopoulos2024ppiplusplus}. We specialize this model-assisted
estimation principle to \emph{paired squared-loss differences} between frozen
confidence candidates. The resulting identity also clarifies which component
of probe error can reverse a candidate comparison. For finite candidate
families, empirical-Bernstein inequalities provide observable simultaneous
bounds~\citep{maurer2009empirical}. When the operational question is only
whether a candidate improves over the deployed score, rather than estimation
of every pairwise contrast, we additionally use bounded-mean betting
\citep{waudbysmith2024betting}. Our contribution here is not a new generic
inference paradigm, but an explicit separation between \emph{constructing} a
confidence repair and obtaining independent statistical evidence that it
should be deployed.

\paragraph{Kernel calibration and residual-moment tests.}
Kernel methods provide a complementary way to detect calibration errors that
are invisible to low-dimensional summaries. MMCE
\citep{kumar2018trainable} and squared kernel calibration error (SKCE)
\citep{widmann2019calibration} embed calibration residuals in a reproducing
kernel Hilbert space, while distribution-matching formulations generalize
these ideas to richer calibration conditions
\citep{marx2023distribution}. In our setting the residual embedding
\[
    \E[(Y-s(Z))\Phi(Z)]
\]
tests whether mean correctness errors remain associated with the declared
features. With a sufficiently rich RKHS, a zero population embedding
identifies the conditional mean predictor; with a finite feature map it tests
only the corresponding residual moments. This feature dependence is
intentional: an audit can only detect errors represented by its chosen
features and kernel.

\paragraph{Dependence, clustered answers, and kernel inference.}
Standard calibration estimators commonly treat evaluated examples as
independent. LLM evaluation often violates this assumption because multiple
answers are sampled from the same prompt and can share reasoning paths,
difficulty, or failure modes. Pairing sibling residuals in a kernel
U-statistic can therefore introduce a nonvanishing finite-sample term even
when the population conditional-mean null holds. We instead treat prompts as
the independent sampling units and average residual features within prompt
before forming cross-prompt products. The corresponding multiplier procedure
is related to wild-bootstrap methods for degenerate kernel statistics
\citep{chwialkowski2014wild}, but is formulated for the prompt-clustered
statistic used here. This distinction is important for LLM settings with
best-of-$K$, self-consistency, repeated decoding, or other multiple-answer
evaluation protocols.

\paragraph{Confidence estimation in language models.}
LLM confidence can be derived from several fundamentally different signals.
Token and sequence likelihoods use the model's generative probability
\citep{jiang2021know}; self-evaluation methods ask the model whether its own
answer is correct, including P(True)
\citep{kadavath2022mostly}; verbalized-confidence methods elicit numerical or
linguistic uncertainty directly
\citep{lin2022teaching,tian2023just,xiong2024uncertainty}; and sampling-based
methods estimate uncertainty through agreement or disagreement among multiple
generations
\citep{wang2023selfconsistency,kuhn2023semantic,farquhar2024entropy}.
These quantities need not coincide: likelihood measures probability of the
generated text, whereas answer-level confidence targets the probability that
the realized answer is correct. Empirical studies consequently find that the
quality of different confidence signals varies with model, task, elicitation
method, and access regime. Our framework is agnostic to the origin of the base
score: it treats any bounded scalar as the object to be retained, repaired, or
compared against richer alternatives.

\paragraph{Semantic uncertainty and answer-level correctness.}
Sampling-based semantic uncertainty methods address the fact that different
surface forms can express the same underlying answer. Semantic entropy
aggregates probability across meaning-equivalent generations
\citep{kuhn2023semantic,farquhar2024entropy}, while self-consistency exploits
agreement across reasoning trajectories
\citep{wang2023selfconsistency}. \citet{nakkiran2026concepts} study conditions
under which calibration of semantic answer classes can emerge from local
sequence-level loss optimality. These works concern uncertainty over semantic
answers or distributions of generated responses. Our object is different: the
answer $A$ has already been generated and fixed, and we ask whether a score
attached to that realized answer predicts its correctness. Semantic or
sampling-based signals can enter our framework as components of $P$, as probe
features in $Z$, or as richer candidate predictors, but the repair,
comparison, and audit questions remain unchanged.

\paragraph{Internal representations and correctness probes.}
A growing line of work predicts hallucination or correctness from hidden
representations rather than from output probabilities alone
\citep{orgad2025intrinsic}. Such probes can reveal correctness information that
is absent from the reported scalar, making them natural witnesses of score
insufficiency. They also create an important comparison problem: if hidden
states are available at serving time, the direct probe should be compared
against any scalar projection of that probe rather than being used only as a
teacher. Conversely, when hidden states, additional sampling, execution
checks, or verifier calls are available offline but unavailable or too costly
at deployment, distilling their information into the existing scalar
interface becomes meaningful. Our experiments therefore distinguish the
\emph{information available for fitting} from the \emph{information permitted
at serving}.

\paragraph{Relationship to our framework.}
Taken together, prior work provides powerful tools for calibration,
group-dependent correction, adaptation, model-assisted inference, and residual
testing. Our framework connects these tools through a different organizing
question: \emph{what evidence is required before changing a deployed
confidence score?} The risk decomposition first separates error recoverable
from the scalar from information already lost by scalar compression; PAIC
studies target-population repair when richer supervision is available only
offline; labeled comparisons ask whether a fitted candidate actually improves
correctness risk; and a separate prompt-level audit tests structured residual
errors that average-risk comparisons can miss. Thus, ``repair,'' ``adopt,''
and ``audit'' are treated as distinct statistical problems rather than as
different names for calibration.
\section{Broader Impact and Implications}
\label{app:impact}
Reliable confidence estimates matter when LLM outputs trigger review,
abstention, regeneration or downstream actions. This work shows that improving
average calibration is not sufficient: a scalar score may discard
correctness-relevant distinctions, a fitted repair may not warrant adoption, and
residual failures may remain after average risk improves. Separating these
questions encourages deployment practices that state the information available
to a confidence predictor, compare repairs on held-out correctness evidence, and
audit structured residual errors at the right sampling unit. None of
calibration, a successful repair or audit non-rejection is a certificate of
safety: downstream utility depends on the application, decision costs, label
quality and the failure modes an audit represents.

\paragraph{Where the distinctions change a decision.}
For review or abstention, the question is whether the score separates the
answers on which errors are costly. If relevant distinctions were discarded,
moving a threshold or fitting another scalar map cannot recover them, and richer
inputs become a candidate to evaluate; a repair that lowers average squared loss
still needs comparison at the intended review cost. For a new population without
labels, PAIC fits a score-only map from an available correctness probe, which
helps when the probe's inputs are available during adaptation but costly or
unavailable at serving; its guarantee names the required transfer
property---small probe bias averaged at each score---and held-out target labels
are still needed for adoption. For regeneration or other corrective actions, an
audit can locate feature-aligned mean errors, but whether a new answer fixes
them needs its own comparison of success and cost.

\end{document}